\documentclass{article}

\usepackage[main, final]{neurips_2026}

\usepackage[utf8]{inputenc} 
\usepackage[T1]{fontenc}    
\usepackage{hyperref}       
\usepackage{url}            
\usepackage{booktabs}       
\usepackage{amsfonts}       
\usepackage{nicefrac}       
\usepackage{microtype}      
\usepackage{xcolor}         
\usepackage{etoc}
\usepackage{import}
\usepackage{amsmath}
\usepackage{amssymb}
\usepackage{mathtools}
\usepackage{subcaption}
\usepackage{amsthm}
\usepackage{bm}
\usepackage{duckuments}
\usepackage{enumitem}
\usepackage{pgfplots}
\pgfplotsset{compat=1.18}
\usepgfplotslibrary{fillbetween}
\usetikzlibrary{arrows.meta,positioning}
\usepackage{makecell}
\usepackage{wrapfig}
\usepackage{duckuments}

\usepackage[most]{tcolorbox}

\newtcolorbox{assumptionbox}[1][]{%
breakable, enhanced, blanker, left=3mm, borderline west = {1mm}{0pt}{teal}}
\newtcolorbox{theorembox}[1][]{%
breakable, enhanced, blanker, left=3mm, borderline west = {1mm}{0pt}{violet}}
\newtcolorbox{definitionbox}[1][]{%
breakable, enhanced, blanker, left=3mm, borderline west = {1mm}{0pt}{brown}}

\usepackage{pifont}
\newcommand{\cmark}{\textcolor{green!60!black}{\ding{51}}}
\newcommand{\xmark}{\textcolor{red!75!black}{\ding{55}}}

\theoremstyle{plain}
\newtheorem{theorem}{Theorem}[section]
\newtheorem{proposition}[theorem]{Proposition}
\newtheorem{lemma}[theorem]{Lemma}
\newtheorem{corollary}[theorem]{Corollary}
\theoremstyle{definition}
\newtheorem{definition}[theorem]{Definition}
\newtheorem{assumption}{Assumption}

\theoremstyle{remark}
\newtheorem{remark}[theorem]{Remark}

\newcommand{\indep}{\perp \!\!\! \perp}

\newcommand{\R}{\mathbb{R}}
\newcommand{\N}{\mathcal{N}}

\newcommand{\x}{\bm{x}}
\newcommand{\w}{\bm{w}}
\newcommand{\eps}{\bm{\varepsilon}}
\newcommand{\Bb}{\bm{b}}

\newcommand{\y}{\bm{y}}
\newcommand{\z}{\bm{z}}

\newcommand{\s}{\bm{s}}

\newcommand{\mparam}{\bm{\theta}}	

\def\1{\bm{1}}

\newcommand{\func}{f_{\mparam_f}}
\newcommand{\ps}{p_{\mparam_{\s}}}
\newcommand{\pz}{p_{\mparam_{\z}}}

\title{End-to-End Identifiable \\ Recurrent Switching
Dynamical Systems}

\author{%
  Carles Balsells-Rodas \quad Zhengrui Xiang \quad Xavier Sumba \quad Yingzhen Li \\
  Imperial College London \\
  \texttt{\{cb221,  yingzhen.li\}@imperial.ac.uk}
}

\begin{document}

\maketitle

\begin{abstract}
  Learning identifiable representations in deep generative models remains a fundamental challenge, particularly for sequential data with regime-switching dynamics. Existing approaches establish identifiability under restrictive assumptions, such as stationarity or limited emission models, and typically rely on variational autoencoder (VAE) estimators, which introduce approximation gaps that limit the recovery of the latent structure. In this work, we address both the theoretical and practical limitations of this setting. First, we establish identifiability of a broad class of recurrent nonlinear switching dynamical systems under flexible assumptions, significantly extending prior results. Second, we introduce $\Omega$SDS, a flow-based estimator that enables exact likelihood optimisation using expectation-maximisation. Through empirical validation on both synthetic and real-world data, our results demonstrate that $\Omega$SDS achieves improved disentanglement compared to VAE-based estimators and more accurate forecasting of underlying dynamics.
\end{abstract}

\etocdepthtag.toc{mtchapter}

\section{Introduction}
Many real-world time series exhibit nonstationary dynamics, where the underlying temporal mechanisms change over time. Examples include seasonal patterns in climate data \citep{saggioro2020reconstructing}, or behavioural modes in human motion \citep{aist-dance-db}. Regime-switching state-space models, including switching dynamical systems (SDSs; \citep{ghahramani2000variational}) and Markov switching models (MSMs; \citep{lindgren1978markov,poritz1982linear}), target these scenarios by introducing discrete regimes (or switches) which explicitly model dynamical changes. Recent sequential latent variable models incorporate regime-switching as structured priors \citep{linderman2017bayesian,dong2020collapsed,ansari2021deep,liu2023graph,geadah2024parsing}, improving flexibility while preserving interpretable continuous states and discrete switches.

Interpretability of deep generative models relies on identifiability, which ensures that the data distribution determines the latent structure up to acceptable equivalences, such as permutations of switches or latent variables. Without identifiability, observationally equivalent model solutions may represent contradictory latent dynamics, and therefore interpretability may reflect arbitrary parametrisations rather than the underlying data-generating mechanisms. For flexible nonlinear latent variable models, identifiability fails without additional structure \citep{hyvarinen1999nonlinear}. Existing results impose such structures either on the latent distribution, for example through auxiliary variables or temporal dependence \citep{khemakhem2020variational,halva2020hidden,klindt2021towards,morioka2021independent,yao2022temporally}, or the emission map \citep{kivva2022identifiability}. Switching dynamical systems introduce additional complexity, since the latent prior is a mixture over regime-dependent dynamics. Existing SDS identifiability results reflect this ``no free lunch'' principle between restricting emissions or latent dynamics: piecewise linear emissions allow flexible nonlinear mixture transitions \citep{balsells-rodas2024on}, while general invertible emissions require more restrictive latent dynamics, such as stationary linear transitions \citep{halva2021disentangling}. 

In this work, we push the above trade-off by establishing identifiability and estimation for recurrent switching dynamical systems (rSDS) with nonlinear transition dynamics and invertible emissions. \vspace{-0.5em}
\begin{itemize}[leftmargin=1.5em,itemsep=0pt]
  \setlength{\itemsep}{1pt}
  \setlength{\parskip}{0pt}
  \setlength{\parsep}{0pt}
\item We show that (i) temporal latents are identifiable up to affine transformations, (ii) recurrent Markov switching dynamics are identifiable up to regime permutations, and (iii) regime-dependent covariances reduce affine ambiguity to permutation and scaling.  These results provide an end-to-end identifiability framework for recurrent deep SDSs. 
\item Motivated by the theory, we propose $\Omega$SDS, a likelihood-based estimator that extracts temporal features using a flow-based emission and learns prior dynamics using expectation-maximisation. Unlike VAE-based estimators, $\Omega$SDS optimises the exact likelihood, which avoids the additional approximation gap introduced by variational inference. 
\item Empirically, exact likelihood improves latent disentanglement on synthetic data, while recurrent switches support interpretable dynamics and strong forecasting performance on controlled physical systems and real-world dancing videos.
\vspace{-0.5em}
\end{itemize}



\paragraph{Related work} SDSs are flexible models, but inference is generally challenging due to the interaction between continuous and discrete states \citep{ghahramani2000variational}. 
Recent works develop inference for switching linear dynamical systems \citep{linderman2017bayesian}, Gumbel-softmax relaxations \citep{becker2019switching}, and infinite switching processes \citep{geadah2024parsing}. Current nonlinear SDSs are mostly trained with variational objectives, including collapsed-switching inference \citep{dong2020collapsed}, recurrent explicit-duration models \citep{ansari2021deep}, and graph-based switching dynamics \citep{liu2023graph}. Flow-based likelihood estimation remains less explored \citep{debezenac2020normalizing}. Identifiability for regime-switching latent variable models builds on nonlinear ICA and finite-mixture theory \citep{khemakhem2020variational,yakowitz1968identifiability,allman2009identifiability}. HMM-ICA uses finite-state HMM identifiability to establish nonlinear ICA \citep{gassiat2016inference,halva2020hidden}, while SNICA proves disentanglement under stationary temporal dependencies without identifying the mixture components \citep{halva2021disentangling}. Identifying latent transitions is harder, as temporal dependencies prevent a direct extension of \citet{allman2009identifiability}. Recently, classical finite mixture model results \citep{yakowitz1968identifiability} have been extended to non-recurrent MSMs, yielding identifiability under Gaussian observations \citep{song2023temporally} and piecewise linear emissions \citep{balsells-rodas2024on}. Alternatively, \citet{song2024causal} establishes results by assuming knowledge of the number of states $K$ and observation-dependent regimes. Related causal-discovery methods allow recurrent switching but identify graph structures \citet{balsells-rodas25causal} rather than the switching process. 
Table~\ref{tab:related_work_comparison} summarises these distinctions. 

\begin{table}[t]
    \caption{
    Comparison with identifiable regime-switching latent variable models.
    Full ident. excludes identifying $K$ and denotes recovery of the latent representation, transitions, and switching process.
    }
    \centering
    \resizebox{\linewidth}{!}{%
    \begin{tabular}{lcccccc}
        \toprule
        Method 
        & Emission 
        & Switching 
        & Nonlin. dyn.
        & Identifies $K$
        & Full ident.
        & Exact LL \\
        \midrule
        HMM-ICA \citep{halva2020hidden}
        & Bijective & Autonomous & \xmark & \xmark & \xmark & \cmark \\

        SNICA \citep{halva2021disentangling}
        & Invertible & Autonomous & \xmark & \xmark & \xmark & \xmark \\

        NCTRL \citep{song2023temporally}
        & Gaussian / obs. & Autonomous & \cmark & \cmark & \cmark & \xmark \\

        iSDS \citep{balsells-rodas2024on}
        & PWL & Autonomous & \cmark & \cmark & \cmark & \xmark \\

        CtrlNS \citep{song2024causal}
        & Invertible & Obs.~dependent & \cmark & \xmark & \cmark & \xmark \\

        SDCI \citep{balsells-rodas25causal}
        & N/A & Recurrent & \cmark & \cmark & \xmark & N/A \\


        \midrule
        \textbf{$\Omega$SDS} (ours)
        & Invertible & Recurrent & \cmark & \cmark & \cmark & \cmark \\
        \bottomrule
    \end{tabular}}
    \label{tab:related_work_comparison}
    \vspace{-0.5em}
\end{table}

\section{Recurrent Switching Dynamical Systems}\label{sec:back_sds}

A recurrent Switching Dynamical System (rSDS), exemplified in Figure~\ref{fig:sds_pgm}, models sequential data by introducing continuous latent states $\z_{1:T}$ and discrete switching variables $\s_{1:T}$. At each time step $t$, the observation $\x_t\in\R^n$ is generated through a (noisy) transformation of the continuous latent variable $\z_t\in\R^m$. While recent nonlinear SDS approaches use additive-noise emissions \citep{dong2020collapsed,ansari2021deep}, we assume a flow-based emission similarly used in \citet{Sorrenson2020Disentanglement}:
\begin{assumptionbox}
\begin{assumption}[Emission model]\label{ass:emission}
    Observations $\x_t\in\R^n$ are generated via a $\mathcal{C}_1$ diffeomorphism $\func:\R^m\times\R^{d}\to\R^n$ of latent variables $\z_t\in\R^m$ and i.i.d. noise $\bm{\varepsilon}_t\in\R^d$, with $d = n-m$.
    \begin{equation}\label{eq:generation_flows}
        \x_t = \func(\z_t, \bm{\varepsilon}_t), \quad \bm{\varepsilon}_t \sim \mathcal{N}(0,\sigma_{\eps}^2I_d).
    \end{equation}
\end{assumption}
\end{assumptionbox}

\begin{wrapfigure}{r}{.38\linewidth}
    \centering
    \includegraphics[width=\linewidth]{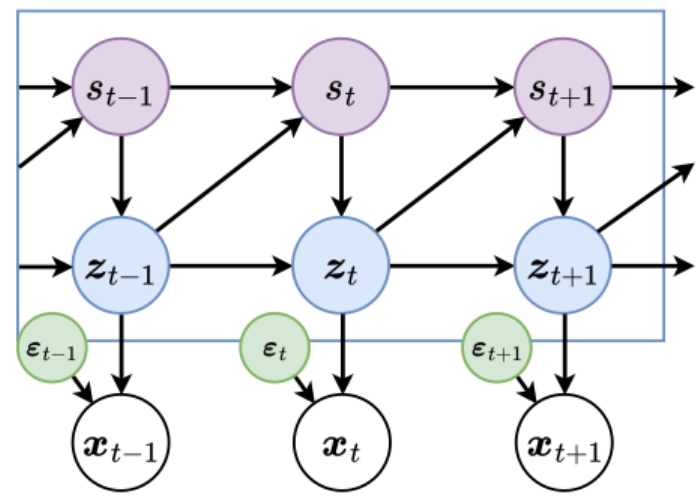}
    \vspace{-1.3em}
    \caption{The considered rSDS model. Observations are generated through an invertible mapping of recurrent switching features and exogenous noises.}
    \label{fig:sds_pgm}
    \vspace{-4em}
\end{wrapfigure}
The latent dynamics $\z_{1:T}$ follow a recurrent Markov Switching Model (rMSM) \citep{ephraim2005revisiting}. Therefore, at time $t$, the discrete switch $s_t\in\{1,\dots,K\}$ determines the transition dynamics of $\z_t$, while the switching process receives feedback from past latent variables $\z_{t-1}$. Under first-order Markov assumptions, the joint distribution of $(\z_{1:T}, \s_{1:T})$ factorises as follows:
\begin{multline}\label{eq:transition_dynamics}
    p_{\mparam_{\z},\mparam_{\s}}(\z_{1:T}, \s_{1:T}) = \ps(s_1)\pz(\z_1|s_1) \\\prod_{t=2}^T \pz(\z_t\mid\z_{t-1}, s_t) \ps(s_t\mid s_{t-1}, \z_{t-1}).
\end{multline}
We make the following assumption on the latent transition dynamics $\pz(\z_t\mid\z_{t-1}, s_t)$.

\begin{assumptionbox}
\begin{assumption}[Analytic Gaussian autoregressive dynamics]\label{ass:gaussian_transition}
    The latent variables $\z_t \in \R^m$ follow a regime-dependent Gaussian autoregressive process:
    $$\z_t \mid s_t, \z_{t-1} \sim \N\Big(\z_t\mid \bm{m}_{s_t}(\z_{t-1}), \Sigma_{s_t}\Big), \quad \Sigma_{s_t} = \mathrm{diag}(\bm \sigma^2_{s_t})$$
    where each mean function $\bm{m}_{k}:\R^m\to\R^m$ is analytic, each covariance $\Sigma_{k}\succ 0$ is diagonal with $\bm \sigma^2_{k}$ denoting the diagonal, and both means $\bm{m}_{k}$ and/or covariances $\Sigma_{k}$ are distinct across regimes.
\end{assumption}
\end{assumptionbox}
This assumption is in the same format as in \citet{balsells-rodas2024on}, where results are established for non-recurrent switching, i.e., $\ps(s_t\mid s_{t-1}, \z_{t-1})=\ps(s_t\mid s_{t-1})$. In contrast, we consider a common recurrent formulation in rSDS \citep{linderman2017bayesian,dong2020collapsed,ansari2021deep} to increase the flexibility of the switching process, where switches can adapt to latent distributional changes rather than following fixed Markov transition matrices.
\begin{assumptionbox}
\begin{assumption}[Recurrent switching process]\label{ass:recurrent}
    The switches are dependent on previous latents:
    $$\ps(s_t=k \mid s_{t-1}=j,\z_{t-1}) = Q_{jk}(\z_{t-1}), \quad Q:\R^{m} \to [0, 1]^{K\times K}, \quad \sum_{k=1}^K Q_{jk}(\z_{t-1})=1,$$
    where $Q$ is continuous almost everywhere in $\R^m$. We also assume each column of $Q(\z)$ contains at least one positive entry for almost every $\z\in\R^m$ and $p(s_1=k)>0$ for all $k\in\{1,\dots,K\}$.
\end{assumption}
\end{assumptionbox}
A.e. continuity of $Q$ is a weak regularity assumption: it allows discontinuities on lower-dimensional boundaries, such as switching surfaces or collision boundaries, e.g., in bouncing-ball dynamics.

\paragraph{Identifiability desiderata in SDSs.}  Let $\bm\Theta$ be a family of parameters which parametrise an rSDS under Eqs.~\eqref{eq:generation_flows} and~\eqref{eq:transition_dynamics}, where $\mparam:=\{\mparam_f, \mparam_{\z}, \mparam_{\s},m,K\}$, $\mparam\in\bm\Theta$. We aim to establish identifiability for every parameter of $\mparam$: (i) the latent dimensionality $m$, (ii) the number of regimes $K$, (iii) the distribution of the latent variables $\z_t$, (iv) the autoregressive components $\pz(\z_t\mid \z_{t-1}, s_t)$, and (v) the switching process $\ps(s_t\mid s_{t-1}, \z_{t-1})$. Additionally, we denote $p_{\mparam}(\x_{1:T})$ and $p_{\mparam_{\z}}(\z_{1:T})$ as the observational and latent distributions induced by the model. Notably, our emission structure allows including the latent dimensionality $m$ in the target equivalence class, which is non-standard in the literature \citep{song2023temporally,balsells-rodas2024on,song2024causal}.

\begin{definitionbox}
\begin{definition}[Latent identifiability up to affine transformations]\label{def:affine_identifiability} 
The latent variables $\z$ are identifiable up to affine transformations, if for any $\mparam,\mparam'\in\bm\Theta$ such that $p_{\mparam}(\x_{1:T})=p_{\mparam'}(\x_{1:T})$, $\forall \x_{1:T}\in\R^{Tn}$, we have $m=m'$ and there exists an invertible affine transformation $\psi:\R^m\to\R^m$ such that $p_{\mparam_{\z}} = \psi^{(T)}_{\#}p_{\mparam'_{\z}}$, where $\psi^{(T)}:\R^{Tm}\to\R^{Tm}$ is defined by $\psi^{(T)}(\z_{1:T}):= (\psi(\z_1),\dots,\psi(\z_T))$.
\end{definition}
\end{definitionbox}
\begin{definitionbox}
\begin{definition}[Transition identifiability up to permutations]\label{def:perm_identifiability}
The transition model is identifiable up to permutations if for any $\mparam,\mparam'\in\bm\Theta$ such that $p_{\mparam_{\z}}(\z_{1:T})=p_{\mparam'_{\z}}(\z_{1:T})$, $\forall \z_{1:T}\in\R^{Tm}$, we have $K= K'$ and there exists a permutation $\pi\in S_K$ such that for any $k\in\{1,\dots K\}$, $\pz(\z_t\mid\z_{t-1}, s_t=k) = p_{\mparam'_{\z}}(\z_t\mid\z_{t-1}, s_t=\pi(k))$, and for any $k,j\in\{1,\dots,K\}$,  $$\ps(s_t=k\mid s_{t-1}=j, \z_{t-1}) = p_{\mparam'_{\s}}(s_t=\pi(k)\mid s_{t-1}=\pi(j), \z_{t-1}).$$
\end{definition}
\end{definitionbox}
Def.~\ref{def:affine_identifiability} concerns identification of the latent representation up to affine transformations, while Def.~\ref{def:perm_identifiability} targets the regime-dependent dynamics once the continuous latent space is fixed. Since the rSDS is composed of several distinct elements, establishing end-to-end identifiability is a challenging task. 

\section{Identifiability Theory for Recurrent Switching Dynamical Systems}\label{sec:method}

In this section, we present our main identifiability results as follows: \vspace{-0.5em}
\begin{enumerate}[leftmargin=1.5em,itemsep=0pt]
  \setlength{\itemsep}{1pt}
  \setlength{\parskip}{0pt}
  \setlength{\parsep}{0pt}
    \item[i)] \textbf{Latent variable identifiability.} We establish identifiability of the latent variable distribution $\pz(\z_{1:T})$ up to affine transformations under invertible emissions in Section~\ref{sec:first_stage}.
    \item[ii)] \textbf{rMSM identifiability.} We then establish identifiability of the transition dynamics up to permutations in Section~\ref{sec:stage_2}.
    \item[iii)] \textbf{Further disentangling latent variables.} Finally, we impose regime-dependent assumptions that reduce the affine ambiguity of latent variables to permutation and scaling in Section~\ref{sec:stage_3}.
    \vspace{-0.5em}
\end{enumerate}
Importantly, each stage may be of independent interest, as we significantly strengthen previous results. 
For each result, we first state the additional assumptions required, followed by the corresponding theoretical result. We then provide a proof sketch and conclude with a discussion of its implications. 


\subsection{Latent Space Identifiability up to Affine Transformations}\label{sec:first_stage}

We first establish the affine identifiability of the continuous latent variables. Since we keep general invertible emissions, identifiability must come from the latent transition dynamics. The following assumptions ensure sufficient regime separation and local variability.

\begin{assumptionbox}
\begin{assumption}[Sticky switching]\label{ass:sticky_switch}
There exists a non-empty open set $\mathcal{Z}\subseteq\R^m$ such that
    \begin{equation}\label{eq:stickiness}
        Q_{kk}(\z) \geq 1 - \epsilon(\z), \quad \epsilon(\z)\in[0, \tfrac12), \ \forall k\in\{1,\dots,K\}, \ \forall \z\in\mathcal{Z},
    \end{equation}
    where $\epsilon:\mathcal{Z}\to [0, \tfrac12)$ is continuous in $\mathcal{Z}$.
\end{assumption}
\end{assumptionbox}
\begin{assumptionbox}
\begin{assumption}[Transition separation]\label{ass:posterior_domi}
    Furthermore, there exist $\z^*\in\mathcal{Z}$, a constant $\ell^*>0$, and an integer $T^*\geq 2$ such that, for any regime $i\in\{1,\dots,K\}$, there exists a latent history
    $\z_{1}^{(i)}(\z^*),\dots, \z_{T^*}^{(i)}(\z^*) \in\mathcal{Z},\ \z_{T^*}^{(i)}(\z^*)=\z^*,$ for which
\begin{equation}\label{eq:likelihood_margin}
    \frac{\N( \z_{\tau}^{(i)}(\z^*)|  \z_{\tau-1}^{(i)}(\z^*),s_\tau=i )}{\max_{j\neq i}\N( \z_{\tau}^{(i)}(\z^*)|  \z_{\tau-1}^{(i)}(\z^*), s_\tau=j )} > e^{\ell^*},\quad \tau=2,\dots,T^*, \quad \ell^* > \log\frac{1+\epsilon^*}{1 - \epsilon^*},
\end{equation}
where $\epsilon^* = \max \{\epsilon(\z_{1}^{(i)}(\z^*)), \dots, \z_{T^*-1}^{(i)}(\z^*)\}$, and $T^*$ is chosen large enough so that
\begin{equation}
    p(s_{T^*}=i\mid \z^{(i)}_{1:T^*}(\z^*)) > \frac12, \quad \forall i\in\{1,\dots,K\}.
\end{equation}
\end{assumption}
\end{assumptionbox}
\begin{assumptionbox}
\begin{assumption}[Covariance separation]\label{ass:regime_domi}
For every regime $k\in\{1,\dots,K\}$, there exists a nonempty set $\mathcal{I}_k\subseteq\{1,\dots,m\}$ such that $\sigma_{k,i}\geq \sigma_{k',i}$, for any $i\in\mathcal{I}_k$ and any $k'\in\{1,\dots,K\}\setminus\{k\}$. Let $\mathcal{T}_k \subseteq\{1,\dots,K\}\setminus \{k\}$ denote the regimes tied with $k$ in variance along all coordinates in $\mathcal{I}_k$:
\begin{equation}
     \mathcal T_k:=
    \Big\{k'\in\{1,\dots,K\}\setminus \{k\}:\ \sigma_{k,i}=\sigma_{k',i},\ \forall i\in\mathcal I_k\Big\}.
\end{equation}
If $\mathcal T_k\neq\emptyset$, we further assume that for every $\z_t\in\R^m$,
\begin{equation}\label{eq:convex_hull_assumption}
 \bm m_{k,\mathcal I_k}(\z_t) \notin\operatorname{conv}(\left\{\bm m_{k',\mathcal I_k}(\z_t):k'\in\mathcal T_k\right\}),
\end{equation} where $\bm m_{k,\mathcal I_k}(\cdot)$ denotes the subvector of $\bm m_{k}(\cdot)$ corresponding to coordinates in $\mathcal{I}_k$, and  $\operatorname{conv}$ denotes the convex hull.
\end{assumption}
\end{assumptionbox}
\begin{assumptionbox}
\begin{assumption}[Locally nondegenerate regime]\label{ass:nondegeneratus}
    There exist a regime $1 \leq k_0 \leq K$ and a point $\z^*\in\R^m$ such that the determinant of the Jacobian of $\bm m_{k_0}(\z^*)$ is nonzero: $\det J_{\bm m_{k_0}}(\z^*) \neq 0.$
\end{assumption}
\end{assumptionbox}
Under these assumptions, we establish latent variable identifiability up to affine transformations, and provide further intuitions on the required conditions after a proof sketch.

\begin{theorembox}
\begin{theorem}\label{thm:affine_identifiability}
    Under assumptions (\ref{ass:emission}--\ref{ass:nondegeneratus}), the latent representation of the rSDS (Eqs.~\eqref{eq:generation_flows},\eqref{eq:transition_dynamics}) is identifiable up to affine transformations (Def. \ref{def:affine_identifiability}), if one of the following conditions holds: \vspace{-0.5em} 
    \begin{itemize}[leftmargin=2.25em,noitemsep]
  \setlength{\itemsep}{1pt}
  \setlength{\parskip}{0pt}
  \setlength{\parsep}{0pt} 
        \item[(i)] the number of regimes $K$ is known.
        \item[(ii)] $K$ is unknown, but $Q:\R^{m}\to [0,1]^{K\times K}$ is analytic.
    \end{itemize}
\end{theorem}
\end{theorembox}
See Appendix \ref{app:main_proof} for proof. When $K$ is known (e.g., controlled physical systems), identifiability reduces to aligning the mixture components, where we provide a proof sketch below. For unknown $K$ (e.g., discovering behavioural patterns in videos), recovering the number of regimes requires additional structure, where condition (ii) in Thm.~\ref{thm:affine_identifiability} provides a sufficient one. 

\emph{Proof sketch (known $K$)}: Define the pushforward measure of $p(\z_{t+1}, \eps_{t+1} \mid \z_{t}, s_t=k)$ through $f_{\mparam_f}$:
\begin{equation*}
    \phi_k(\x_{t+1}| \x_t):=(f_{\mparam_f})_{\#} \big( \N(\z_{t+1} \mid \bm m_k(\z_{t}), \Sigma_{k})\, p_{\eps}(\eps_{t+1}) \big).
\end{equation*} Considering all $k\in\{1,\dots,K\}$, define $\Phi(\x_{t+1}| \x_t) := (\phi_1(\x_{t+1}| \x_t), \dots, \phi_K(\x_{t+1}| \x_t))^\top$ as the pushforward measures of the $K$ Gaussian components in vector form (define $\Phi'(\x_{t+1}| \x_t)$ similarly).\vspace{-.25em}
\begin{enumerate}[leftmargin=2.25em,itemsep=0pt]
  \setlength{\itemsep}{1pt}
  \setlength{\parskip}{0pt}
  \setlength{\parsep}{0pt} 
    \item[\textbf{I.}]  Write the predictive distribution $p_{\mparam}(\x_{t+1}\mid\x_{1:t})$ as a mixture in terms of $\Phi(\x_{t+1}| \x_t)$:\vspace{-.15em}
    \begin{equation}
        p_{\mparam}(\x_{t+1}\mid\x_{1:t}) = \sum_{k=1}^K p(s_{t+1}=k \mid \x_{1:t})\phi_k(\x_{t+1}\mid \x_{t}).
    \end{equation}
    \vspace{-.75em}
    
    Under  (\ref{ass:sticky_switch}--\ref{ass:posterior_domi}) and using Lemma \ref{lemma:concentration_to_floor}, we can choose $K$ distinct histories $(\x_1^{(k)},\dots,\x_{t-1}^{(k)},\x_t)$ that share the same endpoint $\x_t$ but induce different dominant posterior distributions satisfying $p(s_{t+1}=k \mid \x_{1:t})>\tfrac12$. This allows the construction of a full-rank linear system in terms of $\x_t$. 
    We then show there exists an invertible matrix $\Pi(\x_t)$ such that
    \begin{equation} 
        p_{\mparam}(\x_{1:T})=p_{\mparam'}(\x_{1:T}) \quad \Rightarrow \quad \Phi(\x_{t+1}| \x_t) = \Pi(\x_t)\Phi'(\x_{t+1}| \x_t).
    \end{equation}
    \item[\textbf{II.}] Under \eqref{ass:regime_domi}, Prop~\ref{prop:a6_gaussian_dominance}, Prop~\ref{prop:sufficient_conditions_nonnegativity}, and Lemma~\ref{lemma:nonnegative_entries}, $\Pi(\x_t)$ is a constant permutation matrix.
    \item[\textbf{III.}] Since each component in model $\mparam$ corresponds to a permuted component in model $\mparam'$, pick regime $k_0$ from (\ref{ass:nondegeneratus}). Using an iVAE-style argument, we recover the latent dimensionality $m$ and establish that the latent dynamics are identifiable up to affine transformations (Def. \ref{def:affine_identifiability}).
\end{enumerate}


As discussed, establishing the above result requires restricting the flexibility of the nonlinear switching prior through (\ref{ass:sticky_switch}--\ref{ass:nondegeneratus}). 
(\ref{ass:nondegeneratus}) requires one regime for which the Jacobian of the mean transition is locally non-singular, so that variations in latent space are distinguishable. This condition is closely related to sufficient variability assumptions used in related ICA literature \citep{hyvarinen2017nonlinear, khemakhem2020variational,yao2022temporally}, but adapted to our setting where variability arises through the means of Gaussian transitions.

\begin{wrapfigure}{r}{.415\linewidth}
\centering
\vspace{-2.7em}
    \resizebox{.9\linewidth}{!}
        {\begin{tikzpicture}[>=Latex, line cap=round, line join=round, font=\Large]

\def\L{3.75}
\def\rball{0.18}

\colorlet{c1}{orange}
\colorlet{c2}{red!75!black}
\colorlet{c3}{blue}
\colorlet{c4}{green!60!black}

\fill[black!5] (-\L,-\L) rectangle (\L,\L);

\fill[black!50] (-\L,-\L) rectangle (\L,-\L+0.15);
\fill[black!50] (-\L,\L-0.15) rectangle (\L,\L);
\fill[black!50] (-\L,-\L) rectangle (-\L+0.15,\L);
\fill[black!50] (\L-0.15,-\L) rectangle (\L,\L);

\draw[thick] (-\L,-\L) rectangle (\L,\L);

\filldraw[fill=olive, draw=black, thick] (0,-1) circle (\rball);
\node[below=5pt] at (0,-1-\rball) {$z_t$};


\node[align=center, fill=white, fill opacity=0.0, text opacity=1, rounded corners=2pt]
  at (0,3.)
  {\(\displaystyle \boldsymbol m_k(\boldsymbol{z}_{t-1})=\boldsymbol{z}_{t-1}+v_k\)};

\node[align=center] at (0,2.25) {%
  $\displaystyle v_k \in \{$
  \tikz[baseline=-1.5ex]\draw[c2, very thick, -{Latex[length=1.6mm]}] (0,0) -- (-0.28,-0.28);
  ,\;
  \tikz[baseline=-1.5ex]\draw[c1, very thick, -{Latex[length=1.6mm]}] (0,0) -- (0.28,-0.28);
  ,\;
  \tikz[baseline=-0ex]\draw[c4, very thick, -{Latex[length=1.6mm]}] (0,0) -- (-0.28,0.28);
  ,\;
  \tikz[baseline=-0ex]\draw[c3, very thick, -{Latex[length=1.6mm]}] (0,0) -- (0.28,0.28);
  $\displaystyle \}$%
};
\draw[c1, very thick, -{Latex[length=2mm]}]
   (-2.4,1.4) -- (-1.6,0.6) -- (-0.8,-0.2) -- (-0.18,-0.82);
\foreach \x/\y in {-2.4/1.4,-1.6/0.6,-0.8/-0.2}
  \fill[c1] (\x,\y) circle (3pt);

\node[c1, left] at (-2.5,1.4) {$\boldsymbol z_{t-3}^{(1)}$};
\node[c1,  left] at (-1.7,0.6) {$\boldsymbol z_{t-2}^{(1)}$};
\node[c1,  left] at (-0.9,-0.2) {$\boldsymbol z_{t-1}^{(1)}$};

\draw[c2, very thick, -{Latex[length=2mm]}]
  (2.4,1.4) -- (1.6,0.6) -- (0.8,-0.2) -- (0.18,-0.82);
\foreach \x/\y in {2.4/1.4,1.6/0.6,0.8/-0.2}
  \fill[c2] (\x,\y) circle (3pt);

\node[c2,  right] at (2.5,1.4) {$\boldsymbol z_{t-3}^{(2)}$};
\node[c2,  right] at (1.7,0.6) {$\boldsymbol z_{t-2}^{(2)}$};
\node[c2,  right] at (0.9,-0.2) {$\boldsymbol z_{t-1}^{(2)}$};

\draw[c3, very thick, -{Latex[length=2mm]}]
  (-2.4,-3.4) -- (-1.6,-2.6) -- (-0.8,-1.8) -- (-0.18,-1.18);
\foreach \x/\y in {-2.4/-3.4,-1.6/-2.6,-0.8/-1.8}
  \fill[c3] (\x,\y) circle (3pt);

\node[c3,  left] at (-2.5,-3.3) {$\boldsymbol z_{t-3}^{(3)}$};
\node[c3,  left] at (-1.7,-2.5) {$\boldsymbol z_{t-2}^{(3)}$};
\node[c3,  left] at (-0.9,-1.7) {$\boldsymbol z_{t-1}^{(3)}$};

\draw[c4, very thick, -{Latex[length=2mm]}]
   (2.4,-3.4) -- (1.6,-2.6) -- (0.8,-1.8) -- (0.18,-1.18);
\foreach \x/\y in {2.4/-3.4,1.6/-2.6,0.8/-1.8}
  \fill[c4] (\x,\y) circle (3pt);

\node[c4,  right] at (2.5,-3.3) {$\boldsymbol z_{t-3}^{(4)}$};
\node[c4, right] at (1.7,-2.5) {$\boldsymbol z_{t-2}^{(4)}$};
\node[c4, right] at (0.9,-1.7) {$\boldsymbol z_{t-1}^{(4)}$};

\end{tikzpicture}}
        \caption{Bouncing ball example ($K=4$), each corresponding to a fixed velocity $\bm v_k$ in the 2D latent space. Histories $\z_{t-3:t-1}^{(k)}$ associated with different velocities remain separated, illustrating the convex hull and likelihood separation intuitions.}
        \vspace{-1.3em}
    \label{fig:bouncing_ball}
\end{wrapfigure}
\paragraph{Understanding  (\ref{ass:regime_domi}):}
(\ref{ass:regime_domi}) imposes a dominance condition on the regime-dependent variances, preventing any single regime from dominating across all directions. In shared-variance settings, variance dominance is not enough, and the regime-dependent means require further structure across the regimes and dimensions with equal variance. The convex hull condition in Eq.~\eqref{eq:convex_hull_assumption} is geometrically intuitive, but does not allow arbitrary nonlinear parametrisations. To enforce this, the following design is possible for a given regime $k\in[K]$:
\begin{equation}
\label{eq:bouncing_ball}
    \bm m_{k}(\z) = h(\z) + B(\z)\bm v_k,
\end{equation}
with invertible $B(\z)$ for any $\z\in\R^m$ and $\bm v_k\in \R^m$. The condition holds if $\{\bm v_{k',\mathcal{I}_k}: k'\in \mathcal{T}_k\cup \{k\}\}$ are the vertices of a convex hull. Figure~\ref{fig:bouncing_ball} exemplifies (\ref{ass:regime_domi}), where $\bm v_{k}\in\{(1,1), (-1,-1), (-1,1), (1,-1)\}$ forms a convex shape, and $h(\z):=\z$, $B(\z):=\Delta t$.

\paragraph{Understanding  (\ref{ass:sticky_switch}--\ref{ass:posterior_domi}):}
The intuition behind these conditions is that we have histories $\z_{1:t}^{(k)}$ such that the posterior distribution $p(s_t|\z_{1:t}^{(k)})$ concentrates on a certain regime $k$, and by Lemma~\ref{lemma:concentration_to_floor}, this is possible through sticky switches (\ref{ass:sticky_switch}), and likelihood separation across regimes (\ref{ass:posterior_domi}). While \eqref{ass:sticky_switch} is a common modelling choice in regime-switching models \citep{fox2008hdp,linderman2017bayesian,glaser2020recurrent}, \eqref{ass:posterior_domi} is analogous to signal-strength conditions used in change-point detection. For example, \citet{yu2023note} assumes nonzero jumps between changes. In our case, \eqref{ass:posterior_domi} imposes consistent regime-dependent likelihood margins along a history. Figure~\ref{fig:bouncing_ball} further illustrates this notion, where the latent variable transitions of a bouncing ball are separated by a positive margin thanks to having fixed-velocity dynamics driven by $\bm v_k\in\R^2$. Observations of an SDS in this example would correspond to videos where each frame is rendered from the 2D position of the ball in the box. A possible limitation to \eqref{ass:posterior_domi} is the difficulty in maintaining large enough margins during multiple transitions such that the posterior domination is satisfied. However, stickiness plays a key role in weakening the required margin. We expand this discussion in Appendix~\ref{app:dominance_assumptions}, where we show that by Corollary~\ref{cor:half_dominance_time} the sufficient history length $T^*$ can be determined from the initial posterior $p(s_1|\z_1)$, the margin $\ell^*$, and stickiness $\epsilon^*$. Specifically for uniform initial posterior over $K$ regimes, $p(s_1=k|\z_{1})=\tfrac1K,\ k\in\{1,\dots,K\}$, posterior dominance is possible after one transition ($T^*=2$) if
$\ell^* > \log\frac{K-1+\epsilon^*}{1-\epsilon^*}$.
Therefore, (\ref{ass:sticky_switch}--\ref{ass:posterior_domi}) can be interpreted as a design choice on a trade-off between transition separation and switching stickiness. For more details, see Appendix~\ref{app:dominance_assumptions}. Below, we exemplify nonlinear dynamics.

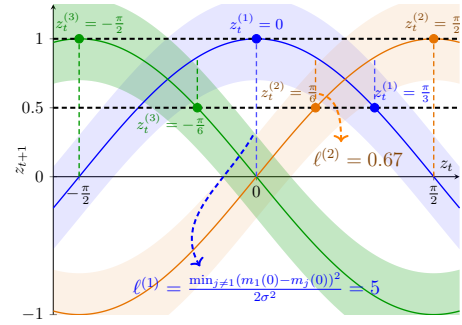
\begin{wrapfigure}{r}{.43\linewidth}
\vspace{-1.em}
        \centering
        \resizebox{\linewidth}{!}
        {\begin{tikzpicture}
\begin{axis}[
    font=\Large,
    width=10cm,
    height=8cm,
    axis y line=left,
    axis x line=middle,
    xmin=-1.8, xmax=1.8,
    ymin=-1.0, ymax=1.25,
    samples=300,
    domain=-1.8:1.8,
    xlabel={$z_t$},
    ylabel={$z_{t+1}$},
    ylabel style={yshift=-10pt},
    xtick={-1.57079632679,0,1.57079632679},
    xticklabels={$-\frac{\pi}{2}$,$0$,$\frac{\pi}{2}$},
    ytick={-1,0, 0.5,1},
    legend style={draw=none, fill=none, at={(0.98,0.98)}, anchor=north east},
    clip=true
]

\addplot[name path=bluemean, blue, thick] {cos(deg(x))};
\addplot[name path=blueupper, draw=none] {cos(deg(x)) + 0.3};
\addplot[name path=bluelower, draw=none] {cos(deg(x)) - 0.3};
\addplot[blue, fill=blue, fill opacity=0.10, draw=none]
fill between[of=blueupper and bluelower];

\addplot[name path=orangemean, orange!90!black, thick] {cos(deg(x - pi/2))};
\addplot[name path=orangeupper, draw=none] {cos(deg(x - pi/2)) + 0.3};
\addplot[name path=orangelower, draw=none] {cos(deg(x - pi/2)) - 0.3};
\addplot[orange!90!black, fill=orange!90!black, fill opacity=0.10, draw=none]
fill between[of=orangeupper and orangelower];

\addplot[name path=greenmean, green!60!black, thick] {cos(deg(x + pi/2))};
\addplot[name path=greenupper, draw=none] {cos(deg(x + pi/2)) + 0.3};
\addplot[name path=greenlower, draw=none] {cos(deg(x + pi/2)) - 0.3};
\addplot[green!50!black, fill=green!60!black, fill opacity=0.10, draw=none]
fill between[of=greenupper and greenlower];
\addplot[
    green!60!black,
    fill=green!60!black,
    fill opacity=0.15,
    draw=none
] fill between[of=greenupper and greenlower];
\addplot[only marks, mark=*, mark size=2.5pt, blue] coordinates {(0,1)};
\addplot[only marks, mark=*, mark size=2.5pt, orange!90!black] coordinates {(pi/2,1)};
\addplot[only marks, mark=*, mark size=2.5pt, green!60!black] coordinates {(-pi/2,1)};

\addplot[densely dashed, blue!70, thick] coordinates {(0,0) (0,1)};
\addplot[densely dashed, orange!90!black, thick] coordinates {(pi/2,0) (pi/2,1)};
\addplot[densely dashed, green!60!black, thick] coordinates {(-pi/2,0) (-pi/2,1)};
\addplot[densely dashed, black!100, very thick] coordinates {(-1.8,1) (1.2*pi/2,1)};
\node[blue, above,font=\Large] at (axis cs:0,1) {$z_t^{(1)}=0$};
\node[orange!50!black, anchor=south east,font=\large] at (axis cs:pi/2,1) {$z_t^{(2)}=\frac{\pi}{2}$};
\node[green!50!black, anchor=south west,font=\large] at (axis cs:-pi/2 + 0.1,1) {$z_t^{(3)}=-\frac{\pi}{2}$};

\draw[->, blue, densely dashed, very thick]
    (axis cs:-0.03,0.3) to[out=-120,in=120] (axis cs:-0.5,-0.65);
\node[blue, below,font=\Large] at (axis cs:0.,-0.65) {$\ell^{(1)} = \frac{\min_{j\neq 1}(m_1(0) - m_j(0))^{2} }{2\sigma^2} =  5$};

\draw[->, orange, densely dashed, very thick]
    (axis cs:0.03 + pi/6,0.6) to[out=-10,in=90] (axis cs:+0.75,0.25);
\node[orange!50!black, below,font=\Large] at (axis cs:0.9,+0.25) {$\ell^{(2)} =  0.67$};

\addplot[densely dashed, black, very thick] coordinates {(-1.8,0.5) (1.8,0.5)};


\addplot[only marks, mark=*, mark size=2.5pt, blue] coordinates {(pi/3,0.5)};
\addplot[only marks, mark=*, mark size=2.5pt, orange!90!black] coordinates {(pi/6,0.5)};
\addplot[only marks, mark=*, mark size=2.5pt, green!60!black] coordinates {(-pi/6,0.5)};

\addplot[densely dashed, blue!70, thick]
coordinates {(pi/3,0.5) (pi/3,{sqrt(3)/2})};

\addplot[densely dashed, orange!90!black, thick]
coordinates {(pi/6,0.5) (pi/6,{sqrt(3)/2})};

\addplot[densely dashed, green!60!black, thick]
coordinates {(-pi/6,0.5) (-pi/6,{sqrt(3)/2})};

\node[blue, above, font=\Large] at (axis cs:0.25+pi/3,0.5) {$z_t^{(1)}=\frac{\pi}{3}$};
\node[orange!50!black, above,font=\Large] at (axis cs:-.23+pi/6,0.5) {$z_t^{(2)}=\frac{\pi}{6}$};
\node[green!50!black, below,font=\Large] at (axis cs:-.23-pi/6,0.5) {$z_t^{(3)}=-\frac{\pi}{6}$};

\end{axis}
\end{tikzpicture}}
        \vspace{-1em}
        \caption{Cosine transition means for the $K=3$ example. For $\z_{t+1}=1,0.5$ we find positive margins $\ell^{*}=5,0.67$, respectively.}
        \label{fig:cosine_example}
        \vspace{-2.25em}
\end{wrapfigure}
\textbf{A visual example for (\ref{ass:posterior_domi}):} Consider $m=1$, $K=3$, equal variances $\sigma^2=0.1$, and cosine transitions:
\vspace{0.25em}
\begin{equation*}
    m_k(z) = \cos\left(z + \frac{\pi(k-1)}2\right), \  k=1,2,3.
\end{equation*}
\vspace{0.25em}
The transition means are shown in Figure \ref{fig:cosine_example}. Choosing the common endpoint $z_2=1$, an aligned set of predecessors is $z_1^{(1)}=0$, $z_1^{(2)}=\frac{\pi}{2}$, and $z_1^{(3)}=-\frac{\pi}{2}$, for which $m_k(z_1^{(k)})=1$, $k=1,2,3$. As illustrated in the figure, the minimum likelihood margin is $\ell^*=5$. Since (\ref{ass:posterior_domi}) requires existence, this single construction is sufficient. Under a uniform initial posterior, the margin is large enough that any $\epsilon^*<\tfrac12$ satisfies Lemma~\ref{lemma:concentration_to_floor} with $T^*=2$. In contrast, choosing a less separated endpoint $z_2=0.5$ yields a margin $\ell^*=0.67$, where no $\epsilon^*>0$ can satisfy  $\ell^* > \log\frac{K-1+\epsilon^*}{1-\epsilon^*}$ to obtain one-step dominance. This illustrates that stronger local separation permits shorter histories, while weaker separation must be compensated either by stronger stickiness or longer histories.


\subsection{Identifying Recurrent Prior Dynamics}\label{sec:stage_2}

Given that the Gaussian family is closed under affine transformations, we continue our analysis by working directly on latent dynamics driven by rMSMs. We present the following identifiability result.
\begin{theorembox}
\begin{theorem}\label{thm:msm_transition_identifiability}
The rMSM is identifiable up to permutations under the following conditions:
\begin{enumerate}[leftmargin=2.25em,itemsep=0pt]
  \setlength{\itemsep}{1pt}
  \setlength{\parskip}{0pt}
  \setlength{\parsep}{0pt}
    \item[(i)] Under Assumption~(\ref{ass:gaussian_transition}), the continuous latent transitions $\pz(\z_t\mid \z_{t-1}, s_t)$   are identifiable.
    \item[(ii)] Under Assumptions (\ref{ass:gaussian_transition}-\ref{ass:posterior_domi}) and further assuming $Q:\R^{m}\to [0,1]^{K\times K}$ is analytic, the rMSM is identifiable following Definition~\ref{def:perm_identifiability}.
\end{enumerate}
\end{theorem}
\end{theorembox}
\emph{Proof sketch.}
See Appendix~\ref{app:proof_identifiability_rmsm} for the full proof. We prove rMSM identifiability in two steps.
\begin{enumerate}[leftmargin=2.25em,itemsep=0pt]
  \setlength{\parskip}{0pt}
  \setlength{\parsep}{0pt}
    \item[(i)] Thm.~\ref{thm:msm_transition_identifiability_z_transitions}: By \citet{yakowitz1968identifiability} and (\ref{ass:gaussian_transition}), the transitions $\pz(\z_t\mid \z_{t-1}, s_t)$  are identifiable up to a permutation indexed by $\z_{t-1}$, which is constant from analyticity by (\ref{ass:gaussian_transition}).
    \item[(ii)] Thm.~\ref{thm:msm_transition_identifiability_ps}: As in Step \textbf{I} of Theorem~\ref{thm:affine_identifiability}, we construct a system of equations in terms of $Q(\z_{t-1})$ by stacking $K$ histories, such that under (\ref{ass:sticky_switch}--\ref{ass:posterior_domi}) the system has a unique solution.
\end{enumerate}

    Our proof strategy for (i) extends \citet{balsells-rodas25causal}, 
    by generalising the transitions to analytic functions, and makes the above result compatible with identifying latent components of the rSDS by combining with the previous affine identifiability results (Def. \ref{def:affine_identifiability}). I.e., there exists a permutation $\pi\in S_k$, and an affine transformation $\psi(\z)=A\z+\Bb$, such that:
    \begin{equation}\label{eq:affine_mean_result}
        \bm m_k(\z) = A \bm m'_{\pi(k)}(A^{-1}(\z - \Bb)) + \Bb, \  Q_{jk}(\z) = Q'_{\pi(j)\pi(k)}(A^{-1}(\z - \Bb)),\  k=1:K,\  \z\in\R^m.
    \end{equation}
%
For other rSDS designs, Thm.~\ref{thm:msm_transition_identifiability} is independent of Section~\ref{sec:first_stage} provided we establish identifiability up to affine transformations (Def~\ref{def:affine_identifiability}). For example, emission piecewise linearity \citep{kivva2022identifiability} could replace (\ref{ass:sticky_switch}-\ref{ass:nondegeneratus}), and Theorem \ref{thm:msm_transition_identifiability} would still be valid to identify the transitions $\bm m_k$ as in Eq.~\eqref{eq:affine_mean_result}. However, for full identifiability including $Q(\z_{t-1})$, further restrictions would be necessary.

\subsection{Disentanglement: Resolving the Affine Indeterminacy}\label{sec:stage_3}

Our last stage consists of resolving the structure of the affine mapping $\psi(\z)=A\z + \Bb$ in Eq.~\eqref{eq:affine_mean_result} to establish identifiable disentangled representations and permit nonlinear ICA \citep{hyvarinen1999nonlinear}. This identifiability, presented in Thm.~\ref{lemma:disentanglement}, requires the following additional assumption.

\begin{assumptionbox}
\begin{assumption}\label{ass:disentanglement}
For a pair of regimes $k_1\neq k_2\in\{1,\dots,K\}$, with diagonal variances $\mathrm{diag}(\bm \sigma^2_{k_1}), \mathrm{diag}(\bm\sigma^2_{k_2})$, respectively, define the ratio vector $\bm r_{k_1,k_2}$ between variances:
\begin{equation}\label{eq:kivva_ratio_vector}
\bm r_{k_1,k_2} := \left(\frac{\sigma_{k_1,1}}{\sigma_{k_2,1}}, \quad \dots, \quad\frac{\sigma_{k_1,m}}{\sigma_{k_2,m}}\right) = \frac{\bm \sigma_{k_1}}{\bm \sigma_{k_2}}\in\R^m.
\end{equation}
For every pair of regimes $k_1\neq k_2,\ k_1,k_2\in\{1,\dots,K\}$, stack ratios $\bm r_{k_1,k_2} \in\R^m$:
\begin{equation}\label{eq:vectorised_kivva}
    R\in\R^{\binom{K}{2}\times m}, \quad R_{(k_1,k_2),i}=\bm r_{k_1,k_2,i}, \quad i\in\{1,\dots,m\}.
\end{equation}
That is, each row of $R$ is a ratio vector formed by combining regime pairs. We assume that $R$ has no two identical columns, i.e., for $i\neq j\in\{1,\dots,m\}$, there exists $(a,b)$ such that  $R_{(a,b),i}\neq R_{(a,b),j}$.
\end{assumption}
\end{assumptionbox}
%
\begin{theorembox}
\begin{theorem}\label{lemma:disentanglement}
    Let $K\geq 2$. Let $\z'_t = A\z_t+\Bb$, where $A\in\R^{m\times m}$ is invertible and $\Bb\in\R^m$. Under Assumption~(\ref{ass:disentanglement}), from $\z'_t$, we can recover an invertible matrix $A'\in\R^{m\times m}$ such that $(A')^{-1}A = PD$, where $P$ is a permutation matrix, and $D$ is diagonal.
\end{theorem}
\end{theorembox}
See Appendix~\ref{app:disentanglement_proof} for the proof, which uses similar eigen-decomposition tricks introduced in the strategy of \citet{kivva2022identifiability}. We provide examples on how Eq.~\eqref{eq:vectorised_kivva} is constructed, and partial disentanglement when (\ref{ass:disentanglement}) is not satisfied. Thm.~\ref{lemma:disentanglement} shows that in regime-dependent settings, increasing the number of regimes with diverse variances strengthens identifiable latent space disentanglement.

\begin{remark}
For mixture models, \citet[Assumption~(P2)]{kivva2022identifiability} enables disentanglement via regime-dependent variances, which extends to the temporal case as follows. For $K$ regimes with diagonal covariances, i.e., $z_{t,i}\indep  z_{t,j} \mid s_t,\z_{t-1}$ for all $i\neq j \in\{1,\dots,m\}$, there exists one pair of regimes $k_1\neq k_2\in\{1,\dots,K\}$ such that the ratio vector $\bm r_{k_1,k_2}$ has distinct elements. 
However, this condition does not scale well to high dimensions, since ratios with close values may challenge disentanglement in practical estimation. We note (\ref{ass:disentanglement}) is weaker than Assumption~(P2) of \citet{kivva2022identifiability}, as a single ratio vector with distinct elements forms $R$ with distinct columns. Therefore, (\ref{ass:disentanglement}) works as a scalable generalisation by fully exploiting the mixture setting of rSDS.
\end{remark}

\section{$\Omega$SDS: Exact Likelihood Learning for SDSs}\label{sec:estimation}

The identifiability results suggest a natural estimation approach: use an invertible emission to distinguish temporal features from exogenous noise, and  perform exact inference on the resulting latent rMSM. We propose $\Omega$SDS, a likelihood-based estimator that follows this principle by combining an invertible flow-based emission model with latent-space inference based on rMSMs.
For simplicity, we drop the subscript $\mparam$ in the model equations. 
The flow defines a generative emission map from the latent dynamical variables and exogenous noise to observations, while its inverse acts as an encoder: $(\z_t, \bm \varepsilon_t) = f^{-1}(\x_t),\,  1\leq t  \leq T$. Using the change-of-variables formula, the marginal log-likelihood of an observed trajectory is expressed as follows.
\begin{equation}\label{eq:omega_sds_likelihood}
    \log p(\x_{1:T}) = \sum_{t=1}^T \log\left| \det J_{f^{-1}}(\x_t)\right|  + \log p(\z_{1:T}) + \log p_{\eps}(\bm \varepsilon_{1:T}), \quad p_{\eps}(\eps_t)=\N(\eps_t; 0,\sigma_{\eps}^2I)
\end{equation}
where $\sigma_{\eps}$ is left as a hyper-parameter, and $J_{f^{-1}}$ denotes the Jacobian of the inverse transformation. In our experiments, we parametrise $f$ using a multi-layer TarFlow \citep{zhai2025normalizing}, $f = (f^{(1)}\circ\dots\circ f^{(L)})$, where each layer is a masked autoregressive flow where the conditioning is implemented using causal transformers. To encourage a decomposition between dynamical variables and exogenous noise, we replace the permutation blocks in TarFlow with invertible LU mixing layers, as introduced in RealNVP \citep{dinh2017density}. For video experiments, we instead use a multi-scale flow architecture, inspired by Glow \citep{kingma2018glow}. See Appendix~\ref{app:architecture} for details.

The latent sequence likelihood $p(\z_{1:T})$ is defined by the rMSM prior with discrete regimes $\s_{1:T}$ in Eq.~\eqref{eq:transition_dynamics}. We optimise the likelihood in Eq.~\eqref{eq:omega_sds_likelihood} by gradient ascent. The gradient of the latent log-likelihood $\log p(\z_{1:T})$ can be written as an expectation under the exact switch posterior,  $\nabla_{\mparam}\log p(\z_{1:T}) = \nabla_{\mparam}\mathbb{E}_{p(\s_{1:T}\mid \z_{1:T})}[\log p(\z_{1:T}, \s_{1:T})]$, which yields the following \citep{dong2020collapsed}:
\begin{align*}
    \nabla_{\mparam}\log p(\z_{1:T}) &=  \sum_{k=1}^K \gamma_{1,k}(\z_{1:T})\nabla_{\mparam}\log p(s_1=k) + \sum_{t=2}^T \sum_{k=1}^K\sum_{l=1}^K \xi_{t,k,l}(\z_{1:T}) \nabla_{\mparam}\log Q_{lk}(\z_{t-1})
    \\
     +  \sum_{k=1}^K & \gamma_{1,k}(\z_{1:T}) \nabla_{\mparam} \log p(\z_1 |s_1=k) +  \sum_{t=2}^{T}\sum_{k=1}^K \gamma_{t,k}(\z_{1:T}) \nabla_{\mparam} \log p(\z_t |\z_{t-1}, s_t=k),
\end{align*}
where $\gamma_{t,k}(\z_{1:T})=p(s_{t}=k \mid \z_{1:T})$ and $\xi_{t,k,l}(\z_{1:T})=p(s_{t}=k, s_{t-1}=l \mid \z_{1:T})$ are the marginal and two-step posterior distributions given $\z_{1:T}$. These quantities are computed using standard forward-backward recursions \citep{bishop2006pattern}, where details are deferred to Appendix \ref{app:derivations_fb}.

Training regime-switching models with nonlinear transitions can suffer from state collapse, where only a subset of regimes is used \citep{dong2020collapsed}. Rather than relying on annealing schedules, we initialise the rMSM parameters from PCA features and use a PCA-alignment auxiliary loss early in training. This stabilises the flow decomposition and provides a reliable early latent representation for regime learning. PCA-based initialisation is standard in SDSs \citep{linderman2017bayesian}.

\paragraph{Consistency.}
Unlike VAE-based approaches, $\Omega$SDS optimises the exact marginal likelihood. By invertibility of the emission, $(\z_t,\eps_t)=\func^{-1}(\x_t)$, the latent $\pz(\z_{1:T})$ is an MSM likelihood with covariate-dependent transitions $Q(\z_{t-1})$, aligning with the setting of MLE consistency results of \citet{pouzo2022maximum}.
Therefore, under the usual assumptions for MSM and MLE consistency, including correct specification, global optimisation, positivity of the transition probabilities, ergodicity, and likelihood regularity, the resulting MLE favours consistency up to the equivalence class characterised in Defs.~\ref{def:affine_identifiability}, \ref{def:perm_identifiability}, and Theorem~\ref{lemma:disentanglement} \citep{wald1949note,pouzo2022maximum}. These regularity conditions are standard in likelihood-based consistency analyses and are additional to our identifiability assumptions, but compatible with them. However, the latent path in $\Omega$SDS depends on the flow parameters $\z_t=\func^{-1}(\x_t)_{1:m}$, and the flow, transition means, covariances, and switching process are estimated jointly. Therefore, a formal consistency statement would require further analysis, which we leave as future work. This contrasts with VAE-based SDSs, which optimise an ELBO and require sufficient coverage of the variational family \citep{khemakhem2020variational}. Collapsed-switching approaches marginalise $\s_{1:T}$ conditionally on $\z_{1:T}$, but rely on an amortised variational posterior $q_{\phi}(\z_{1:T}|\x_{1:T})$ which may introduce approximation gaps \citep{balsells-rodas2024on}. We further discuss this empirically in Appendix~\ref{app:vae_estimation}.


\section{Experiments}\label{sec:experiments}

We evaluate $\Omega$SDS through synthetic data with controlled ground-truth factors, 2D bouncing ball videos, and real-world dancing videos. Code for our implementation and experiments can be found in \url{https://github.com/charlio23/omega-SDS}. Appendix \ref{app:assumption_aligmnent} discusses the alignment between our assumptions and the synthetic datasets.

\subsection{Synthetic Data}\label{sec:synthetic_exp}

To assess the advantages of $\Omega$SDS over VAE-based alternatives, we follow the settings in iSDS \citep{balsells-rodas2024on}, using a non-recurrent $Q$, piecewise-linear emissions, and two dimensionality cases: noiseless observations ($n=m$), and noise-augmented observations ($n=5m$). We compare $\Omega$SDS with iSDS, NCTRL \citep{song2023temporally}, CtrlNS \citep{song2024causal}, TDRL \citep{yao2022temporally}, and iVAE \citep{khemakhem2020variational} (with ground-truth regimes as auxiliary variables). Appendix~\ref{app:synthetic_data} provides data-generation details, ablations on Assumption~(\ref{ass:disentanglement}), and identification results for the number of regimes $K$ and recurrent $Q(\z)$.

\begin{figure}[t]
    \centering
    \includegraphics[width=0.95\linewidth]{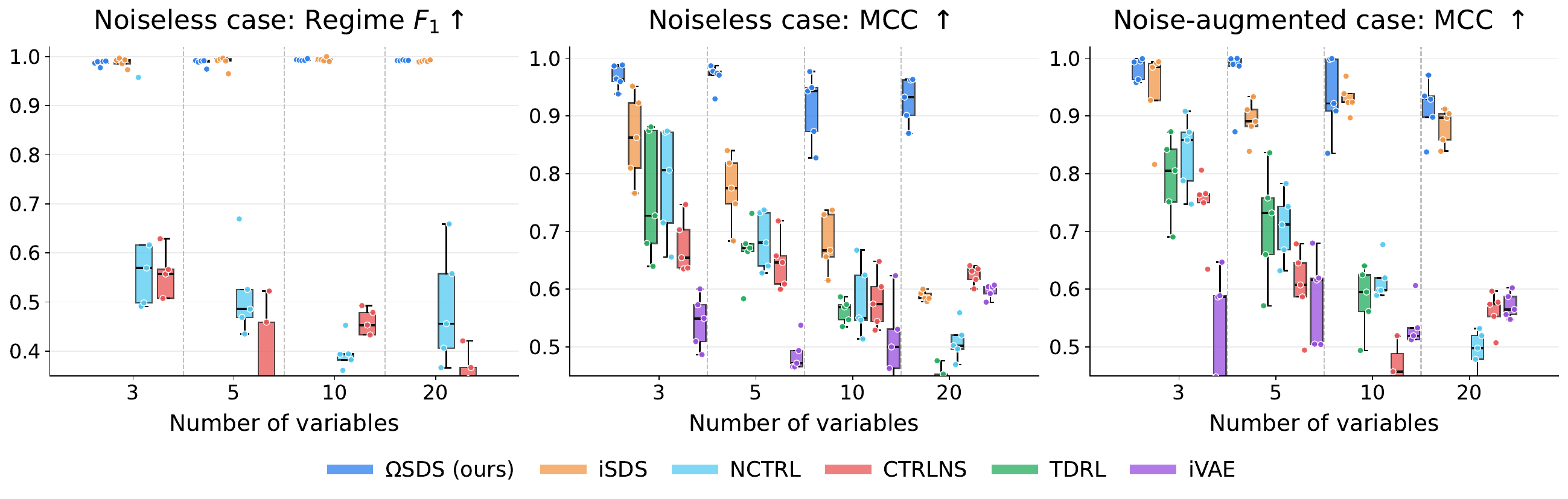}
    \vspace{-.75em}
    \caption{
    Synthetic data results under increasing dimensionality. 
    Left: Regime $F_1$ score ($\Omega$SDS and iSDS achieve near perfect recovery). 
    Middle and right: MCC of noiseless and noise-augmented settings. Higher is better for both $F_1$ and MCC. Reported results use 5 different dataset seeds.
    }
    \vspace{-1em}
    \label{fig:synthetic_increasing_dimensions}
\end{figure}

Figure~\ref{fig:synthetic_increasing_dimensions} reports switch recovery $F_1$ in the noiseless setting and latent recovery MCC in both settings. As expected, $\Omega$SDS recovers regimes comparably to iSDS, since the data is compatible with a non-recurrent SDS. The main difference appears in latent recovery. In the noiseless setting, $\Omega$SDS maintains consistently high MCC scores across dimensions, while VAE-based approaches degrade with increasing latent dimensionality. This suggests that $\Omega$SDS provides a more scalable approach to disentanglement.
In the noise-augmented setting, VAE-based methods improve, possibly due to the additional observed constraints. However, $\Omega$SDS remains competitive in both settings. Overall, these results show that $\Omega$SDS provides a scalable alternative to VAE-based SDS methods for disentanglement, even in non-recurrent settings where existing baselines are directly applicable.



\subsection{Video Experiments}\label{sec:videos}

\begin{figure}[t]
\centering

\begin{minipage}{0.63\linewidth}
\centering
\captionof{table}{Regime estimation and long-rollout quality. Regime $F_1$ is computed on observed input sequences ($T=64$) and on predicted regimes from $t=65$ to $t=100$. LPIPS is computed on 5-frame-context rollouts. Higher is better for $F_1$; lower is better for LPIPS. We highlight the \textbf{best} and \underline{second-best} methods.}
\centering
\label{tab:bb32_regime_rollout}
{\setlength{\tabcolsep}{3pt}
\resizebox{\linewidth}{!}{
\begin{tabular}{lcccc}
\toprule
Method
& Input $F_1 \uparrow$
& Pred. $F_1 \uparrow$
& LPIPS@20 $\downarrow$
& LPIPS@500 $\downarrow$ \\
\midrule
VRNN
& -- 
& -- 
& 0.014
& \underline{0.028} \\
KVAE
& 0.105
& 0.105
& 0.053
& 0.324 \\
iSDS
& 0.819
& 0.345
& 0.111
& 0.106 \\
REDSDS
& 0.655
& 0.580
& 0.073
& 0.135 \\
SNLDS
& \underline{0.998}
& \underline{0.913}
& \underline{0.007}
& 0.059 \\
\midrule
$\Omega$SDS (A)
& 0.867
& 0.362
& 0.110
& 0.092 \\
$\Omega$SDS (R)
& \textbf{1.000}
& \textbf{0.932}
& \textbf{0.003}
& \textbf{0.027} \\
\bottomrule
\end{tabular}}}
\end{minipage}
\hfill
\begin{minipage}{0.35\linewidth}
\centering
\includegraphics[width=0.9\linewidth]{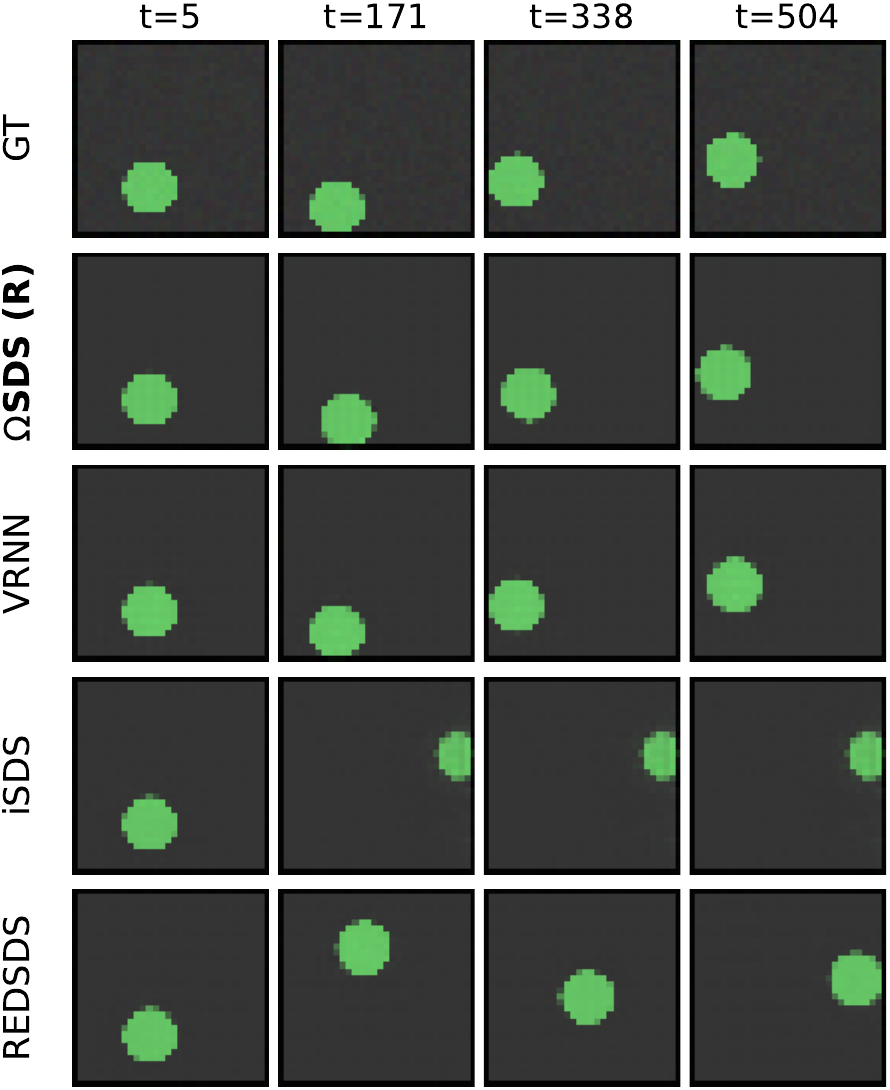}
\vspace{-0.3em}
\captionof{figure}{Long forecast example.}
\label{fig:long_rollout}
\end{minipage}
\vspace{-2em}
\end{figure}

\paragraph{2D Bouncing Ball Videos.}
We synthesise videos of a ball bouncing inside a 2D square box, starting from random positions and one of four velocity modes, $k\in\{\nearrow, \nwarrow, \searrow, \swarrow\}$. 
This benchmark suits $\Omega$SDS and showcases our theoretical results, where the discrete regimes determine the direction of motion, and switches depend on the continuous states through wall collisions. In contrast, previous identifiable SDS models (e.g., iSDS; \citealp{balsells-rodas2024on}) cannot incorporate recurrent feedback from continuous states into the switching process. We train on $10k$ videos of length $T=64$, with noisy coloured $32\times 32$ frames. See Appendix \ref{appendix:bouncing_ball} for details on data generation and training.

We compare recurrent and non-recurrent variants, $\Omega$SDS (R) and $\Omega$SDS (A), against iSDS and other baselines that are not designed around identifiability: VRNN, KVAE, SNLDS, and REDSDS. All switching methods use $K=4$, while VRNN is a no-switching recurrent baseline. 
Table~\ref{tab:bb32_regime_rollout} reports regime $F_1$ scores on held-out observed sequences of length $T=64$, and on predictions from $t=65$ to $t=100$. We also report LPIPS errors for long-horizon forecasts at $20$ and $500$ steps, using a context of $5$ observed frames. Figure~\ref{fig:long_rollout} shows a representative $500$-step rollout. The recurrent model, $\Omega$SDS (R), achieves the strongest regime estimation and prediction performance, while also obtaining the best long-horizon perceptual quality. SNLDS is the closest competitor, since it is a VAE-based design with rMSMs in latent space. Crucially, both recurrent switching methods are competitive with VRNN, a memory-based sequence model with unidentifiable latent dynamics. The comparison between $\Omega$SDS (R), $\Omega$SDS (A), and iSDS highlights the role of recurrent switches. While non-recurrent switching can recover regimes on observed sequences, it struggles to predict future dynamics. Additional evaluation and visualisations are provided in Appendix~\ref{appendix:bouncing_ball}. Overall, these results show that enforcing identifiable latent dynamics does not sacrifice predictive performance, while providing an interpretable decomposition of complex dynamics into simpler behavioural modes.

\begin{figure}
    \centering
    \includegraphics[width=0.9\linewidth]{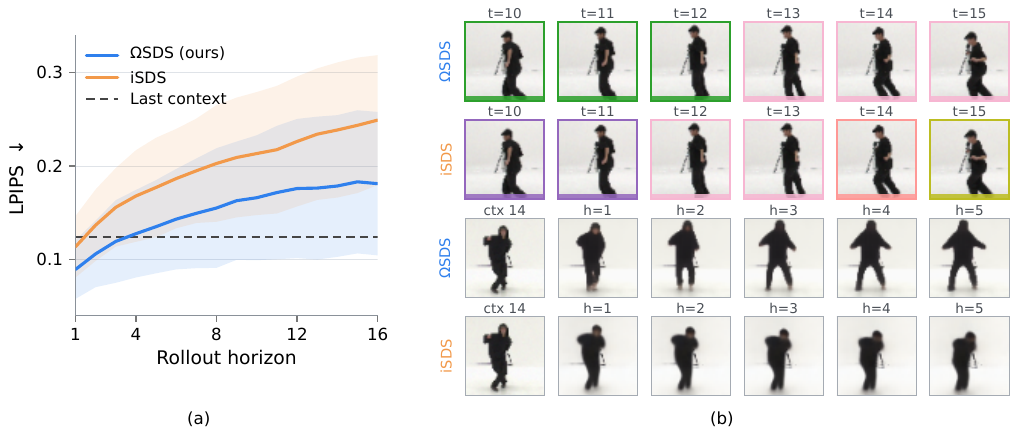}
    \vspace{-.5em}
    \caption{AIST dance video results. (a) Mean LPIPS error of predicted frames from $15$ context frames across $100$ held-out sequences, with shaded regions indicating one standard deviation. For reference, ``last context'' displays the average LPIPS error in terms of the final context frame. (b) Segmentation and forecasting visualisations, where colours indicate distinct inferred regimes.}
    \label{fig:aist_mainpaper}
    \vspace{-.25em}
\end{figure}

\paragraph{Dancing Videos.} To showcase our identifiability results in real-world scenarios, we explore the dance video dataset of \citet{aist-dance-db} and focus on segmenting and forecasting sequences from the \textit{Middle Hip Hop} genre. See Appendix~\ref{app:dancing_videos} for details, along with additional visualisations. Figure~\ref{fig:aist_mainpaper} shows results of $\Omega$SDS in contrast to the non-recurrent identifiable approach iSDS \citep{balsells-rodas2024on}. We (a) report LPIPS as a function of rollout time, and (b) compare posterior regime estimation (top two rows) and rollout predictions with 15 context frames (bottom two rows). $\Omega$SDS improves rollout prediction and yields more interpretable dynamics by decomposing dance patterns, such as ``move up'' or ``move down'' in the shown example. The generation quality of $\Omega$SDS is substantially better than iSDS, thanks to the use of recurrent dynamics and flow-based emissions. Although $\Omega$SDS does not predict the exact future dance sequence, it produces more plausible movements and therefore achieves superior forecasting performance.




\section{Conclusions}\label{sec:conclusions}

We presented identifiability results for recurrent switching nonlinear dynamical systems by imposing restrictions on the latent dynamics, recovering the latent representation up to acceptable equivalences under general invertible emissions. 
We also introduced $\Omega$SDS, a flow-based estimator that uses exact-likelihood optimisation, avoiding the approximation gap introduced by VAE-based estimators. Empirically, our method improves disentanglement over variational alternatives and shows strong forecasting performance due to recurrent switching.

\paragraph{Limitations.}
Our results require constraints on both the emission and latent dynamics. While standard, invertibility excludes information-losing effects such as occlusions. On the other hand, the smoothness in (\ref{ass:gaussian_transition}), the likelihood-separation condition (\ref{ass:posterior_domi}), and the geometric condition in the equal-variance case (\ref{ass:regime_domi}) restrict the class of admissible nonlinear dynamics, preventing arbitrary parametrisations. Still, likelihood separation can be balanced with sticky switching (\ref{ass:sticky_switch}). 
Furthermore, while $\Omega$SDS avoids the variational approximation gap, consistency remains open, since consistency of rMSMs alone requires further assumptions from covariate-dependent MSMs \citep{pouzo2022maximum}, and the joint flow-rMSM parameter estimation requires further analysis.
Experimentally, the AIST dataset does not provide ground-truth regimes, preventing quantitative evaluation of the learned switching dynamics. For such real-world scenarios where the latent space is unknown, assumption verification is a fundamental limitation of identifiability theory, which restricts interpretability of the learned structures. Therefore, to extend these guarantees to real-world applications, future work will focus on weaker separation conditions, robustness to model misspecification, and extensions of likelihood-based models to broader nonlinear ICA and causal representation learning settings.

\section*{Acknowledgements}

C.~B.-R., X.~S., and Z.~X. thank the Department of Computing at Imperial College London for PhD scholarships. We also thank reviewers of NeurIPS 2026 for pointing out minor errors in proofs for a few lemmas during the submission phase.

\bibliographystyle{plainnat}
\bibliography{bibtex}


\appendix
\newpage

\begin{center}
  \LARGE
  \textbf{Appendix for ``End-to-End Identifiable \\ Recurrent Switching Dynamical Systems''}
\end{center}

\etocdepthtag.toc{mtappendix}
\etocsettagdepth{mtchapter}{none}
\etocsettagdepth{mtappendix}{subsection}
{\small \tableofcontents}

\section{Unidentifiability of Non-Temporal Mixture Generative Models}\label{app:unidentifiability}

\citet{khemakhem2020variational} illustrates identifiability for unconditional priors. This notion also extends to models with mixture priors, where the latent variables are sampled from a mixture distribution:
\begin{equation}
   \x = \func(\z), \quad \z\sim \sum_{k=1}^K p(s=k)p(\z \mid s=k).
\end{equation}
In this setting, even the number of components $K$ is unidentifiable from the observed data, which leads to an unidentifiable latent representation. To see this, consider two models with transformations $f_{\mparam}$ and $f_{\mparam'}$, and latent distributions with $K$ and $K'$ mixture components, respectively, such that $p_{\mparam}(\x)=p_{\mparam'}(\x)$ for any  $\x\in\R^n$. We can write the observed density as a pushforward measure of $\z$ through $\func$ (or $f_{\mparam'_f}$): 
\begin{equation}
    (f_{\mparam\#} p_{\mparam_z})(\x) = (f_{\mparam'\#} p_{\mparam'_z})(\x),
\end{equation}
where the composition $\func^{-1} \circ f_{\mparam'_{f}}$ can transform one latent distribution into another. Since a general invertible transformation can map any distribution to a uniform via its CDF, the composition is flexible enough to arbitrarily create or destroy mixture components. This makes $K$ unidentifiable under the following construction:
\begin{equation}
    p_{\mparam}(\x) = \left((f^{-1}_{\mparam}\circ f_{\mparam'})_\# p_{\mparam_z'}\right)(\x) , \quad \text{with } (f^{-1}_{\mparam}\circ f_{\mparam'})  = F^{-1}_{p_{\mparam_z'}}\circ F_{p_{\mparam_z}},
\end{equation}
where $F^{-1}_{p_{\mparam_z'}}$ and $F_{\mparam_z}$ are the inverse CDF and the CDF of $p_{\mparam_z'}$ and $p_{\mparam_z}$, respectively.
To solve the above problem, \citet{kivva2022identifiability} proves identifiability of mixture priors under the assumption of piecewise linear $f_{\mparam}$. This imposes restrictions on $f_{\mparam}$, but enables identifiability while maintaining a flexible GMM prior. Comparing to \citet{khemakhem2020variational}, which instead constrains the prior, this highlights a ``no free lunch'' principle for identifiability: without restricting the generative mapping or the latent distribution, the model is fundamentally unidentifiable.

\section{Proofs and Extended Discussions on Identifiability}\label{sec:appendix_proof}

\subsection{Proof of Theorem \ref{thm:affine_identifiability}}\label{app:main_proof}

\begin{proof}

\textbf{Notation.} Consider two models $\mparam,\mparam'\in\bm\Theta$ with emission maps $f_{\mparam_f}$ and $f_{\mparam'_f}$, respectively. For simplicity, write $f := f_{\mparam_f},\ f' := f_{\mparam'_f}$. For the first model, define $\w_t = (\z_t, \eps_t) = f^{-1}(\x_t)$, $\z_t = f^{-1}(\x_t)_{1:m}$. For the second model, define  $\w'_t=(\z'_t,\eps'_t)= f'^{-1}(\x_t)$, $\z'_t = f'^{-1}(\x_t)_{1:m'}$.
Denote by $g_k(\z_{t+1}\mid \z_{t}) := \mathcal{N}(\z_{t+1}; \bm{m}_k(\z_{t}), \Sigma_k)$ the $k-$th latent regime transition density, and $\tilde{g}_k(\w_{t+1}\mid \z_{t}) = g_k(\z_{t+1}\mid \z_{t})p_{\eps}(\eps_{t+1})$. Since $\eps_{t+1}$ is independent of $\w_t$, we write $\tilde{g}_k(\w_{t+1}\mid \z_{t})$ instead of $\tilde{g}_k(\w_{t+1}\mid \w_{t})$. Define the pushforward measure of $\tilde{g}_k$ to observation space as $\phi_k(\cdot\mid\x_t) = f_\# \tilde{g}_k(\cdot\mid\z_t)$ with $\z_t=f^{-1}(\x_t)_{1:m}$, and define $g'_{k'}$, $\tilde{g}'_{k'}$, and $\phi'_{k'}$ analogously for the second model, using $f'$ and $p'_{\eps}$.

\textbf{Step I.} For two parametrisations $\mparam\neq \mparam'$ such that $p_{\mparam}(\x_{1:T})=p_{\mparam'}(\x_{1:T})$, their marginal densities on $\x_{1:t}$ are equal for every $t\leq T$. Using conditional densities implies
\begin{equation}
    p_{\mparam}(\x_{t+1}\mid\x_{1:t}) = p_{\mparam'}(\x_{t+1}\mid\x_{1:t}),
\end{equation}
for every $1\leq t < T$ and almost any $\x_{1:t}\in\R^{tn}$.

Now consider model $\mparam$, and fix a history $\bm{h}=\x_{1:t-1}$. Since $f$ is invertible, conditioning on $(\x_t, \bm h)$ is equivalent to conditioning on $(\z_t, \eps_t, \z_{1:t-1}, \eps_{1:t-1})$ through $f^{-1}$. Since the temporal dynamics are independent of the exogenous noise, the conditional distribution of $\bm w_t$ depends on the history only through $(\z_t, \tilde{\bm h})$, where $\z_t=f^{-1}(\x_t)_{1:m}$ and $\tilde{\bm h}=\z_{1:t-1}$. Therefore,
\begin{align}
p_{\mparam}(\x_{t+1}\mid \x_t, \bm h)
&=
\left|\det J_{f^{-1}}(\x_{t+1})\right|
\,p_{\mparam}(\w_{t+1}\mid \z_t,\tilde{\bm h})\\
&=
\left|\det J_{f^{-1}}(\x_{t+1})\right|
\sum_{k=1}^K p_{\mparam}(s_{t+1}=k\mid \z_t,\tilde{\bm h})\,
\tilde{g}_k(\w_{t+1}\mid \z_t)\\
&=
\sum_{k=1}^K p_{\mparam}(s_{t+1}=k\mid \z_t,\tilde{\bm h})\,
\phi_k(\x_{t+1}\mid \x_t).
\end{align}
Denote $\Phi(\x_{t+1}|\x_{t}):=(\phi_1(\x_{t+1}|\x_{t}), \dots, \phi_K(\x_{t+1}|\x_{t}))^\top$, and let $\bm H=(\tilde{\bm h}^{(1)},\dots,\tilde{\bm h}^{(K)})$ be a collection of $K$ observed histories, where $\bm h^{(i)}=\x_{1:t-1}^{(i)}$. For model $\mparam$, denote $\tilde{\bm h}^{(i)}$ the $i$-th latent history under $f^{-1}$, and define:
\begin{equation}
    \bm P_{\mparam,\bm H}(\x_{t+1}|\x_{t}) := \begin{pmatrix}
        p_{\mparam}(\x_{t+1} | \x_{t},\bm h ^{(1)}) \\
        \vdots \\
        p_{\mparam}(\x_{t+1} | \x_t,\bm h ^{(K)})
    \end{pmatrix}.
\end{equation}
Also define
\begin{equation}
    A_{\bm H}(\z_t):= \begin{pmatrix}
        p_{\mparam}(s_{t+1}=1 \mid \z_t,\tilde{\bm h}^{(1)}) & \dots &   p_{\mparam}(s_{t+1}=K \mid \z_t,\tilde{\bm h}^{(1)}) \\
        \vdots & \ddots & \vdots \\
        p_{\mparam}(s_{t+1}=1 \mid \z_t,\tilde{\bm h}^{(K)}) & \dots & p_{\mparam}(s_{t+1}=K \mid \z_t,\tilde{\bm h}^{(K)})
    \end{pmatrix}.
\end{equation}
Then, we can write the observed conditional densities as mixtures over shared one-step observational components $\Phi(\x_{t+1} | \x_{t})$, with mixture weights indexed by $\bm H$:
\begin{equation}
        \bm P_{\mparam,\bm H}(\x_{t+1}|\x_{t})
    = A_{\bm H}(\z_t)
    \Phi(\x_{t+1} | \x_{t}).
\end{equation}
For two parametrisations $\mparam\neq \mparam'$ where $p_{\mparam}(\x_{t+1}\mid \x_{1:t}) = p_{\mparam'}(\x_{t+1}\mid \x_{1:t})$, under the same set of histories $\bm H$, we define $\bm P_{\mparam',\bm H}(\x_{t+1}|\x_{t})$ and $A'_{\bm H}(\z'_t)$ analogously to $\bm P_{\mparam,\bm H}(\x_{t+1}|\x_{t})$ and $A_{\bm H}(\z_t)$, respectively. Note we use the same observed histories $\bm H$ for both models, which are model-dependent through $f^{-1}$ and
$(f')^{-1}$. Then, under $p_{\mparam}(\x_{1:T})=p_{\mparam'}(\x_{1:T})$, we have
\begin{equation}\label{eq:equivalence_identity}
    A_{\bm H}(\z_t)
    \Phi(\x_{t+1} | \x_{t}) = \bm P_{\mparam, \bm H}(\x_{t+1}|\x_{t}) = \bm P_{\mparam',\bm H}(\x_{t+1}|\x_{t})=  A'_{\bm H}(\z'_t)\Phi'(\x_{t+1} | \x_{t}).
\end{equation}
We note that, for the same collection of $K$ observed histories $\bm H$ and a model $\mparam'$ with $K'$ regimes, the matrix $A'_{\bm H}$ has size $K\times K'$, while $\Phi'(\x_{t+1} | \x_{t})$ has length $K'$.

We now show that $A_{\bm H}(\z_t)$ is invertible and continuous on a nonzero measure set in $\R^m$. By first-order Markov switching dynamics, row $i$ of $A_{\bm H}(\z_t)$ is the one-step predictive distribution of $s_{t+1}$ given the latent history $\tilde{\bm h} ^{i}$ ending at $\z_{t}$. Therefore, $A_{\bm H}(\z_t)$ factorises as follows:
\begin{equation}
    A_{\bm H}(\z_t)_{ik} = p_{\mparam}(s_{t+1}=k \mid \z_t,\tilde{\bm h}^{(i)}) = \sum_{j=1}^K p_{\mparam}(s_t=j\mid \z_t,\tilde{\bm h}^{(i)}) Q_{jk}(\z_t),
\end{equation}
\begin{equation}
    A_{\bm H}(\z_t) = F_{\bm H}(\z_t)Q(\z_t),
\end{equation}
where row $i$ of $F_{\bm H}(\z_t)$ is the filtering posterior $(p_{\mparam}(s_t=1\mid \z_t, \tilde{\bm h}^{(i)}), \dots , p_{\mparam}(s_t=K\mid \z_t, \tilde{\bm h}^{(i)}) )$ ending at $\z_t$. We note that exogenous noise terms disappear from the above equation due to conditioning. 

From (\ref{ass:sticky_switch}--\ref{ass:posterior_domi}), there exists a nonzero measure set $\mathcal{Z}$ on which we can choose $\tilde{\bm h}^{(i)}=(\z_1^{(i)}, \dots, \z_{t-1}^{(i)})$, with $\z_{\tau}^{(i)}\in\mathcal{Z}$, such that Lemma \ref{lemma:concentration_to_floor} gives $p(s_{t}= i | \z^*_t,\tilde{\bm h}^{(i)}) >  \frac{1}{2}$ for large enough $t>1$ and some $\z^*_t\in\mathcal{Z}$. Using Lemma \ref{lemma:concentration_to_floor} for $i\in\{1,\dots,K\}$, we can order histories such that $F_{\bm H}(\z^*_t)_{ii} > \frac{1}{2}$, i.e., history $i$ dominates the filtering distribution on regime $i$. Since each row of $F_{\bm H}(\z^*_t)$ sums to one, this implies strict row diagonal dominance, and therefore, $F_{\bm H}(\z^*_t)$ is invertible by the Levy--Desplanques theorem \citep{wu2005bounds}. Similarly by (\ref{ass:sticky_switch}), $Q_{kk}(\z_t)>1/2$ for $\z_t\in\mathcal Z$, which means $Q(\z_t)$ is also strictly row diagonally dominant and
invertible. Therefore,  $A_{\bm H}(\z^*_t)$ is invertible. Now we extend invertibility to a nonzero measure set $\mathcal{U}\subset\R^m$. For each fixed history $\tilde{\bm h}^{(i)}$ with a common endpoint $\z^*_t$, define
\begin{equation}
    M_i(\z^*_t) = \log g_i(\z^*_t \mid \z_{t-1}^{(i)}) - \max_{j\neq i}\log g_j(\z^*_t \mid \z_{t-1}^{(i)}).
\end{equation}
The functions in $M_i$ are continuous given Gaussianity, and since $M_i(\z^*_t)> \ell^*$, there exists a neighbourhood $\mathcal U_i\subset \mathcal{Z}$ of $\z^*_t$ such that 
\begin{equation}\label{eq:margin}
    M_i(\z_{t})> \ell^*, \quad \forall\z_t\in \mathcal U_i, \ \forall i\in\{1,\dots,K\}.
\end{equation}
Let $\mathcal U = \cap_{i=1}^K \mathcal U_i,$ which is a non-empty set since each $\mathcal U_i$ is a neighbourhood of the same endpoint $\z^*_t$. Then from Eq.~\eqref{eq:margin}, all $K$ histories satisfy the same conditions of Lemma~\ref{lemma:concentration_to_floor} for every $\z_t\in\mathcal{U}$. Therefore, the same strict diagonal dominance argument applies, and $F_{\bm H}(\z_t)$ is invertible for any $\z_t\in \mathcal U$. The filtering equation involves finite sums with Gaussians $g_k(\z_t\mid \z_{t-1})$, which implies $F_{\bm H}(\z_t)$ is continuous on $\mathcal{U}$. Since $Q(\z_t)$ is continuous on $\mathcal{Z}$, and $\mathcal{U}\subset \mathcal{Z}$, it follows that $A_{\bm H}(\z_t)$ is invertible and continuous for any $\z_t\in \mathcal{U}$, where $\mathcal{U}$ has nonzero measure.

\paragraph{Case (ii): Unknown $K$ and analytic $Q(\z_{t})$.}
The previous argument gives a nonzero measure set $\mathcal{U}\subset\R^m$ where $F_{\bm H}(\z_t)$ is invertible. This implies $\det F_{\bm H}(\z_t)$ is not identically zero. For fixed histories $\bm H$, the filtering probabilities are finite sums and products of Gaussians, and therefore $\det F_{\bm H}(\z_t)$ is analytic. From \citet{mityagin2015zero}, its zero set has measure zero, and this implies $F_{\bm H}(\z_t)$ is invertible almost
everywhere. When assuming $Q(\z_{t})$ is analytic, the same argument applies. By (\ref{ass:sticky_switch}), $Q(\z_{t})$ is strictly row diagonally dominant, and hence invertible on the non-empty open set $\mathcal{Z}$. Therefore, $\det Q(\z_{t})$ is analytic and not identically zero, which implies $Q(\z_{t})$ is invertible almost everywhere. Given the above, we have that $A_{\bm H}(\z_t) = F_{\bm H}(\z_t)Q(\z_t)$ is invertible for almost any $\z_t\in\R^m$.

\textbf{Step II.} On the nonzero set $\mathcal{U}$ where $A_{\bm H}(\z_t)$ is invertible, the identity from Eq.~\eqref{eq:equivalence_identity} gives
\begin{equation}\label{eq:component_relation}
    \Phi(\x_{t+1}\mid \x_t)
    =
    \Pi_{\bm H}(\x_t)
    \Phi'(\x_{t+1}\mid \x_t),
    \qquad
    \Pi_{\bm H}(\x_t):=A_{\bm H}(\z_t)^{-1}A'_{\bm H}(\z'_t),
\end{equation}
where $\z_t=f^{-1}(\x_t)_{1:m}$ and
$\z'_t=f'^{-1}(\x_t)_{1:m'}$. To clarify, for model $\mparam'$ with $K'$ regimes,
$A'_{\bm H}(\z'_t)\in\R^{K\times K'}$ and
$\Pi_{\bm H}(\x_t)\in\R^{K\times K'}$.
Integrating over $\x_{t+1}$ gives:
\begin{equation}
        \int\Phi(\x_{t+1} | \x_{t}) d\x_{t+1}
    = \int \Pi_{\bm H}(\x_{t}) \Phi'(\x_{t+1} | \x_{t}) d\x_{t+1} \implies \Pi_{\bm H}(\x_{t})\bm 1  = \bm 1.
\end{equation}
We can also show $\Pi_{\bm H}(\x_{t})$ is full row rank using linear independence properties of the Gaussian family. We start by writing a linear independence equation for the rows of $\Pi_{\bm H}(\x_t)$ in vector form:
\begin{equation}
    c^\top \Pi_{\bm H}(\x_t)=0,
    \qquad c\in\R^K .
\end{equation}
where $c=0$ if and only if the rows of $\Pi_{\bm H}(\x_{t})$ are linearly independent. Multiplying by $\Phi'(\x_{t+1}\mid\x_t)$ and using
\eqref{eq:component_relation}, we obtain, for every $\x_{t+1}\in\R^n$,
\begin{multline*}
     0 = \big(c^\top\Pi_{\bm H}(\x_t)\big)\Phi'(\x_{t+1}|\x_t) = c^\top\big(\Pi_{\bm H}(\x_t)\Phi'(\x_{t+1}|\x_t) \big)  \\ =c^\top\Phi(\x_{t+1}|\x_t) =  \sum_{k=1}^K c_k\phi_k(\x_{t+1}|\x_t),
\end{multline*}
where since the Gaussian family is linearly independent \citep{yakowitz1968identifiability} and $f$ is injective, their
pushforwards $\{\phi_1(\cdot\mid\x_t),\dots,\phi_K(\cdot\mid\x_t)\}$ are also linearly independent. Thus $c=0$, which implies $\Pi_{\bm H}(\x_{t})$ is full row rank. If model $\mparam$ has $K$ regimes, then $K\leq K'$. At this stage, the proof branches according to (i) known $K$, and (ii) unknown $K$, but analytic $Q(\z_{t-1})$.
\begin{itemize}[leftmargin=2em]
    \item[(i)] For known $K$,  $\Pi_{\bm H}(\x_t)\in\R^{K\times K}$ is square and, since it has full row rank, it is invertible.
    \item[(ii)] Unknown $K$ but analytic $Q$: By the analytic extension argument on the last paragraph of Step \textbf{I},  $A_{\bm H}(\z_t)$ is invertible almost everywhere. We now apply the same argument to model $\mparam'$ using (\ref{ass:sticky_switch}--\ref{ass:posterior_domi}) and Lemma~\ref{lemma:concentration_to_floor}. Therefore, we choose a different collection of $K'$ observed histories $\bm H'$ such that $A'_{\bm H'}(\z'_t)$ is invertible almost everywhere. Given $f$ and $f'$ are invertible, we can choose the same $\x_t$ such that both $A'_{\bm H'}(\z'_t)$ and $A_{\bm H}(\z_t)$ are invertible. For this $\x_t$, we have
\begin{equation}
    \Phi'(\x_{t+1}\mid\x_t)
    =
    \Pi'_{\bm H'}(\x_t)\Phi(\x_{t+1}\mid\x_t),
    \qquad
    \Pi'_{\bm H'}(\x_t)
    :=
    \big(A'_{\bm H'}(\z'_t)\big)^{-1}A_{\bm H'}(\z_t),
\end{equation}
where $\Pi'_{\bm H'}(\x_t)\in\R^{K'\times K}$. Under the same linear independence arguments on model $\mparam'$, $\Pi'_{\bm H'}(\x_t)\in\R^{K'\times K}$ is full-row rank. Therefore, $K'\leq K$. Combined with the previous inequality $K\leq K'$, this yields $K=K'$ and hence $\Pi_{\bm H}(\x_t)$ is square and invertible.
\end{itemize}
Now we show $\Pi_{\bm H}(\x_t)$ is a permutation. Under (\ref{ass:gaussian_transition}) and (\ref{ass:regime_domi}) for model $\mparam'$, by Proposition \ref{prop:sufficient_conditions_nonnegativity} and Lemma \ref{lemma:nonnegative_entries}, each element in $\Pi_{\bm H}(\x_t)$ is nonnegative. Similarly, using the same result on model $\mparam$, and using
\begin{equation}
    \Phi'(\x_{t+1} | \x_{t}) = \Pi_{\bm H}(\x_t)^{-1} \Phi(\x_{t+1} | \x_{t}),
\end{equation}
each element in $\Pi_{\bm H}(\x_t)^{-1}$ is also nonnegative. Furthermore, integrating over $\x_{t+1}$ again gives $\Pi_{\bm H}(\x_t)^{-1}\bm 1=\bm 1$. Therefore, we know
\begin{equation}
     \Pi_{\bm H}(\x_t) \geq 0, \quad \Pi_{\bm H}(\x_t)^{-1} \geq 0 , \quad \Pi_{\bm H}(\x_t)\bm 1 = \bm 1.
\end{equation}
We prove $\Pi_{\bm H}(\x_t)$ has at most one positive entry by contradiction. Take any row $i$, and assume it has two positive entries in columns $j\neq j'$
\begin{equation}
    \Pi_{\bm H}(\x_t)_{ij} > 0, \quad\Pi_{\bm H}(\x_t)_{ij'} > 0.
\end{equation}
Given $\Pi_{\bm H}(\x_t)\Pi_{\bm H}(\x_t)^{-1} = I_{K}$, for some $k\neq i$ we have:
\begin{equation}
    0 = \left(\Pi_{\bm H}(\x_t)\Pi_{\bm H}(\x_t)^{-1}\right)_{ik} = \sum_{j=1}^{K}\Pi_{\bm H}(\x_t)_{ij}\Pi_{\bm H}(\x_t)^{-1}_{jk}.
\end{equation}
Given both matrices have nonnegative entries, for $\Pi_{\bm H}(\x_t)_{ij}>0$ and $\Pi_{\bm H}(\x_t)_{ij'}>0$, we have $\Pi_{\bm H}(\x_t)^{-1}_{jk}=0$ and $\Pi_{\bm H}(\x_t)^{-1}_{j'k}=0$ for every $k\neq i$. Therefore rows $j$ and $j'$ of $\Pi_{\bm H}(\x_t)^{-1}$ can only have nonnegative entries on column $i$. Since $\Pi_{\bm H}(\x_t)^{-1}$ is invertible, neither row can be zero, therefore 
\begin{equation}
    \Pi_{\bm H}(\x_t)^{-1}_{ji}>0, \quad \Pi_{\bm H}(\x_t)^{-1}_{j'i}>0.
\end{equation}
Now consider  $\Pi_{\bm H}(\x_t)^{-1}\Pi_{\bm H}(\x_t) = I_{K}$:
\begin{equation}
    0 = \left(\Pi_{\bm H}(\x_t)^{-1}\Pi_{\bm H}(\x_t)\right)_{jj'} = \sum_{k=1}^{K}\Pi_{\bm H}(\x_t)^{-1}_{jk}\Pi_{\bm H}(\x_t)_{kj'} = \Pi_{\bm H}(\x_t)^{-1}_{ji}\Pi_{\bm H}(\x_t)_{ij'}.
\end{equation}
But $\Pi_{\bm H}(\x_t)^{-1}_{ji}>0$ implies $\Pi_{\bm H}(\x_t)_{ij'}=0$, which is a contradiction.
Therefore,  every row of $\Pi_{\bm H}(\x_t)$ has at most one positive entry, and since it is invertible, every column has also at most one positive entry. Finally, the condition $\Pi_{\bm H}(\x_t)\bm 1 = \bm 1$, forces $\Pi_{\bm H}(\x_t)$ to be a permutation matrix, as each nonzero entry must be $1$. 

Finally, we show  $\Pi_{\bm H}(\x_t)$ is locally constant by continuity. Denote by $\mathcal{V}\subseteq\R^n$ a nonzero measure set of points $\x_t$ such that $\z_t=f^{-1}(\x_t)_{1:m}\in\mathcal U$, where $A_{\bm H}(\z_t)$ is invertible on $\mathcal{U}$, and $\Pi_{\bm H}(\x_t)=A_{\bm H}(\z_t)^{-1}A'_{\bm H}(\z'_t)$ with $\z_t=f^{-1}(\x_t)_{1:m}$ and $\z'_t=f'^{-1}(\x_t)_{1:m'}$. Since $F_{\bm H}(\z_t)$ and $F'_{\bm H}(\z'_t)$ are continuous, and each $Q$ and $Q'$ are continuous almost everywhere from (\ref{ass:recurrent}), both  $A_{\bm H}(\z_t)$ and $A'_{\bm H}(\z'_t)$ are continuous. In case (ii), continuity follows directly from the analyticity of $Q$ and $Q'$.

Therefore, $\Pi_{\bm H}(\x)$ is also continuous and a permutation for any $\x_t\in\mathcal{V}$. This implies there exists a nonzero measure subset $\mathcal V_0\subseteq\mathcal V$ on which $\Pi_{\bm H}(\x_t)$ is constant, as otherwise any change in the permutation would induce a discontinuity in $\mathcal{V}_0\subseteq\mathcal{V}$. Therefore, for every $(\z_t,\eps_t)$ such that $f(\z_t,\eps_t)\in\mathcal V_0$, there exists a fixed permutation $\pi\in S_K$, such that for any $k\in\{1,\dots, K\}$:
\begin{equation}
    \phi_k(\x_{t+1} \mid f(\z_{t},\eps_t)) =  \phi'_{\pi(k)}(\x_{t+1} \mid f(\z_{t}, \eps_t)).
\end{equation}

\textbf{Step III.} Now we finish the proof using an iVAE-style argument, with derivatives instead of auxiliary variables \citep{khemakhem2020variational}. Given that $f$ is a $\mathcal{C}^1$ diffeomorphism, there exists a latent reparametrisation such that:
\begin{equation}
   \w'_t =  (\z'_t, \eps'_t) = (f'^{-1}\circ f)(\z_t, \eps_t) = \psi(\z_t, \eps_t) = \psi(\w_t),
\end{equation}
with $\psi:\R^{m+d} \to \R^{m'+d'}$ invertible. By Step~\textbf{II}, there exists a nonzero measure set $\mathcal{V}_0\subseteq\R^n$ and a permutation $\pi\in S_K$ such that, for every $k\in\{1,\dots, K\}$, and every $\w_t=(\z_t,\eps_t)$ with $f(\w_t)\in\mathcal{V}_0$, the reparametrisation $\w'_t = \psi(\w_t)$ gives the following relation:
\begin{equation}
    \tilde{g}_k(\w_{t+1}\mid \w_t) = |\det J_{\psi}(\w_{t+1})|  \tilde{g}_{\pi(k)}'( \psi(\w_{t+1})\mid \psi(\w_t) )
\end{equation}

Taking logarithms and differentiating first with respect to $\w_{t}$ followed by $\w_{t+1}$, the Jacobian term vanishes because it depends only on $\w_{t+1}$. Thus,
\begin{equation}\label{eq:mixed-hessian-identity}
H_k(\w_{t+1},\w_t)
=
J_{\psi}(\w_{t+1})^{\top} H'_{\pi(k)}\big(\psi(\w_{t+1}),\psi(\w_t)\big)J_{\psi}(\w_{t}),
\end{equation}
where
\begin{align}
H_k(\w_{t+1},\w_t)
&:=
\nabla_{\w_{t+1}}\nabla_{\w_t}^{\top}
\log \tilde g_k(\w_{t+1}\mid \w_t),\\
H'_{k'}(\w'_{t+1},\w'_t)
&:=
\nabla_{\w'_{t+1}}\nabla_{\w'_t}^{\top}
\log \tilde g'_{k'}(\w'_{t+1}\mid \w'_t).
\end{align}
Since $\log \tilde{g}_k(\w_{t+1}\mid \w_t) = \log g_k(\z_{t+1}\mid \z_t) + \log p_{\eps}(\eps_{t+1})$, and $ \log p_{\eps}(\eps_{t+1})$ does not depend on $\w_t$, it contributes zero to the mixed derivative. Therefore
\begin{equation}\label{eq:Hk-block}
H_k(\w_{t+1},\w_t)
=
\begin{pmatrix}
\Sigma_k^{-1}J_{\bm m_k}(\z_t) & \bm 0\\
\bm 0 & \bm 0
\end{pmatrix},
\qquad
H'_{k'}(\w'_{t+1},\w'_t)
=
\begin{pmatrix}
\Sigma_{k'}^{\prime -1}J_{\bm m'_{k'}}(\z'_t) & \bm 0\\
\bm 0 & \bm 0
\end{pmatrix}.
\end{equation}
We first identify the latent dimension $m$. Given Eq.~\eqref{eq:Hk-block},
\begin{equation*}
\operatorname{rank}(H_k(\w_{t+1},\w_t)) = \operatorname{rank}(J_{\bm m_k}(\z_t)) \le m, \ \operatorname{rank}(H'_{k'}(\w'_{t+1},\w'_t)) = \operatorname{rank}(J_{\bm m'_{k'}}(\z'_t)) \le m'.
\end{equation*}
Under (\ref{ass:nondegeneratus}), there exist $k_0\in\{1,\dots,K\}$ and $\z^*\in\R^m$ such that $\det J_{\bm m_{k_0}}(\z^*)\neq 0$. Since $\bm m_{k_0}$ is an analytic function, the function $\det J_{\bm m_{k_0}}$ is analytic and not identically zero. Therefore, its zero set has measure zero \citep{mityagin2015zero}. Therefore, for almost every
$\w_t=(\z_t,\eps_t)\in f^{-1}(\mathcal V_0)$, $\det J_{\bm m_{k_0}}(\z_t)\neq 0$, and $\operatorname{rank}(J_{\bm m_{k_0}}(\z_t)) = m$.
Fix  $\w_t\in f^{-1}(\mathcal V_0)$. Since $J_\psi$ is invertible, Eq.~\eqref{eq:mixed-hessian-identity} preserves rank, and therefore
\begin{equation*}
m=\operatorname{rank} \big(H_{k_0}(\w_{t+1},\w_t)\big)
=
\operatorname{rank} \big(H'_{\pi(k_0)}(\w'_{t+1},\w'_t)\big)
\le m'.
\end{equation*}
Equivalently, using (\ref{ass:nondegeneratus}) on model $\mparam'$ yields $m'\leq m$. Therefore, $m=m'$, which implies $d=d'$.

We now show that the temporal component of the reparametrisation $\psi$ depends only on temporal features $\z_t$. Write the reparametrisation as follows
\begin{equation*}
    \psi(\z, \eps) = \big(\psi_{\z}(\z,\eps), \psi_{\eps}(\z, \eps)\big),\quad \psi_{\z}:\R^{m+d}\to\R^{m}, \quad \psi_{\eps}:\R^{m+d}\to\R^{d}
\end{equation*} 
For $k=k_0$, the null space (or kernel) of matrix $H_{k_0}$ is $\ker H_{k_0}(\w_{t+1},\w_t)=\{\bm 0\}\times\R^d$
and similarly for $H'_{\pi(k_0)}$. We now choose any $\bm v\in\ker H_{k_0}(\w_{t+1},\w_t)$. From Eq.~\eqref{eq:mixed-hessian-identity} and using properties of the null space, we have
\begin{equation}\label{eq:null_space}
    0 = H_{k_0}(\w_{t+1},\w_t)\bm v = J_{\psi}(\w_{t+1})^{\top} H'_{\pi(k_0)}\big(\psi(\w_{t+1}),\psi(\w_t)\big)J_{\psi}(\w_{t})\bm v,
\end{equation}
where since $J_{\psi}(\w_{t+1})^{\top}$ is invertible (from invertibility of $\psi$), the vector it multiplies must be zero:
\begin{equation}
    H'_{\pi(k_0)}\big(\psi(\w_{t+1}),\psi(\w_t)\big)J_{\psi}(\w_{t})\bm v = 0.
\end{equation}
Therefore, $J_{\psi}(\w_{t})\bm v\in\ker H'_{\pi(k_0)}(\psi(\w_{t+1}),\psi(\w_t))$. We know $\bm v = (\bm 0, \bm u)^\top$, with arbitrary $\bm u \in \R^d$. Using the Eq.~\ref{eq:null_space}, for every $\bm u\in\R^d$, there exists some $\bm u'\in\R^d$ such that:
\begin{equation}\label{eq:block_triangular}
    \begin{pmatrix}
        \nabla_{\z}\psi_{\z}(\z,\eps) & \nabla_{\eps}\psi_{\z}(\z,\eps) \\
        \nabla_{\z}\psi_{\eps}(\z, \eps) & \nabla_{\eps}\psi_{\eps}(\z, \eps)
    \end{pmatrix}\begin{pmatrix}
        \bm 0 \\
        \bm u
    \end{pmatrix} = \begin{pmatrix}
        \bm 0 \\
        \bm u'
    \end{pmatrix} \implies \nabla_{\eps}\psi_{\z}(\z,\eps) \bm u = \bm 0.
\end{equation}
Therefore $\nabla_{\eps}\psi_{\z}(\z,\eps) = \bm 0$, which implies $\psi_{\z}(\z,\eps) = \psi_{\z}(\z)$. We note from Eq.~\eqref{eq:block_triangular}, $J_{\psi}(\w_{t})$ is block lower triangular, and since $\psi$ is invertible, the diagonal block  $J_{\psi_{\z}}(\z)$ is invertible. Now, take the top-left $m\times m$ block of Eq.~\eqref{eq:mixed-hessian-identity}. For any regime $k$, we obtain
\begin{equation}\label{eq:step3-top-left}
\Sigma_{k}^{-1}J_{\bm m_{k}}(\z_t)
=
J_{\psi_{\z}}(\z_{t+1})^\top
\Sigma_{\pi(k)}^{\prime -1}
J_{\bm m'_{\pi(k)}}(\psi_{\z}(\z_t))
J_{\psi_{\z}}(\z_t).
\end{equation}
Now, we use the regime $k'_0$ in model $\mparam'$ that satisfies (\ref{ass:nondegeneratus}), and set $\hat{k}=\pi^{-1}(k'_0)$ for convenience. Using similar arguments on analyticity of $\bm m'_{k'_0}$, $J_{\bm m'_{k'_0}}(\psi_{\z}(\z_t))$ is invertible almost everywhere. The matrices on the right-hand side are invertible almost everywhere. Rearranging yields,
\begin{equation*}
J_{\psi_{\z}}(\z_{t+1})^\top
=
\Sigma_{\hat k}^{-1}
J_{\bm m_{\hat k}}(\z_t)
J_{\psi_{\z}}(\z_t)^{-1}
J_{\bm m'_{k'_0}}(\psi_{\z}(\z_t))^{-1}
\Sigma'_{k'_0},
\end{equation*}
where the right-hand side does not depend on $\z_{t+1}$. Therefore,  $J_{\psi_{\z}}(\z_{t+1})$ is constant in $\z_{t+1}$, for $\z_{t}$ and $\eps_t$ such that $\w_t=(\z_t,\eps_t)\in f^{-1}(\mathcal V_0)$. Since the Gaussian transition density has full support on $\R^m$, this identity extends to all $\z_{t+1}\in\R^m$. Therefore $\psi_{\z}(\z_t)$ is affine:
\begin{equation}
\z'_t=\psi_{\z}(\z_t)=A\z_t+\Bb,
\quad A\in\R^{m\times m},\quad \Bb,\z_t\in\R^m,
\end{equation}
Here $A$ denotes the affine reparametrisation matrix and should not be confused
with the history-dependent matrix $A_{\bm H}$ used above. Since $\psi_{\z}(\z_t)$ is invertible, $A$ is invertible.

\end{proof}

\subsection{Discussion on Assumptions~(\ref{ass:sticky_switch}--\ref{ass:posterior_domi})} \label{app:dominance_assumptions}
As shown in Step \textbf{I} of the proof of Theorem~\ref{thm:affine_identifiability}, Assumptions~(\ref{ass:sticky_switch}--\ref{ass:posterior_domi}) are crucial to show invertibility of the linear system obtained by stacking $K$ histories. Below, we provide a concentration result showing that histories aligned with a given regime induce filtering posteriors dominated by that regime.
\begin{theorembox}
\begin{lemma}\label{lemma:concentration_to_floor}
Consider a regime path $(k_1^*, \dots, k_{t}^*)$, and latent variables $(\z_{1}, \dots, \z_{t})$, such that for all $\tau=2,\dots,t$,
\begin{equation}
    Q_{k^*_{\tau-1},k^*_{\tau}}(\z_{\tau-1})
    =
    \max_{k'\in\{1,\dots,K\}}
    Q_{k^*_{\tau-1},k'}(\z_{\tau-1})
    \geq 1-\epsilon^*,
\end{equation}
for some $\epsilon^*\in[0,\tfrac12)$. Denote $g_k(\z_t\mid \z_{t-1}) = \mathcal{N}(\z_{t}; \bm{m}_k(\z_{t-1}), \Sigma_k)$ and assume
\begin{equation}
    \frac{g_{k_\tau^*}(\z_{\tau}| \z_{\tau-1})}{\max_{k'\neq k_\tau^*}g_{k'}(\z_{\tau}| \z_{\tau-1})} \geq e^{\ell^*},
\end{equation}
for some $\ell^* > 0$. If
\begin{equation}\label{eq:dist_sticky_separation_1}
    \ell^{*} > \log \frac{1 + \epsilon^*}{1 -\epsilon^*},
\end{equation}
then, there exists $T^*<+\infty$, such that for all $t\geq T^*$, we have
\begin{equation}
    p(s_{t}= k_{t}^* | \z_{1:t}) >  \frac{1}{2}.
\end{equation}
\end{lemma}
\end{theorembox}
\begin{proof}
For $\tau= 2, \dots, t$, define the predictive and filtering posteriors respectively as follows:
\begin{equation}
    \pi_\tau(k) = p(s_{\tau}=k | \z_{1:\tau-1}), \quad \alpha_{\tau}(k)=p(s_\tau=k| \z_{1:\tau}).
\end{equation}
Using Bayes' rule:
\begin{equation}
    \alpha_\tau(k_\tau^*) = \frac{\pi_\tau(k_\tau^*) g_{k_\tau^*}(\z_{\tau}| \z_{\tau-1})}{\pi_\tau(k_\tau^*) g_{k_\tau^*}(\z_{\tau}| \z_{\tau-1}) + \sum_{k\neq k_\tau^*}^{K} \pi_\tau(k) g_{k}(\z_{\tau}| \z_{\tau-1})},
\end{equation}
where we find an upper bound to the denominator.
\begin{equation}
    \sum_{k\neq k_\tau^*}^{K} \pi_\tau(k) g_{k}(\z_{\tau}| \z_{\tau-1}) \leq \sum_{k\neq k_\tau^*}^{K} \pi_\tau(k) \max_{k\neq k_\tau^{*}}g_{k}(\z_{\tau}| \z_{\tau-1}) \leq  (1 - \pi_\tau(k_\tau^*)) g_{k_\tau^*}(\z_{\tau}| \z_{\tau-1}) e^{-\ell^*}.
\end{equation}
Let $R_{\tau}$ be the posterior odds at time $\tau$, defined as follows:
\begin{equation}
    R_\tau := \frac{1 - \alpha_\tau(k_\tau^*)}{\alpha_\tau(k_\tau^*)}.
\end{equation}
Given $\alpha_\tau(k_\tau^*)$ and $1 - \alpha_\tau(k_\tau^*)$ share the denominator, $R_\tau$ has the following simpler form:
\begin{equation}
    R_\tau = \frac{\sum_{k\neq k_\tau^*}^{K} \pi_\tau(k) g_{k}(\z_{\tau}| \z_{\tau-1})}{\pi_\tau(k_\tau^*) g_{k_\tau^*}(\z_{\tau}| \z_{\tau-1})} \leq \frac{ (1 - \pi_\tau(k_\tau^*)) g_{k_\tau^*}(\z_{\tau}| \z_{\tau-1}) e^{-\ell^*}}{\pi_\tau(k_\tau^*) g_{k_\tau^*}(\z_{\tau}| \z_{\tau-1})} = e^{-\ell^{*}} \frac{1- \pi_\tau(k_\tau^*)}{\pi_\tau(k_\tau^*)}
\end{equation}
Now we want to find bounds for both $1- \pi_\tau(k_\tau^*)$ and $\pi_\tau(k_\tau^*)$. Consider the predictive posterior on the aligned state at time $\tau$:
\begin{equation}
    \pi_\tau(k_\tau^*)
=
\sum_{k=1}^K \alpha_{\tau-1}(k)Q_{k,k_\tau^*}(\z_{\tau-1})
\geq
\alpha_{\tau-1}(k_{\tau-1}^*)
Q_{k_{\tau-1}^*,k_\tau^*}(\z_{\tau-1})
\geq
\alpha_{\tau-1}(k_{\tau-1}^*)(1-\epsilon^*).
\end{equation}
Then, for $1- \pi_\tau(k_\tau^*)$:
\begin{equation}
    1- \pi_\tau(k_\tau^*) \leq 1 - \alpha_{\tau-1}(k_{\tau-1}^*)(1 - \epsilon^*) = 1 - \alpha_{\tau-1}(k_{\tau-1}^*) + \alpha_{\tau-1}(k_{\tau-1}^*)\epsilon^*,
\end{equation}
and we can find a recursion for $R_\tau$:
\begin{align}
    R_\tau & \leq e^{-\ell^{*}} \frac{ 1 - \alpha_{\tau-1}(k_{\tau-1}^*) + \alpha_{\tau-1}(k_{\tau-1}^*)\epsilon^*}{ \alpha_{\tau-1}(k_{\tau-1}^*)(1 - \epsilon^*)} = e^{-\ell^{*}} \left(\frac{ R_{\tau-1} }{1 - \epsilon^*} + \frac{ \epsilon^*}{ 1 - \epsilon^*} \right) \\
    & \leq R_{\tau-1}\frac{e^{-\ell^*}}{1 - \epsilon^*} + \frac{e^{-\ell^*} \epsilon^*}{1 - \epsilon^*} = a R_{\tau-1} +b,
\end{align}
where we define
\begin{equation}
    a := \frac{e^{-\ell^*}}{1 - \epsilon^*}, \quad b := \frac{e^{-\ell^*} \epsilon^*}{1 - \epsilon^*}.
\end{equation} 
Therefore, for $R_1>0$, we know:
\begin{equation}
    R_{t} \leq \left(\prod_{\tau=2}^{t}a\right)R_1 + \sum_{\tau=2}^{t} b \prod_{i = \tau+1}^{t} a = a^{t-1}R_1 + b \sum_{r=0}^{t-2} a^r.
\end{equation}
We notice the sum on the right is a geometric sum, which simplifies as follows:
\begin{equation}
    R_{t} \leq a^{t-1}R_1 + b \frac{1 - a^{t-1}}{1 - a}.
\end{equation}
Now, given Eq.~\eqref{eq:dist_sticky_separation_1}, we obtain $a < 1$: 
\begin{equation}
    \ell^*> \log\frac{1+\epsilon^*}{1-\epsilon^*} \implies \ell^*> \log\frac{1}{1-\epsilon^*} \implies a = \frac{e^{-\ell^*}}{1-\epsilon^*}< 1.
\end{equation}
Therefore, the above expression admits a floor $R_\infty$ defined as
\begin{equation}
    R_\infty = \frac{b}{1 - a} = \frac{e^{-\ell^*}\epsilon^*}{1-\epsilon^*-e^{-\ell^*}}.
\end{equation}
Then,
\begin{equation}
    R_t \leq a^{t-1}R_1 + (1 - a^{t-1})R_{\infty} = a^{t-1}(R_1 - R_{\infty}) + R_{\infty}.
\end{equation} 
Moreover,
\begin{equation}
\ell^* > \log \frac{1+\epsilon^*}{1-\epsilon^*}
\implies e^{-\ell^*}(1+\epsilon^*)<1-\epsilon^*\implies
\frac{e^{-\ell^*}\epsilon^*}{1-\epsilon^*-e^{-\ell^*}}<1 \implies R_\infty<1.
\end{equation}

Therefore, Eq.~\eqref{eq:dist_sticky_separation_1} implies \(R_\infty<1\). Since $a<1$, the term $a^{t-1}(R_1-R_\infty)$ vanishes as $t\to\infty$. Given $R_{\infty}<1$, there exists $T^*<+\infty$ such that for all $t\ge T^*$, $R_t<1$. Finally,
\begin{equation}
  R_t<1
\implies
\frac{1-\alpha_t(k_t^*)}{\alpha_t(k_t^*)}<1
\implies
\alpha_t(k_t^*)>\frac12.  
\end{equation}
Therefore, we have $p(s_{t}= k_{t}^* | \z_{1:t}) = \alpha_t(k_t^*)>\frac12$ for $t\geq T^*$.
\end{proof}

The concentration result above shows that, given initial posterior odds $R_1$, a likelihood margin $\ell^*$, and stickiness $\epsilon^*$, there exists a finite time length $T^*$ after which the posterior is dominated by the aligned regime. The sufficient length $T^*$ is determined by $R_1$, $\ell^*$ and $\epsilon^*$. Below, we provide the minimum length required for posterior dominance.
\begin{theorembox}
\begin{corollary}
\label{cor:half_dominance_time}
Let $\z_{1:t}$ be a history from Lemma~\ref{lemma:concentration_to_floor}, with associated regime path $(k_1^*,\dots,k_t^*)$. Define, 
\begin{equation}
    a:=\frac{e^{-\ell^*}}{1-\epsilon^*},
\qquad
R_\infty:=\frac{e^{-\ell^*}\epsilon^*}{1-\epsilon^*-e^{-\ell^*}},\qquad R_1=\frac{1-p(s_1=k^*\mid \z_1)}{p(s_1=k^*\mid \z_1)},
\end{equation}
Assume $R_1>R_{\infty}$. Then any time index $t<+\infty$ which satisfies the equation below ensures $p(s_{t}= k^*_t | \z_{1:t})>\frac12$
\begin{equation}
\label{eq:t_half_dominance}
t \;\ge\; 2+\left\lfloor
\log\!\left(\frac{1-R_\infty}{R_1-R_\infty}\right)\frac{1}{\log a}
\right\rfloor.
\end{equation}
\end{corollary}
\end{theorembox}
\begin{proof}
By Lemma~\ref{lemma:concentration_to_floor}, $R_t \le a^{t-1}(R_1-R_\infty)+R_\infty$. For $R_t<1$, we have
\begin{equation*}
a^{t-1}(R_1-R_\infty)+R_\infty<1 \implies a^{t-1}(R_1-R_\infty)<1-R_\infty.
\end{equation*}
Since $R_1>R_\infty$ and $R_\infty<1$, dividing both sides gives
\begin{equation*}
    a^{t-1}<\frac{1-R_\infty}{R_1-R_\infty}.
\end{equation*}
Because $0<a<1$, we have $\log a<0$, so taking logarithms reverses the inequality:
\begin{equation*}
t-1>
\frac{\log\!\left(\frac{1-R_\infty}{R_1-R_\infty}\right)}{\log a} \implies t \;\ge\; 2+\left\lfloor
\log\!\left(\frac{1-R_\infty}{R_1-R_\infty}\right)\frac{1}{\log a}
\right\rfloor.
\end{equation*}
\end{proof}
In Section~\ref{sec:first_stage}, we provide a simple example where posterior dominance is possible after one transition. In the example, we assume uniform initial posterior $p(s_1=k|\z_{1})=\tfrac1K,\ k\in\{1,\dots,K\}$. Therefore, the initial posterior odds are $R_1=K-1$. Substituting $R_1$, together with $R_{\infty}$ and $a$ into Corollary~\ref{cor:half_dominance_time} yields the one-step dominance condition 
\begin{equation*}
    \ell^* > \log\frac{K-1+\epsilon^*}{1-\epsilon^*}.
\end{equation*}

\paragraph{(\ref{ass:posterior_domi}) under nonlinear Gaussian transitions (\ref{ass:gaussian_transition})}

To further illustrate the intuition behind Assumption~\ref{ass:posterior_domi}, we contextualise it under the nonlinear Gaussian transitions in Assumption~\ref{ass:gaussian_transition}. In this setting, the likelihood margin $\ell^\star$ admits an explicit form.

For simplicity, assume equal isotropic transition variances across regimes, i.e. $\Sigma_k=\sigma^2 I$ for all $k\in\{1,\dots,K\}$. Then, for two regimes $i,k\in\{1,\dots,K\}$, the log-likelihood margin is
\begin{equation}
    \ell_{i,k}(\z_t,\z_{t-1}) =\frac{1}{2\sigma^2}
\left(
\left\|\z_t - \bm m_k(\z_{t-1})\right\|^2
-
\left\|\z_t - \bm m_i(\z_{t-1})\right\|^2
\right).
\end{equation}
Therefore, if a history is aligned with regime $i$, i.e. $\z_t = \bm m_i(\z_{t-1})$, the margin is reduced to a scaled distance between the functions.
\begin{equation}
    \ell_{i,k}(\z_{t-1})
    =
    \frac{1}{2\sigma^2}
    \left\|
    \bm m_i(\z_{t-1}) - \bm m_k(\z_{t-1})
    \right\|^2 .
\end{equation}
This simplified expression further shows that while (\ref{ass:posterior_domi}) imposes certain history-alignment conditions, well-separated regime-specific transition functions naturally support posterior dominance. Notably, in high-dimensional settings, each well-separated coordinate contributes positively to the likelihood margin. 

\subsection{Discussion on Assumption~(\ref{ass:regime_domi})}

Recall Step \textbf{II} in the proof of Theorem~\ref{thm:affine_identifiability}, where we have a matrix $\Pi_{\bm H}(\x_t)$ such that
\begin{equation*}
    \Phi(\x_{t+1} | \x_{t}) =  \Pi_{\bm H}(\x_t)\Phi'(\x_{t+1} | \x_{t})
\end{equation*}
In this context, Assumption~(\ref{ass:regime_domi}) is used to show that the matrix $ \Pi_{\bm H}(\x_t)$ has nonnegative entries. First, we establish a general sufficient condition. The condition requires that each component asymptotically dominates all other components along some direction. 
\begin{theorembox}
\begin{lemma}\label{lemma:nonnegative_entries}
    Assume that for a fixed $\x_t$, we have 
    \begin{equation}\label{eq:main_identity_relation}
    \Phi(\x_{t+1} | \x_{t}) =  \Pi_{\bm H}(\x_t)\Phi'(\x_{t+1} | \x_{t}),
    \end{equation}
    where $\Pi_{\bm H}(\x_t)\in\R^{K\times K}$, and $\Phi$ and $\Phi'$ are the one-step observed component densities defined in Theorem~\ref{thm:affine_identifiability}.  Let $\z_t'=f'^{-1}(\x_t)_{1:m'}$.
    Assume that, for every regime $k\in\{1,\dots, K\}$, there exists a direction  $\bm{d}'_k(\z_t')\in\R^{m'}$ and a fixed $\eps'_{t+1}\in\R^{d'}$, such that with $\x_{t+1}^{(k)}(a) = f'\bigl(a\bm d'_k(\z_t'),\eps'_{t+1}\bigr)$, we have, for every $k'\neq k$,
\begin{equation}\label{eq:limit}
        \lim_{a\to\infty}\frac{
            \phi'_{k'}(\x_{t+1}^{(k)}(a)\mid \x_t)}
            {\phi'_k(\x_{t+1}^{(k)}(a)\mid \x_t)}
        = 0.
    \end{equation}
Then, $\Pi_{\bm H}(\x_t)$ has nonnegative entries.
\end{lemma}
\end{theorembox}
\begin{proof} 
Take row $i$ of Eq.~\eqref{eq:main_identity_relation} and divide by $\phi'_k(\x_{t+1}\mid \x_t)$:
\begin{equation}
    \frac{\phi_i(\x_{t+1} | \x_{t})}{\phi'_k(\x_{t+1}\mid \x_t)} =  \Pi_{\bm H}(\x_t)_{ik} + \sum_{k'\neq k}^K \Pi_{\bm H}(\x_t)_{ik'}\frac{\phi'_{k'}(\x_{t+1} | \x_{t})}{\phi'_k(\x_{t+1}\mid \x_t)}
\end{equation}
Evaluate at $\x_{t+1}=\x_{t+1}^{(k)}(a)= f'\bigl(a\bm d'_k(\z_t'),\eps'_{t+1}\bigr)$, and take the limit $a\to\infty$. The ratio terms in the sum vanish:
\begin{equation}
     \Pi_{\bm H}(\x_t)_{ik}  = \lim_{a\to\infty} \left(\frac{\phi_i(\x_{t+1}(a) | \x_{t})}{\phi'_k(\x_{t+1}(a)\mid \x_t)}\right) \geq 0.
\end{equation}
Therefore, $ \Pi_{\bm H}(\x_t)_{ik}\geq 0$. Since we choose arbitrary $i,k\in\{1,\dots,K\}$, all the entries in $\Pi_{\bm H}(\x_t)$ are nonnegative.
\end{proof}
The lemma above establishes nonnegativity of $\Pi_{\bm H}(\x_t)$ under an asymptotic dominance condition on the observed components. We show that Assumption~(\ref{ass:regime_domi}) is sufficient for this dominance condition, and thus for applying Lemma~\ref{lemma:nonnegative_entries}. First, we connect the dominance condition in Lemma~\ref{lemma:nonnegative_entries} to latent Gaussian transitions under Assumptions~(\ref{ass:emission}--\ref{ass:gaussian_transition})
\begin{theorembox}
\begin{proposition}\label{prop:sufficient_conditions_nonnegativity}
Under Assumptions~(\ref{ass:emission}--\ref{ass:gaussian_transition}),
fix $\x_t$ and write $(\z_t,\eps_t)=f^{-1}(\x_t)$.
A sufficient condition for the dominance condition in
Lemma~\ref{lemma:nonnegative_entries}, i.e., for regime
$k\in\{1,\dots,K\}$ there exists
$\bm d_k(\z_t)\in\R^m$
such that for any $k'\neq k$, Eq.~\eqref{eq:limit} holds,
is the following. For every regime $k\in\{1,\dots,K\}$,
there exists
a direction $\bm d_k(\z_t)\in\R^m$
such that, for every $k'\neq k$, either:
\begin{enumerate}
\item[(i)]
$\bm d_k(\z_t)^\top
\left(
    \Sigma_{k'}^{-1}-\Sigma_k^{-1}
\right)
\bm d_k(\z_t)>0,$
or
\item[(ii)]
$
\bm d_k(\z_t)^\top
\left(
    \Sigma_{k'}^{-1}-\Sigma_k^{-1}
\right)
\bm d_k(\z_t)=0$
and
$\bm d_k(\z_t)^\top
\left(
    \Sigma_k^{-1}\bm m_k(\z_t)
    -
    \Sigma_{k'}^{-1}\bm m_{k'}(\z_t)
\right)>0$.
\end{enumerate}
\end{proposition}
\end{theorembox}
\begin{proof}
We first connect the observed component dominance to the
Gaussian component dominance. Fix $\x_t$ and write
$(\z_t,\eps_t)=f^{-1}(\x_t)$.
Given $k\in\{1,\dots,K\}$ and a direction
$\bm d_k(\z_t)$ and $\eps_{t+1}$, write
\[
\x_{t+1}^{(k)}(a)
=
f\bigl(a\bm d_k(\z_t),\eps_{t+1}\bigr).
\]
By the change-of-variable rule, we have
\begin{equation}
\phi_k(\x_{t+1}^{(k)}(a)\mid\x_t)
=
\left|
    \det J_{f^{-1}(\x_{t+1}^{(k)}(a))}
\right|
g_k\bigl(a\bm d_k(\z_t)\mid\z_t\bigr)
p_{\eps}(\eps_{t+1}).
\end{equation}
For any $k'\neq k\in\{1,\dots,K\}$,
\begin{equation}
\frac{
    \phi_{k'}(\x_{t+1}^{(k)}(a)\mid\x_t)
}{
    \phi_k(\x_{t+1}^{(k)}(a)\mid\x_t)
}
=
\frac{
    g_{k'}\bigl(a\bm d_k(\z_t)\mid\z_t\bigr)
}{
    g_k\bigl(a\bm d_k(\z_t)\mid\z_t\bigr)
}.
\label{eq:observed_gaussian_ratio}
\end{equation}
Therefore, we continue by establishing dominance over Gaussian
components. Consider the log ratio between two Gaussians
$g_{k'}(\z_{t+1}\mid\z_t)$ and
$g_k(\z_{t+1}\mid\z_t)$, and consider the terms dependent
on $\z_{t+1}$.
\begin{multline}
\log
\frac{
    g_{k'}(\z_{t+1}\mid\z_t)
}{
    g_k(\z_{t+1}\mid\z_t)
}
=
-\frac{1}{2}
\bigl(
    \z_{t+1}-\bm m_{k'}(\z_t)
\bigr)^\top
\Sigma_{k'}^{-1}
\bigl(
    \z_{t+1}-\bm m_{k'}(\z_t)
\bigr)
\\
 +
\frac{1}{2}
\bigl(
    \z_{t+1}-\bm m_k(\z_t)
\bigr)^\top
\Sigma_k^{-1}
\bigl(
    \z_{t+1}-\bm m_k(\z_t)
\bigr)
+
\mathcal O(1).
\end{multline}
Group quadratic and linear terms together.
\begin{multline}
\log\frac{g_{k'}(\z_{t+1}\mid\z_t)}{g_k(\z_{t+1}\mid\z_t)}=-\frac{1}{2}\z_{t+1}^\top\left(\Sigma_{k'}^{-1}-\Sigma_k^{-1}\right)\z_{t+1}\\
 -\z_{t+1}^\top\left(\Sigma_k^{-1}\bm m_k(\z_t)-\Sigma_{k'}^{-1}\bm m_{k'}(\z_t)\right)+\mathcal O(1).
\end{multline}
Now we take a ray with direction $\bm d_k(\z_t)\in\R^m$:$\z_{t+1}=a\bm d_k(\z_t)$, $a\in\R$.
Since $\x_t$, and therefore $\z_t$, is fixed,
$\bm d_k(\z_t)$ is fixed along this limit. We obtain
\begin{multline}
\log\frac{g_{k'}\bigl(a\bm d_k(\z_t)\mid\z_t\bigr)}{g_k\bigl(a\bm d_k(\z_t)\mid\z_t\bigr)}=-\frac{a^2}{2}\bm d_k(\z_t)^\top\left(\Sigma_{k'}^{-1}-\Sigma_k^{-1}\right)\bm d_k(\z_t)\\
 -a\bm d_k(\z_t)^\top\left(\Sigma_k^{-1}\bm m_k(\z_t)-\Sigma_{k'}^{-1}\bm m_{k'}(\z_t)\right)+\mathcal O(1).
\end{multline}
Since Eq.~\eqref{eq:limit} converges to $0$ for
$a\to\infty$, we need the above equation to converge to
$-\infty$.
Under condition~(i), the coefficient of the quadratic term
is strictly positive, so the log ratio converges to
$-\infty$. Under condition~(ii), the quadratic term
vanishes and the coefficient of the linear term is strictly
positive, so the log ratio again converges to $-\infty$.
Therefore,
\begin{equation}
\lim_{a\to\infty}
\frac{
    g_{k'}\bigl(a\bm d_k(\z_t)\mid\z_t\bigr)
}{
    g_k\bigl(a\bm d_k(\z_t)\mid\z_t\bigr)
}
=0.
\end{equation}
Since the same direction $\bm d_k(\z_t)$ satisfies one of
conditions~(i)--(ii) for every $k'\neq k$, all ratios vanish
simultaneously along the same ray. Equation~
\eqref{eq:observed_gaussian_ratio} then gives
Eq.~\eqref{eq:limit}.
\end{proof}

{

We finally connect dominance under nonlinear Gaussian transitions to Assumption~\eqref{ass:regime_domi} through the following corollary. Recall that Assumption~\eqref{ass:gaussian_transition} assumes diagonal covariances, denoted by $\Sigma_k=\mathrm{diag}(\bm \sigma_k^2)$, for each $k\in\{1,\dots,K\}$.

\begin{theorembox}
\begin{proposition}\label{prop:a6_gaussian_dominance}
Under Assumptions~(\ref{ass:gaussian_transition}) and~(\ref{ass:regime_domi}), for every $\z_t\in\R^m$ and $k\in\{1,\dots,K\}$, there exists a direction $\bm d_k(\z_t)\in\R^m$ such that, for every $k'\neq k$, either condition~(i) or condition~(ii) of Proposition~\ref{prop:sufficient_conditions_nonnegativity} holds.
\end{proposition}
\end{theorembox}
\begin{proof}
    Fix $\z_t$ and regime $k$. First, consider the nonempty coordinate set $\mathcal{I}_k$ and suppose $\mathcal{T}_k=\emptyset$. Then, for every $k'\neq k$, there exists at least one coordinate $i\in\mathcal{I}_k$ such that $\sigma_{k,i} > \sigma_{k',i}$. We can simply choose $\bm d_k(\z_t)$ independently of $\z_t$:  
    \begin{equation}
        \bm d_k(\z_t) = \sum_{i\in\mathcal{I}_k}\mathbf e_{i},
    \end{equation} 
    where $\mathbf e_{i}$ is the canonical vector of the $i-$th dimension. Since the covariance matrices are diagonal, substituting this choice into condition~(i) of Proposition~\ref{prop:sufficient_conditions_nonnegativity} gives, for any $k'\neq k$:
\begin{equation}
\bm d_k(\z_t)^\top
\left(
    \Sigma_{k'}^{-1}-\Sigma_k^{-1}
\right)
\bm d_k(\z_t)  =
\sum_{i\in\mathcal I_k}
\left[
    \left(\frac{1}{\sigma_{k',i}}\right)^2
    -
    \left(\frac{1}{\sigma_{k,i}}\right)^2
\right] > 0.
\label{eq:variance_dominance_sum}
\end{equation}
Every term in the sum is nonnegative because $\sigma_{k,i}\geq\sigma_{k',i}$ for every
$i\in\mathcal I_k$, and at least one term is strictly
positive because $k'\notin\mathcal T_k$. Therefore,  condition~(i) of Proposition~\ref{prop:sufficient_conditions_nonnegativity} holds for any $k'\neq k$.

Now suppose $\mathcal{T}_k\neq\emptyset$ and denote $C_k(\z_t)
    :=
    \operatorname{conv}(
    \left\{
        \bm m_{k',\mathcal I_k}(\z_t):
        k'\in\mathcal T_k
    \right\}).$
Since $C_k(\z_t)$ is compact and convex from Eq.~\eqref{eq:convex_hull_assumption} in Assumption~\ref{ass:regime_domi},
$\bm m_{k,\mathcal I_k}(\z_t)\notin C_k(\z_t)$. The strict
separation theorem
\citep[Proposition~1.5.3]{bertsekas2009convex} shows that there exists a nonzero vector
$\bm u_k(\z_t)\in\R^{|\mathcal I_k|}$ and a scalar $b\in\R$ such
that
\begin{equation}\label{eq:strict_separation_theorem}
    \bm u_k(\z_t)^\top \bm y
    <
    b
    <
    \bm u_k(\z_t)^\top
    \bm m_{k,\mathcal I_k}(\z_t),
    \qquad
    \forall \bm y\in C_k(\z_t).
\end{equation}
Since $\bm m_{k',\mathcal I_k}(\z_t)\in C_k(\z_t)$ for every
$k'\in\mathcal T_k$, using $\bm m_{k',\mathcal I_k}(\z_t)$ on Eq.~\eqref{eq:strict_separation_theorem} yields
\begin{equation}\label{eq:convex_hull_consequence}
    \bm u_k(\z_t)^\top
    \left(
        \bm m_{k,\mathcal I_k}(\z_t)
        -
        \bm m_{k',\mathcal I_k}(\z_t)
    \right)>0,
    \qquad
    \forall k'\in\mathcal T_k.
\end{equation}
Since the above inequality is strict and $\mathcal{I}_k$ is finite, we can perturb $\bm u_{k}(\z_t)$ such that all its entries are nonzero while preserving Eq.~\eqref{eq:convex_hull_consequence}. We can now connect Assumption~(\ref{ass:regime_domi}) with condition (ii) of Proposition~\ref{prop:sufficient_conditions_nonnegativity}. Choose $\bm d_k(\z_t)$ as follows.
\begin{equation}
d_{k,i}(\z_t) = \begin{cases}
    \sigma^{2}_{k,i} u_{k,i}(\z_t), & i\in \mathcal{I}_k, \\
    0, & i\notin \mathcal{I}_k.
\end{cases}
\end{equation}
Consider $k'\in \mathcal{T}_k$, where by definition, $\sigma_{k,i}=\sigma_{k',i}$ for any $i\in\mathcal{I}_k$. Since $\bm d_{k,i}(\z_t)=0$ if $i\notin \mathcal{I}_k$, the quadratic equation in condition~(ii) from Proposition~\ref{prop:sufficient_conditions_nonnegativity} holds:
\begin{equation}
    \bm d_k(\z_t)^\top
\left(
    \Sigma_{k'}^{-1}-\Sigma_k^{-1}
\right)
\bm d_k(\z_t)  =
\sum_{i\in\mathcal I_k}
\left[
    \left(\frac{d_{k,i}}{\sigma_{k',i}}\right)^2
    -
    \left(\frac{d_{k,i}}{\sigma_{k,i}}\right)^2
\right] = 0.
\end{equation}
Substituting our choice of $\bm d_k(\z_t)$ in the linear part of condition~(ii) gives:
\begin{align}
    \bm d_k(\z_t)& ^\top
\left(
    \Sigma_k^{-1}\bm m_k(\z_t)
    -
    \Sigma_{k'}^{-1}\bm m_{k'}(\z_t)
\right) \\
 & = \sum_{i\in\mathcal I_k}  \sigma^{2}_{k,i} u_{k,i}(\z_t) \left(\frac{m_{k,i}(\z_t)}{\sigma^{2}_{k,i}} - \frac{m_{k',i}(\z_t)}{\sigma^{2}_{k',i}}\right) \\
  & = \sum_{i\in\mathcal I_k}   u_{k,i}(\z_t) \Big(m_{k,i}(\z_t)- m_{k',i}(\z_t)\Big) \\
   & = \bm u_{k}(\z_t)^\top \Big(m_{k,\mathcal I_k}(\z_t)- m_{k',\mathcal I_k}(\z_t)\Big) > 0
\end{align}
where the variances cancel out because they are equal across $\mathcal{I}_k$. The resulting inequality follows from Eq.~\eqref{eq:convex_hull_consequence}. Therefore, condition (ii) holds for every $k'\in\mathcal{T}_k$. Finally, consider $k'\notin \mathcal{T}_k$. This case is similar to the case $\mathcal{T}_k=\emptyset$, where  there is at least one $i\in\mathcal{I}_k$ such that $\sigma_{k,i} > \sigma_{k',i}$. Since every coordinate $d_{k,i}(\z_t)$ for $i\in\mathcal{I}_k$ is nonzero, condition (i) of Proposition~\ref{prop:sufficient_conditions_nonnegativity} is satisfied:
\begin{equation}
    \bm d_k(\z_t)^\top
\left(
    \Sigma_{k'}^{-1}-\Sigma_k^{-1}
\right)
\bm d_k(\z_t)  =
\sum_{i\in\mathcal I_k} d_{k,i}^2
\left[
    \left(\frac{1}{\sigma_{k',i}}\right)^2
    -
    \left(\frac{1}{\sigma_{k,i}}\right)^2
\right] > 0,
\end{equation}
The inequality follows the same argument as Eq.~\eqref{eq:variance_dominance_sum}. Therefore, under Assumption~\eqref{ass:regime_domi}, the conditions of Proposition~\ref{prop:sufficient_conditions_nonnegativity} hold.
\end{proof}
}

\subsection{Proof of Theorem~\ref{thm:msm_transition_identifiability}} \label{app:proof_identifiability_rmsm}

Consider our latent variable model, which is an MSM with recurrent state transitions:

\begin{equation}\label{eq:cond_markov}
    p_{\mparam_{\z},\mparam_{\s}}(\z_{1:T}, \s_{1:T}) = \prod_{t=1}^T\pz(\z_{t} | \z_{t-L_z:t-1}, s_{t}) \ps(s_t | s_{t-1}, \z_{t-1})
\end{equation}
where for the first part of Theorem~\ref{thm:msm_transition_identifiability}, we consider latent transitions with arbitrary lag $L_z$. This broadens the applicability of our theoretical results. Under this setting, and w.l.o.g., one can adapt (\ref{ass:gaussian_transition}) to a multi-lag setting where the rMSM considers variables $\z_t$ which are Gaussian, such that the means and covariances are dependent on $s_t$ and $\z_{t-L_z:t-1}$ as follows:
\begin{equation}\label{eq:nonlinear_gaussian}
    \pz(\z_t | \z_{t-L_z:t-1}, s_t) = \N(\z_{t} ; \bm{m}_{s_t}\big(\z_{t-1}, \dots, \z_{t-L_z}), \Sigma_{s_t}\big)
\end{equation}
where $\bm{m}_{k}:\R^{mL_z}\to\R^m$ are analytic functions, distinct for all $k\in\{1,\dots, K\}$. We establish $L_z-$lagged regime-dependent transition identifiability up to permutations below.

\vspace{2em}

\subsubsection{Identifiability of autoregressive components (Theorem \ref{thm:msm_transition_identifiability} (i))}

\begin{theorembox}
\begin{theorem}\label{thm:msm_transition_identifiability_z_transitions}
Under assumption (\ref{ass:gaussian_transition}), the Markov Switching Model is identifiable with respect to its autoregressive components $\pz(\z_t\mid\z_{t-1}, s_t)$ up to permutation.   
\end{theorem}
\end{theorembox}
\begin{proof}
    For notation simplicity, we drop the subscript $\mparam$, and consider two MSM distributions $p(\z_{1:T})$ and $\hat{p}(\z_{1:T})$. Assume there exist two MSMs satisfying $p(\z_{1:T}) = \hat{p}(\z_{1:T})$, with $K$ and $\hat{K}$ regimes, respectively. Since wlog. $p(\z_{1:t}) = p(\z_{t} | \z_{1:t-1})p(\z_{1:t-1})$, the fact that $p(\z_{1:T}) = \hat{p}(\z_{1:T})$, and $p(\z_{t} | \z_{1:t-1})$ is a probability distribution imply that
\begin{equation}
    p(\z_{t} | \z_{1:t-1}) = \hat{p}(\z_{t} | \z_{1:t-1}), \quad \forall t = 2, ..., T.
\end{equation}
Now, since the model assumes a conditional Markov structure controlled by the current state $s_{t}$ \eqref{eq:cond_markov}, we can show that the conditional distribution is a finite mixture distribution:
\begin{equation}
    p(\z_{t} | \z_{1:t-1}) = \sum_{k=1}^K p(s_{t} = k | \z_{1:t-1}) p(\z_{t} |\z_{t-L_z:t-1}, s_{t}=k).
\label{eq:mixture}
\end{equation}
Since $p(\z_{t} |\z_{t-L_z:t-1}, s_{t}=k)$ is Gaussian \eqref{eq:nonlinear_gaussian}, Proposition 2 in \citet{yakowitz1968identifiability} shows Gaussians with distinct parameters are identifiable up to permutations. Therefore, we have $K = \hat{K}$, and for almost any given $\z_{t-L_z:t-1} = \bm{\alpha} \in \mathbb{R}^{mL_z}$ (modulo some zero-measure sets) we have the following result: there exists a permutation indexed by $\bm{\alpha}$,  $\pi:\R^{mL_z}\to S_K$ such that $\hat{k} = \pi(\bm{\alpha},k)$ for any $k \in \{1, ..., K\}$
\begin{equation}\label{eq:identifiability}
\begin{aligned}
p(s_{t} = k | \z_{1:t-L_z-1}, \z_{t-L_z:t-1} = \bm{\alpha}) &= \hat{p}\big(s_{t} = \pi(\bm{\alpha}, k) | \z_{1:t-L_z-1}, \z_{t-L_z:t-1}  = \bm{\alpha}\big), \\
p(\z_{t} |\z_{t-L_z:t-1} = \bm{\alpha}, s_{t}=k) &= \hat{p}\big(\z_{t} | \z_{t-L_z:t-1} = \bm{\alpha}, s_{t}=\pi(\bm{\alpha}, k)\big).
\end{aligned}
\end{equation}
To clarify, Proposition 2 in \citet{yakowitz1968identifiability} establishes that multivariate Gaussian distributions are identifiable if and only if the means or covariances are distinct. Wlog., assume covariances are identical, the distribution is unidentifiable on the set where pairs of mean functions coincide. For any model, we can determine this set:
\begin{equation}
    \mathcal{C}_{k,k'} := \Big\{\bm{\alpha}\in\R^{mL_z} :  \bm{m}_k(\bm{\alpha})=\bm{m}_{k'}(\bm{\alpha}) \Big\}, \quad \mathcal{C} := \bigcup_{k\neq k'}^K\mathcal{C}_{k,k'}.
\end{equation}
where for $\hat{p}$ we can determine a similar collision set $\hat{\mathcal{C}}$ finding similar $\hat{\mathcal{C}}_{k,k'}$ for any $k\neq k'\in\{1,\dots,\hat{K}\}$. Given that the zero set of an analytic function has Lebesgue measure zero \cite{mityagin2015zero},  $\mathcal{C}_{k,k'}$ has Lebesgue measure zero. Therefore $\mathcal{C}$ and $\hat{\mathcal{C}}$ have Lebesgue measure zero since $K<+\infty$. Therefore, for any pair of models, Equation \eqref{eq:identifiability} holds for any $\bm{\alpha}\in\mathcal{Z}$, where $\mathcal{Z} =\R^{mL_z} \setminus( \mathcal{C} \cup \hat{\mathcal{C}} )$, and $\mu(\mathcal{Z})=\mu(\R^{mL_z})$.

Now we show that $\pi(\bm{\alpha}, k)$ is a constant almost everywhere, i.e., $\pi(\bm{\alpha}, k) = \pi(k)$. First, define the pairwise distance between the mean functions and the minimum distance:
\begin{equation}
    d_{k,k'}(\bm{\alpha}) := || \bm{m}_k(\bm{\alpha}) -  \bm{m}_{k'}(\bm{\alpha})||, \quad \delta_0(\bm{\alpha}) := \min_{k\neq k'} d_{k,k'}(\bm{\alpha}).
\end{equation}
Choose a non-collision point $\bm{\alpha}_0\in\mathcal{Z}$, where we know $\delta_0(\bm{\alpha}_0) > 0$, and $\bm{m}_k(\bm{\alpha}_0) = \hat{\bm{m}}_{\hat{k}}(\bm{\alpha}_0)$ with $\hat{k}=\pi(\bm{\alpha}_0,k)$ for any $k\in\{1,\dots,K\}$ (from Eq. \eqref{eq:identifiability}). Define $\hat{\delta}_0(\bm{\alpha}_0)>0$  similarly for $\hat{\bm{m}}_{\hat{k}},\ \hat{k}\in\{1,\dots,K\}$. Now set $\delta:=\frac{1}{4} \min \{\delta_0(\bm{\alpha}_0),  \hat{\delta}_0(\bm{\alpha}_0)\}$. By continuity, for all $k, \hat{k}\in\{1,\dots,K\}$, there exists $\epsilon >0$ such that we can construct an $\epsilon$-ball $B(\bm{\alpha}_0, \epsilon)$ using $\ell_2$-norm, where
\begin{equation}
    || \bm{m}_k(\bm{\alpha}_0) -  \bm{m}_{k}(\bm{\alpha})|| < \delta, \quad || \hat{\bm{m}}_{\hat{k}}(\bm{\alpha}_0) -  \hat{\bm{m}}_{\hat{k}}(\bm{\alpha})|| < \delta, \quad \forall \bm{\alpha}\in B(\bm{\alpha}_0, \epsilon).
\end{equation}
Now for any $\bm{\alpha}\in B(\bm{\alpha}_0, \epsilon)$ and $k\neq k'\in\{1,\dots,K\}$,  we can compute a lower bound for $d_{k,k'}(\bm{\alpha})$:
\begin{align*}
    d_{k,k'}(\bm{\alpha}) & = ||\bm{m}_k(\bm{\alpha}) -  \bm{m}_{k'}(\bm{\alpha}) || \\ & \geq   ||\bm{m}_k(\bm{\alpha}_0) -  \bm{m}_{k'}(\bm{\alpha}_0)||  - ||\bm{m}_k(\bm{\alpha}_0) -  \bm{m}_{k}(\bm{\alpha})|| - ||\bm{m}_{k'}(\bm{\alpha}_0) -  \bm{m}_{k'}(\bm{\alpha})||  \\ & >  \delta_0(\bm{\alpha}_0) - 2 \delta  \geq\delta_0(\bm{\alpha}_0)/2.
\end{align*}
And similarly, we can establish $\hat{d}_{k,k'}(\bm{\alpha}) > \hat{\delta}_0(\bm{\alpha}_0)/2$ for any $\bm{\alpha}\in B(\bm{\alpha}_0, \epsilon)$ with $\epsilon>0$.

Assume $\pi(\bm{\alpha}, k)$ is non-constant. Then $\exists \bm{\alpha}_1\in B(\bm{\alpha}_0, \epsilon)$ such that:
\begin{equation}
    \hat{k}_1=\pi(\bm{\alpha}_1,k)\neq\pi(\bm{\alpha}_0, k)=\hat{k}_0
\end{equation}
From Eq. \eqref{eq:identifiability}, we have $\bm{m}_k(\bm{\alpha}_0) = \hat{\bm{m}}_{\hat{k}_0}(\bm{\alpha}_0)$ and $\bm{m}_k(\bm{\alpha}_1) = \hat{\bm{m}}_{\hat{k}_1}(\bm{\alpha}_1)$ and continuity gives
\begin{align*}
    ||\hat{\bm{m}}_{\hat{k}_1}(\bm{\alpha}_1) -  \hat{\bm{m}}_{\hat{k}_0}(\bm{\alpha}_1) || \leq   ||\hat{\bm{m}}_{\hat{k}_1}(\bm{\alpha}_1) -  \bm{m}_{k}(\bm{\alpha}_0) || +  ||\bm{m}_k(\bm{\alpha}_0) -  \hat{\bm{m}}_{\hat{k}_0}(\bm{\alpha}_1) || < 2\delta \leq \hat{\delta}_0(\bm{\alpha}_0)/2,
\end{align*}
but because  $\hat{k}_0 \neq \hat{k}_1$, we also have
\begin{equation*}
    ||\hat{\bm{m}}_{\hat{k}_1}(\bm{\alpha}_1) -  \hat{\bm{m}}_{\hat{k}_0}(\bm{\alpha}_1) || = \hat{d}_{\hat{k}_0,\hat{k}_1}(\bm{\alpha}_1) > \hat{\delta}_0(\bm{\alpha}_0)/2,
\end{equation*}
and a contradiction is reached. Therefore, for any  $\bm{\alpha}_0\in\mathcal{Z}$, there exists an open subset $\mathcal{U}\subset\mathcal{Z}$ where $\pi(\bm\alpha,k)=\pi_{\mathcal{U}}(k)$ for all $\bm{\alpha}\in\mathcal{U}$
\begin{equation}
    \bm{m}_k(\bm{\alpha}) = \hat{\bm{m}}_{\pi_{\mathcal{U}}(k)}(\bm{\alpha}), \quad \forall \bm{\alpha} \in \mathcal{U},
\end{equation}
Since the mean functions on both models are analytic and distinct, equality on a nonempty open set $\mathcal{U}$ implies equality everywhere on $\R^{mL_z}$ by the identity theorem for vector-valued real analytic functions. This implies two distinct permutations $\pi_{\mathcal{U}_1} \neq \pi_{\mathcal{U}_2}$ cannot occur on disjoint open subsets $\mathcal{U}_1\subset\mathcal{Z},\ \mathcal{U}_2\subset\mathcal{Z}$. Otherwise, $\exists k \in\{1,\dots,K\}$ such that $\pi_{\mathcal{U}_1}(k)\neq \pi_{\mathcal{U}_2}(k)$, and hence forcing $\hat{\bm{m}}_{\pi_{\mathcal{U}_1}(k)} \equiv \bm{m}_{k} \equiv \hat{\bm{m}}_{\pi_{\mathcal{U}_2}(k)}$ contradicting the assumption of distinct means on $\hat{\bm{m}}_k,\ k\in\{1,\dots,K\}$. 

In summary, there exists a global permutation $\pi\in S_{K}$ such that $\hat{k} = \pi(k),\ k\in{1,\dots,K}$ and the following equivalence holds almost everywhere for $\z_{t-L_z:t-1} \in \R^{mL_z}$:
\begin{equation}
p(\z_{t} | \z_{t-L_z:t-1}, s_{t}=k) = \hat{p}\big(\z_{t} | \z_{t-L_z:t-1}, s_{t}=\pi(k)\big), \quad \forall \z_{t-L_z:t-1} \in \mathcal{Z}, \quad \mu(\mathcal{Z}) = \mu(\mathbb{R}^{mL_z}),
\end{equation}
and from Gaussian assumptions on $\z_t$, and analyticity on $\bm{m}_k$,
\begin{equation}
\bm{m}_k \equiv \hat{\bm{m}}_{\pi(k)}, \quad \Sigma_{k} \equiv \hat{\Sigma}_{\pi(k)}.
\end{equation}
\end{proof}

\subsubsection{Identifiability of switching process (Theorem~\ref{thm:msm_transition_identifiability} (ii))}
Below we present the identifiability of the switching process up to permutation by further assuming (\ref{ass:sticky_switch}--\ref{ass:posterior_domi}). Therefore, we fall back to the single-lagged case with $L_z=1$ for simplicity.
\begin{theorembox}
\begin{theorem}\label{thm:msm_transition_identifiability_ps}
The MSM is identifiable up to state permutation equivalence, further assuming (\ref{ass:sticky_switch}--\ref{ass:posterior_domi}) and that the transition matrix $p(s_t | s_{t-1}, \z_{t-1}) = Q(\z_{t-1})$ is an analytic function in $\z_{t-1}$.
\end{theorem}
\end{theorembox}
\begin{proof}
From Theorem \ref{thm:msm_transition_identifiability}, we have established that the number of discrete states $K$ is identifiable, and the transitions $p(\z_{t}| \z_{t-1}, s_t)$ and one-step predictive filtering posteriors $p(s_{t}| \z_{1:t-1})$ are identifiable up to permutation equivalence. This identifiability result extends to all time steps $t > 1$, and due to regime stationarity of MSMs, this permutation equivalence does not change over time.

Now w.l.o.g., fix a particular ordering of the discrete state. Note that for $t > 1$, $p(s_{t} = k' | \z_{1:t})$ is also identifiable for all $k' \in \{1, ..., K \}$, as $p(\z_{t} | \z_{t-1}, s_{t}=k')$ and $p(s_{t} = k' | \z_{1:t-1})$ are identifiable as well (from Theorem \ref{thm:msm_transition_identifiability}) and then by applying Bayes' rule:
\begin{equation}
    p(s_{t} = k' | \z_{1:t}) \propto p(\z_{t} | \z_{t-1}, s_{t}=k') p(s_{t} = k' | \z_{1:t-1}).
\end{equation}
Now fix $t\geq T^*$ and choose $\bm\beta\in\mathcal Z$ as given by (\ref{ass:sticky_switch}--\ref{ass:posterior_domi}).
For each regime $k\in\{1,\dots,K\}$, (\ref{ass:sticky_switch}--\ref{ass:posterior_domi}) and
Lemma~\ref{lemma:concentration_to_floor} provide a history $\z_{1:t}^{(k)}=(\bm\alpha_k,\bm\beta)$ such that $p(s_t=k\mid \z_{1:t-1}=\bm\alpha_k, \z_t=\bm\beta)>\frac12.$

Define the matrices
\begin{equation}
A(\z)=
\begin{pmatrix}
a_1(\z)\\
\vdots\\
a_K(\z)
\end{pmatrix},
\qquad
a_k(\z):=p(s_t\mid \z_{1:t-1}=\bm\alpha_k, \z_t=\z),
\end{equation}
and
\begin{equation}
B(\z)=
\begin{pmatrix}
b_1(\z)\\
\vdots\\
b_K(\z)
\end{pmatrix},
\qquad
b_k(\z):=p(s_{t+1}\mid \z_{1:t-1}=\bm\alpha_k, \z_t=\z).
\end{equation}
For every $\z\in\R^m$, the recurrent transition model implies $B(\z)=A(\z)Q(\z)$. At $\z=\bm\beta$, the matrix $A(\bm\beta)$ is strictly diagonally dominant given Lemma \ref{lemma:concentration_to_floor}, and therefore $A(\bm\beta)$ is invertible by the Levy-Desplanques theorem \citep{wu2005bounds}. Furthermore, $A(\z)$ is also continuous, as each row is formed by a weighted average of continuous functions in $\z$. Therefore, there exists a nonzero measure set $\mathcal{U}_{\bm\beta}$, such that $A(\z)$ is invertible for any $\z\in\mathcal{U}_{\bm\beta}$. This implies
\begin{equation}
    Q(\z) = A(\z)^{-1}B(\z), \quad \z\in\mathcal{U}_{\bm\beta},
\end{equation}
i.e., $Q(\z)$ is identifiable on $\mathcal{U}_{\bm\beta}$.

The above statement implies any other transition matrix $Q'(\z)$ inducing the same distribution $p(\z_{t+1}\mid \z_{1:t-1}, \z_t=\z)$ satisfies $Q(\z)=Q'(\z)$ for any $\z\in\mathcal{U}_{\bm\beta}$. Since both $Q(\z)$, and $Q'(\z)$ are analytic and coincide on the nonempty open set $\mathcal{U}_{\bm\beta}$, the identity theorem for real-analytic functions implies $Q(\z)=Q'(\z)$ for any $\z\in\R^m$. Therefore, $Q$ is identifiable up to state permutation equivalence.

\end{proof}

\subsubsection{Alternative assumptions on the switching process}
We additionally establish permutation identifiability of more restrictive switching processes, such as autonomous switching  $\ps(s_{t}\mid s_{t-1})$ and observation-dependent switching from the MSM view, $\ps(s_{t}\mid\z_t)$, where $\z_t$ denotes an MSM observed variable. The latter connects to self-exciting threshold autoregressive models (SETAR; \citealp{setar1980}). Establishing such results does not involve additional assumptions beyond (\ref{ass:gaussian_transition}). Therefore, for completeness, we consider again an arbitrary lag $L_z$.
\begin{theorembox}
\begin{corollary}
    Under Assumption~(\ref{ass:gaussian_transition}), the switch transitions are identifiable up to state permutation equivalence if the transition matrix is observation-dependent with respect to MSM observations, i.e., $\ps(s_t|s_{t-1}, \z_{t-1}):=\ps(s_{t}| \z_{t-1})$.
\end{corollary}
\end{theorembox}
\begin{proof}
    Drop the subscript $\ps$ for simplicity. From Theorem \ref{thm:msm_transition_identifiability_z_transitions}, the number of discrete states $K$ is identifiable, and the one-step filtering posteriors $p(s_t| \z_{1:t-1})$ are identifiable. In an observation-dependent model, $s_t$ is conditionally independent of $\z_{1:t-2}$  given $\z_{t-1}$. Therefore,
    \begin{equation}
        p(s_t| \z_{1:t-1}) = p(s_t | \z_{t-1})
    \end{equation}
    which implies the transitions are identifiable up to permutations.
\end{proof} 
\begin{theorembox}
\begin{theorem}\label{thm:autonomous:Q}
 Under assumption (\ref{ass:gaussian_transition}), the switch transitions are identifiable up to state permutation equivalence if the transition matrix is autonomous: $\ps(s_t|s_{t-1}, \z_{t-1}):=\ps(s_t | s_{t-1})$.
\end{theorem}
\end{theorembox}
\begin{proof}
Again, drop the subscript $\ps$ for simplicity. From Theorem \ref{thm:msm_transition_identifiability_z_transitions}, we have established that the number of discrete states $K$ is identifiable, and the transitions $p(\z_{t}| \z_{t-1}, s_t)$ and one-step predictive filtering posteriors $p(s_{t}| \z_{1:t-1})$ are identifiable up to permutation equivalence. This identifiability result extends to all time steps $t > 1$, and due to regime stationarity of MSMs, this permutation equivalence does not change over time.

Now fix a particular ordering of the discrete state w.l.o.g.. Note that for $t > 1$, $p(s_{t} = k' | \z_{1:t})$ is also identifiable for all $k' \in \{1, ..., K \}$, as $p(\z_{t} | \z_{t-L_z:t-1}, s_{t}=k')$ and $p(s_{t} = k' | \z_{1:t-1})$ are identifiable as well (from Theorem \ref{thm:msm_transition_identifiability}) and then by applying Bayes' rule:
\begin{equation}
    p(s_{t} = k' | \z_{1:t}) \propto p(\z_{t} | \z_{t-L_z:t-1}, s_{t}=k') p(s_{t} = k' | \z_{1:t-1}).
    \label{eq:bayes_rule_discrete_state}
\end{equation}
Under an autonomous first-order Markov state transition model:
\begin{equation}
    p(s_{t+1} = k | \z_{1:t}) = \sum_{k'=1}^K p(s_{t} = k' | \z_{1:t}) p(s_{t+1} = k | s_{t} = k').
\label{eq:marginal_state_prob_s1}
\end{equation}
Denote the $K \times K$ transition matrix $p(s_{t+1} | s_{t})$ as $Q$ whose $k$th column is $Q_k = [p(s_{t+1} = k | s_{t} = 1), ..., p(s_{t+1} = k | s_{t} = K)]$. To determine $Q$, we can construct a linear system from Eq.~\eqref{eq:bayes_rule_discrete_state} by using different values of $\z_{1:t}$. Fix $\z_{1:t-1}$ and select $K$ different points $\{\bm{\beta}_1, \dots, \bm{\beta}_K\}$ for $\z_t$. Now, denote $A,B$ matrices:
\begin{equation}
    A =(a_1, \dots, a_K), \quad a_k =  p(s_{t}\mid \z_{1:t-1}, \z_t=\bm{\beta}_k),
\end{equation}
\begin{equation}
    B =(b_1, \dots, b_K)^\top, \quad b_k =  p(s_{t+1}\mid \z_{1:t-1}, \z_{t}=\bm{\beta}_k),
\end{equation}
which yields the following linear system:
\begin{equation}
    B = AQ.
\end{equation}
Using Eq.~\eqref{eq:bayes_rule_discrete_state} and given $p(s_{t}=k| \z_{1:t-1})>0$, the filtering posterior admits the following form:
\begin{equation}
    p(s_{t} = k | \z_{1:t}) = \frac{w_k p(\z_{t}| \z_{t-L_z:t-1}, s_{t}=k)}{\sum_{k'} w_{k'} p(\z_{t}| \z_{t-L_z:t-1}, s_{t}=k')}, \quad w_k> 0.
\end{equation}
And we can write $A$ as follows.
\begin{equation}
    A = \mathrm{diag}(w) G D^{-1},
\end{equation}
where
\begin{equation}
    G_{kj} = p(\z_{t}=\bm{\beta}_j | \z_{t-L_z:t-1}, s_{t}=k), \quad D_{jj} = \sum_{k=1}^K w_k p(\z_{t}=\bm{\beta}_j| \z_{t-L_z:t-1}, s_{t}=k).
\end{equation}
From $w_k>0$ and $p(\z_{t}=\bm{\beta}_j| \z_{t-L_z:t-1}, s_{t}=k) > 0$, $D$ is invertible. Furthermore, $G$ is invertible. To see why, Gaussian distributions with distinct means or covariances are linearly independent \citep{yakowitz1968identifiability}. Therefore, we can select $\bm{\beta}_1, ..., \bm{\beta}_K$, such that $G$ is invertible.
Since both $A$ and $B$ are identifiable and $A$ is invertible,
$Q = A^{-1}B$.
Therefore, the autonomous transition matrix is identifiable up to state permutations.
\end{proof}

\subsection{Extended discussion on disentanglement and Proof of Theorem~\ref{lemma:disentanglement}}\label{app:disentanglement_proof}
We begin with a simple example illustrating \citet[Assumption~(P2)]{kivva2022identifiability}.
\paragraph{Example.} Consider $m=3,\ K=2$ with $\bm{\sigma}_1 = (1, \tfrac12, \tfrac14)$, and $\bm{\sigma}_2 = (1, \tfrac14, \tfrac12)$. We note the above condition holds given $\bm{\sigma}_1/\bm{\sigma}_2 = (1, 2 ,\tfrac12)$.

The intuition of Theorem~\ref{lemma:disentanglement} starts by considering the case where the ratio vector has some, but not all, distinct entries. For a pair of regimes where \citet[Assumption~(P2)]{kivva2022identifiability} is violated, we have dimension groups with equal ratio (from labels $\{1,\dots,m\}$), such that we obtain identifiability across groups, but cannot resolve intra-group affine ambiguity. Therefore, with additional regimes, such that each regime resolves different groups, latent variable disentanglement could be obtained.

\paragraph{Example.} Assume $m=3$, but now $K=3$. Assume for $k=1,2$, the ratio vector is:
\begin{align*}
    & \sigma_{1,1} =1,\quad \sigma_{1,2} =0.25,\quad \sigma_{1,3} =0.25,\\
    & \sigma_{2,1} =1,\quad \sigma_{2,2} =0.25,\quad \sigma_{2,3} =0.5,\\
    & \bm\sigma_1/\bm\sigma_2 = (1, \quad 1, \quad 0.5).
\end{align*}

Intuitively, using a similar Kivva-style SVD-based proof, we can distinguish the first and second groups. Therefore, the affine ambiguity is reduced, as the equivalence introduces a block-diagonal matrix.
$\z'=PB\z + \Bb$, where for this example the block diagonal matrix is as follows:
\begin{equation*}
    B = \begin{pmatrix}
        b_1 & b_2 & 0 \\
        b_3 & b_4 & 0 \\
        0 & 0 & b_5 \\
    \end{pmatrix}
\end{equation*}
Introducing additional regimes will introduce additional restrictions to the equivalence, and therefore, if we obtain another pair, e.g., $k=2,3$, such that the relation for that becomes
\begin{equation*}
    B = \begin{pmatrix}
        b_1' & 0 & 0 \\
        0 & b_2' & b_3' \\
        0 & b_4' & b_5' \\
    \end{pmatrix}
\end{equation*}
we will obtain a diagonal term in the equivalence.

\subsubsection{Block-permutation disentanglement}

We formalise the above example by proving block-permutation identifiability. Assume we have a non-temporal GMM for $\z\in \R^m$, expressed as follows:
\begin{equation}
    \z \sim \sum_{k=1}^K c_k \mathcal{N}(\bm{\mu}_k, diag(\bm{\sigma}^2_k)).
\end{equation}

\begin{theorembox}
\begin{lemma}\label{lemma:block_diagonal_identifiaiblity}
    Let $K\geq 2$ and $c_k>0$ for all $k\in \{1,\dots,K\}$. Let $\z' = A\z+\Bb$, where $A\in\R^{m\times m}$ is invertible and $\Bb\in\R^m$. Assume there exist $k_1\neq k_2\in\{1,\dots,K\}$ such that the ratio vector $\bm{r}_{k_1,k_2}$ induces a partition $\{1,\dots,m\}=I_1\cup\dots\cup I_L$ where for any $\ell\in\{1,\dots,L\}$, $L\leq m$,
    \begin{equation}
        i,j \in I_\ell \ \iff\ r_{k_1,k_2,i} = r_{k_1,k_2,j}.
    \end{equation}
    and for $\ell\neq\ell',\ I_\ell\cap I_\ell'=\emptyset$.
    That is, each $I_\ell$ contains exactly those dimension labels with the same variance ratio. Then from $\z'$, we can recover an invertible matrix $A'\in\R^{m\times m}$ such that $(A')^{-1}A = P_1BP_2$, where $P_1, P_2$ are permutation matrices, and $B$ is a block diagonal matrix defined as follows:
    \begin{equation}
        B = \operatorname{blockdiag}(B_1, \dots, B_L)
    \end{equation}
    where each $B_{\ell}\in\R^{|I_\ell|\times|I_\ell|},\ \ell\in\{1,\dots,L\}$ is invertible.
\end{lemma}
\end{theorembox}
\begin{remark}
    In particular, if all ratios are distinct, i.e., $|I_\ell|=1$ for all $\ell\in\{1,\dots,L\}$ (equivalently $L=m$), $B$ is diagonal. This reduces to \citet{kivva2022identifiability}.
\end{remark}
\begin{remark}
    If coordinates are ordered by blocks, $P_2=I$, and $(A')^{-1}A = P_1B$.
\end{remark}
\begin{proof}

Let $\z$ be GMM-distributed with diagonal covariances $\Sigma_k=\mathrm{diag}(\bm{\sigma}_k^2),\ k\in\{1,\dots,K\}$. Let $\hat{\Sigma}_k$ denote the covariance matrix of the $k$-th component of $\z'$. Under the affine transformation $\z'=A\z + \Bb$, we have
\begin{equation}
    \hat{\Sigma}_k = A\Sigma_k A^\top,
\end{equation}
where each  $\hat{\Sigma}_k,\ k\in\{1,\dots,K\}$ is symmetric positive definite. Fix two regimes $k_1$ and $k_2$ (wlog. $k_1=1, k_2=2$) and consider the matrix $M:=\hat{\Sigma}^{-1}_2\hat{\Sigma}_1$:
\begin{equation}\label{eq:diagonal_product}
    M=\hat{\Sigma}^{-1}_2\hat{\Sigma}_1 = A^{-\top}\Sigma_2^{-1}A^{-1}A\Sigma_1 A^\top = A^{-\top}D^2 A^\top,
\end{equation}
where $D=diag(\bm{r}_{(1,2)})$. Given Eq \eqref{eq:diagonal_product}, $M$ is similar to the diagonal $D^2$, and hence is diagonalisable. From the eigen-decomposition of $M$, we may find invertible $U$ such that $M=UD^2U^{-1}$. Define $A':=U^{-\top}$. Then, $M=(A')^{-\top}D^2 (A')^\top$. Combining with Eq. \eqref{eq:diagonal_product} yields (after transposing)
\begin{equation}
    AD^2 A^{-1} = A'D^2 (A')^{-1} \implies SD^2 = D^2S, 
\end{equation}
with $S=(A')^{-1} A$. Given $D^2$ is diagonal, we have
\begin{equation}
    (d_j^2-d^2_i) S_{ij} = 0,\quad i,j\in\{1,\dots,m\}.
\end{equation}
This implies $S_{ij}=0$ whenever $d_i\neq d_j$, and no restriction appears if $d_i=d_j$. Given $d_i = r_{1,2,i}$, $R$ is structured from the partition $\{I_1, \dots, I_L\}$, since $I_\ell=\{i\in\{1,\dots,m\}: d_i=d_\ell\},\ \ell\in\{1,\dots,L\}$. Therefore, $S$ is block-diagonal with respect to the partition up to coordinate re-ordering. That is, there exists a permutation matrix $P$, such that:
\begin{equation}
    S = (A')^{-1}A = P^\top BP, \quad B = \operatorname{blockdiag}(B_1,\dots, B_L),
\end{equation}
where each block $B_{\ell}\in\R^{|I_\ell|\times|I_\ell|}, \ell\in\{1,\dots,L\}$ is invertible. Invertibility of each $B_{\ell}$ follows from invertibility of $S$. Finally, note that the eigendecomposition of $M$ is not unique, as we can reorder eigenvectors. Therefore, we may redefine $\tilde{A}':=A'\tilde{P}$. Then, $\tilde{S} = (A'\tilde{P})^{-1}A = \tilde{P}^\top S$. Since we have shown $S=P^\top B P$, it follows that
\begin{equation}
    \tilde{S} = (\tilde{P}^\top P^\top) B P.
\end{equation}
Since the choice of $\tilde{P}$ relabels coordinates, we therefore can obtain invertible $A'$ (replacing $\tilde{A}'$ by $A'$) such that
\begin{equation}
    (A')^{-1}A=P_1 BP_2,
\end{equation}
with $P_1, P_2$ permutation matrices and $B$ block-diagonal with invertible blocks.
\end{proof}

\subsubsection{Proof of Theorem~\ref{lemma:disentanglement}}

Lemma~\ref{lemma:block_diagonal_identifiaiblity} shows block-permutation identifiability for a single pair of regimes. When considering all possible pairs, we can construct the ratio matrix $R$ presented under (\ref{ass:disentanglement}), and the proof of Theorem~\ref{lemma:disentanglement} aggregates all regime pairs together to strengthen the block-permutation identifiability result into a single permutation of latent variables.

\begin{proof}
The proof follows largely from Lemma \ref{lemma:block_diagonal_identifiaiblity}, where the previous commutativity requirement now induces increasing restrictions for every row of the ratio matrix $R$. For each pair of regimes $i=(i_1, i_2)$ with $i_1< i_2,\ i_1,i_2\in\{1,\dots,K\}$, define $M_i :=\hat{\Sigma}^{-1}_{i_2}\hat{\Sigma}_{i_1}$. Using $\hat{\Sigma}_k = A\Sigma_kA^\top$,
\begin{equation}\label{eq:diagonal_2}
    M_i = A^{-\top}\Sigma_{i_2}^{-1}A^{-1}A\Sigma_{i_1} A^\top = A^{-\top}D_i^2 A^\top.
\end{equation}
where $D_i$ is diagonal, and $D_i=\operatorname{diag}(R_{i,:})=\operatorname{diag}(\bm{r}_{(i_1,i_2)})$ ($D_i$ is constructed from the $i$-th row of $R$). Therefore, all $M_i$ are diagonalisable under $A^{-\top}$, and furthermore, any two $M_i,M_j, i\neq j$ commute:
\begin{equation}
M_iM_j = A^{-\top}D_i^2 A^\top A^{-\top}D_j^2 A^\top = A^{-\top}D_i^2D_j^2 A^\top = A^{-\top}D_j^2D_i^2 A^\top = M_jM_i.
\end{equation}
Therefore, all $M_i$ are simultaneously diagonalisable under the same basis \citep[p. 64]{horn2012matrix}. That is, there exists an invertible matrix $U$ such that $M_i=UD_i^2U^{-1}$ for all $i\in\{1,\dots,\binom{K}{2}\}$. Define $A':=U^{-\top}$, then
\begin{equation}
M_i=(A')^{-\top}D_i^2(A')^\top.   
\end{equation}
Define $S:=(A')^{-1} A$. Combining with Eq. \eqref{eq:diagonal_2} (and transposing) yields $SD_i^2 = D_i^2S$. And since each $D_i^2$ is diagonal, we have
\begin{equation}
    (d_{i,k}^2-d^2_{i,j}) S_{kj} = 0,\quad k,j\in\{1,\dots,m\}, \quad \forall i\in\left\{1,\dots,\binom{K}{2}\right\}.
\end{equation}
This implies $S_{kj}=0$ whenever $d_{i,k}\neq d_{i,j}$. Given $d_{i,j}=R_{ij}$, and no two columns of $R$ are equal, this implies $S$ is a diagonal matrix. As in Lemma \ref{lemma:block_diagonal_identifiaiblity}, the eigendecomposition is defined up to permutation of eigenvectors. Replacing $A'$ for $A'P$, with $P$ a permutation matrix, transforms $S$ into $P^{\top}S$. Therefore, we can write $S$ up to relabelling:
\begin{equation}
    S = (A')^{-1}A = PD,
\end{equation}
with $P$ a permutation matrix and $D$ diagonal.
\end{proof}

\section{Method Details and Discussions}\label{app:estimation}

\subsection{Forward-backward Inference}\label{app:derivations_fb}
Below we provide further details on computing the posterior distributions $\gamma_{t,k}(\z_{1:T})=p(s_{t}=k \mid \z_{1:T})$ and $\xi_{t,k,l}(\z_{1:T})=p(s_{t}=k, s_{t-1}=l \mid \z_{1:T})$. Given a sequence $\z_{1:T}$ with $K$ regimes, for $k\in\{1,\dots,K\}$, the forward algorithm computes messages $\alpha_{t,k}(\z_{1:t}) = p(\z_{1:t},s_t=k)$ for $t=2,\dots,T$:
\begin{multline}\label{eq:forward}
    \alpha_{t,k}(\z_{1:t}) = p(\z_t \mid \z_{t-1}, s_t = k)  
\times \sum_{l=1}^K Q_{lk}(\z_{t-1}) \alpha_{t-1,l}(\z_{1:t-1}), \\ \alpha_{1,k}(\z_{1}) = p(\z_1\mid s_1=k)p(s_1=k),
\end{multline}
and the backward algorithm computes messages $\beta_{t,k}(\z_{t+1:T}) = p(\z_{t+1:T}\mid \z_t, s_t=k)$ for $t=T-1,\dots,1$:
\begin{equation}\label{eq:bck}
   \beta_{t,k}(\z_{t+1:T}) = \sum_{l=1}^K p_{\mparam}(\z_{t+1} | \z_{t}, s_{t+1} = l) 
\times Q_{kl}(\z_{t}) \beta_{t+1,l}(\z_{t+1:T}),
\end{equation}
with $\beta_{T,k}=1$ for all $k\in\{1,\dots,K\}$. The posteriors $\gamma_{t,k}(\z_{1:T})=p(s_{t}=k \mid \z_{1:T})$ and $\xi_{t,k,l}(\z_{1:T})=p(s_{t}=k, s_{t-1}=l \mid \z_{1:T})$ are then computed as follows 
\begin{multline}
    \gamma_{t,k}(\z_{1:T}) = \frac{\alpha_{t,k}(\z_{1:t})\beta_{t,k}(\z_{t+1:T})}{\sum_{j=1}^K\alpha_{t,j}(\z_{1:t})\beta_{t,j}(\z_{t+1:T})}, \\
    \xi_{t,k,l}(\z_{1:T}) = \frac{\alpha_{t-1,l}(\z_{1:t-1})\beta_{t,k}(\z_{t+1:T})p(\z_{t}\mid\z_{t-1}, s_t=k)Q_{lk}(\z_{t-1})}{\sum_{j=1}^K \alpha_{t,j}(\z_{1:t})\beta_{t,j}(\z_{t+1:T})}
\end{multline}
To avoid numerical issues, we normalise Equations~\eqref{eq:forward} and~\eqref{eq:bck} at each time step $t$ with respect to $s_t$, and run computations in log-space using log-sum-exp tricks. 

\subsection{$\Omega$SDS implementation details}\label{app:architecture}

\paragraph{Synthetic data.} As discussed in Section~\ref{sec:estimation}, for the synthetic experiments we parametrise the emission map using TarFlow \citep{zhai2025normalizing}. TarFlow provides an expressive invertible transformation that unmixes the observed features without requiring deep implementations. In contrast, affine couplings \citep{dinh2017density} require many coupling layers to perform nonlinear unmixing, which we found to destabilise training in higher-dimensional settings.

TarFlow was originally designed for image data. Considering the $n$-dimensional input vectors used in our synthetic dataset, we use the corresponding non-spatial configuration with channel size $1$ and patch size $1$. All the synthetic data experiments use $4$ TarFlow blocks, where the transformer architecture of each block uses $64$ hidden dimensions, $2$ layers, and $4$ heads. However, instead of the permutation operation used in the original TarFlow architecture, we use invertible LU linear mixing layers, as in RealNVP \citet{dinh2017density}. We found this modification important in practice, as keeping the original permutation operation led to training failure with the rMSM failing to recover regime information. The output of the final flow block is split into rMSM features $\z_t$ and exogenous noise $\eps_t$.

\paragraph{Video data.} To train videos, we instead use a multi-scale Glow-style architecture \citep{kingma2018glow}, which provides scalability in high-dimensional image inputs. The architecture consists of invertible $1\times 1$ convolutions, squeeze operations, split operations, and coupling layers. Following the original Glow design, the flow is organised into $B$ blocks, each containing $L$ coupling layers. At each block, the squeeze operation downsamples the spatial dimension by a factor of $2$, and rearranges the extra dimensions into additional channels. The split operation then discards half of the features from subsequent flow operations, but retains them as part of the latent representation. 

This hierarchical structure provides a decomposition that we use to distinguish temporal features from exogenous noise. For an input image $\x_t$ and $B$ blocks, the inverse flow computes:
\begin{equation*}
    f^{-1}(\x_t) = \big(\eps_t^{(1)}, \dots, \eps_t^{(B-1)}, \z_t\big),
\end{equation*}
where the rMSM features represent the deepest features of the flow, and the remaining features are aggregated as exogenous noise $\eps_t = \operatorname{concat}(\eps_t^{(1)}, \dots, \eps_t^{(B-1)})$.
The marginal likelihood is then computed following Eq.~\ref{eq:omega_sds_likelihood}. For training, both the bouncing-ball and dancing-video experiments use $7$ coupling layers per block with $128$ channels. We use $4$ blocks for bouncing balls, and $5$ blocks for dancing videos. We found that using deeper architectures made the model more prone to state collapse.

\paragraph{rMSM implementation details} Our latent rMSM is implemented similarly to \citet{balsells-rodas2024on}. For synthetic data, the transitions consist of two-layer MLPs with $32$ hidden dimensions and a nonlinear activation that matches the one used in the data-generation process (see Section~\ref{app:synthetic_data}). For video experiments, we use three-layer MLPs with $128$ hidden dimensions and GeLU activations. In this setting, the recurrent switching process $Q$ is parametrised using a two-layer MLP.

\subsection{Likelihood-based vs VAE-based estimation}\label{app:vae_estimation}

To further motivate likelihood-based models in identifiable deep generative models, we extend the discussion in Section~\ref{sec:experiments} by comparing the exact likelihood objective of $\Omega$SDS with VAE-based alternatives. Our goal is to understand how the form of a variational approximation can affect structure learning in contrast to exact inference under the same generative modelling assumptions.

We consider synthetic data with $n=m=10$ dimensions and $K=4$, where further data-generation and training details are discussed in Appendix~\ref{app:synthetic_data}. Table~\ref{tab:nonlinear_snlds_posteriors} summarises the results in terms of regime $F_1$ score and latent recovery measured by MCC. We consider three variational approximations under the collapsed-switching factorisation $q(\z_{1:T}\mid \x_{1:T})\,p(\s_{1:T}\mid \z_{1:T})$ used in \cite{balsells-rodas2024on}, together with one fully amortised posterior of the form $q(\z_{1:T}\mid \x_{1:T})\,q(\s_{1:T}\mid \x_{1:T})$. The latter uses bidirectional RNNs to encode both continuous and discrete latent variables, and $\s_{1:T}$ is sampled using Gumbel-softmax relaxation \citep{jang2017categorical}. For the collapsed-switching family, we consider the following choices for the continuous posterior: (1) a factored encoder $q(\z_{1:T}\mid\x_{1:T})=\prod_t q(\z_t\mid \x_{1:T})$ implemented with a shared two-layer MLP, (2) the same factored encoder under IWAE \citep{burda2015importance} using $10$ importance samples, and (3) an encoder which uses a bi-directional RNN, followed by a forward causal RNN, as proposed in \citet{dong2020collapsed,ansari2021deep}.

The results show that VAE-based estimators can recover the regime information accurately, but produce substantially worse continuous latent representations. Furthermore, the choice of variational family is critical, as the more expressive amortised posterior fails in both structure learning and regime estimation. In comparison, $\Omega$SDS achieves strong performance in both regime estimation and structure learning. These results suggest that identifiability alone is not sufficient for reliable latent recovery, since the estimation method plays a crucial role in uncovering the underlying dynamics. By directly optimising the exact likelihood, $\Omega$SDS avoids estimation bottlenecks that can arise from the variational posterior design. We note this result is more pronounced in settings with many latent dimensions and a comparable number of observed dimensions. In our experiments, VAE-based disentanglement improves in lower-dimensional settings ($m=3$ or $m=5$), or when the observed dimension is much larger than the latent dimensions, e.g., $n=5m$ (see Section~\ref{sec:synthetic_exp}).

\begin{table}[t]
     \centering
     \caption{Comparison of different posterior parameterisations on the synthetic data.}
     \begin{tabular}{lcc}
         \toprule
         Posterior & MCC & Regime $F_1$ \\
         \midrule
         $\prod_tq(\z_{t}\mid \x_{t})\,p(\s_{1:T}\mid \z_{1:T})$ & 0.621 & 0.995 \\
         $\prod_tq(\z_{t}\mid \x_{t})\,p(\s_{1:T}\mid \z_{1:T})$ + IWAE & 0.621 & 0.998 \\
         $q(\z_{1:T}\mid \x_{1:T})\,p(\s_{1:T}\mid \z_{1:T})$ (BiRNN+RNN) & 0.622 & 0.990 \\
         $q(\z_{1:T}\mid \x_{1:T})\,q(\s_{1:T}\mid \x_{1:T})$ & 0.629 & 0.757  \\
         \midrule
         $\Omega$SDS & 0.977 & 0.993 \\
         \bottomrule
     \end{tabular}
     \label{tab:nonlinear_snlds_posteriors}
 \end{table}

\section{Experimental Details}\label{app:experimental_details}

\subsection{Assumption to Dataset Correspondence}\label{app:assumption_aligmnent}

Below, we discuss the correspondence between the assumptions that enable our theoretical results and the synthetic scenarios explored in our work. For real-world data, establishing such correspondence without domain knowledge is an inherent limitation.
\begin{itemize}
    \item[\eqref{ass:emission}]
    The synthetic data satisfy the invertible emission assumption by construction, including the noise-augmented setting. For bouncing-ball videos, the ball is rendered in pixel space from points in the box, but invertibility of the image renderer is difficult to ensure, since nearby points can yield the same pixel representation.
    \item[\eqref{ass:gaussian_transition}]
    The synthetic data use neural networks with cosine activations as transition means and regime-dependent covariances. Bouncing-ball dynamics are linear when conditioned on the regimes, as shown in Eq.~\eqref{eq:bouncing_ball}.
    \item[\eqref{ass:recurrent}]
    Synthetic data assume autonomous switching, which is a special case of the recurrent model, while bouncing-ball switches depend on box collisions.
    \item[\eqref{ass:sticky_switch}]
    For synthetic data, we set the diagonal entries of the transition matrix $Q$ to $0.9$. For bouncing balls, the transition $Q(\z)$ is close to a diagonal in the interior of the box and changes at collision boundaries.
    \item[\eqref{ass:posterior_domi}]
    For likelihood separation on synthetic data, we use causally-governed dynamics, where differences in regime-dependent graphs encourage likelihood separation as discussed in Appendix~\ref{app:dominance_assumptions}. Below, we numerically verify the required transition-separation conditions for the generated synthetic datasets. For bouncing ball dynamics, distinct velocities and low variance yield a log-likelihood margin of $2(\Delta t)^2/\sigma^2$ using Eq.~\ref{eq:bouncing_ball}. High switching stickiness in the interior gives $\epsilon^*\approx0$. Therefore, on points sufficiently far from the box edges, we can construct $K=4$ histories with a common endpoint. As discussed in Section~\ref{sec:first_stage}, Figure~\ref{fig:bouncing_ball} illustrates transition separation.
    \item[\eqref{ass:regime_domi}]
    In the synthetic data, we enforce covariance dominance through data generation. The bouncing-ball dynamics showcase the equal-variance case, where the distinct velocities $\bm v_k$ form the vertices of a convex hull in $\R^2$, as discussed in Section~\ref{sec:first_stage}.
    \item[\eqref{ass:nondegeneratus}]
    The synthetic data use full-rank transition structures, and the regime-dependent linear dynamics in the bouncing-ball data are full-rank in $\R^2$.
    \item[\eqref{ass:disentanglement}]
    Similarly to \eqref{ass:regime_domi}, we enforce this condition in the synthetic data. For bouncing-ball videos, this assumption does not hold since we assume equal variances.
\end{itemize}

Since the synthetic data give us access to the underlying dynamics, we can numerically verify whether the likelihood separation conditions given by \eqref{ass:sticky_switch}--\eqref{ass:posterior_domi} are satisfied in practice. We assess this following the discussion in Appendix~\ref{app:dominance_assumptions}, where Corollary~\ref{cor:half_dominance_time} shows how likelihood separation and switching stickiness determine the required history length. Since we consider switch transitions with $Q_{kk}=0.9,\ k=1,\dots,K$, we have $\epsilon^*=0.1$, and the required likelihood-separation threshold from (\ref{ass:posterior_domi}) is
\begin{equation}
\Delta_{\mathrm{sep}} = \log\!\left(\frac{1+\epsilon^*}{1-\epsilon^*}\right)
= \log(1.1/0.9) \approx 0.2007.
\label{eq:numerical-required-margin}
\end{equation}
Therefore, given an endpoint $\z^*$ and a predecessor point $\z_1$, we can compute the likelihood margin exactly
\begin{equation}
\Delta_k(\z_1,\z^*)
:=
\log p(\z^*\mid \z_1,s_2=k)
-
\max_{k'\in\{1,\dots,K\}\setminus \{k\}}
\log p(\z^*\mid \z_1,s_2=k'),
\label{eq:numerical-transition-margin}
\end{equation}
which is given by the sampled regime-dependent means $\bm m_k$ and variances $\bm \sigma_k^2$. Since \eqref{ass:posterior_domi} requires the existence of $K$ histories, we fix the common endpoint to $\z^*=0$ and consider two-step histories $(\z_1^{(k)},0)$ for each $k\in\{1,\dots,K\}$. Then, we find a predecessor $\z_1^{(k)}$ that gives a large value of $\Delta_k(\z_1^{(k)},0)$ while maintaining posterior dominance. That is, we compute $\z_1^{(k)}=\arg\max_{\z_1} \Delta_k(\z_1,0)$ subject to $p(s_2=k\mid \z_1,0)>\frac12$. After repeating this process for each regime, the smallest likelihood margin is $\Delta_{\min}= \min_{k\in\{1,\dots,K\}} \Delta_k(\z_1^{(k)},0)$. 
Table~\ref{tab:numerical-transition-separation} reports the numerical verification for distinct dimension and regime-number configurations, where $\min_s\Delta_{\min}^{(s)}$ denotes the minimum transition separation across the $5$ random data seeds. For each dataset configuration, the selected histories satisfy $\min_{s}\Delta_{\min}^{(s)}>\Delta_{\mathrm{sep}}$, confirming that \eqref{ass:posterior_domi} holds for our synthetic data.
\begin{table}[t]
\centering
\caption{Numerical \eqref{ass:posterior_domi} verification, where $\min_{s}\Delta_{\min}^{(s)}$ denotes the minimum margin found across the $5$ random dataset seeds.}
\label{tab:numerical-transition-separation}
\begin{tabular}{c c c}
\toprule
$m$
& $K$
& $\min_s\Delta_{\min}^{(s)}$ \\
\midrule
$3$  & $3$ & $3.80$ \\
$5$  & $3$ & $7.83$ \\
$10$ & $4$ & $11.93$ \\
$20$ & $5$ & $53.27$ \\
\bottomrule
\end{tabular}
\end{table}

\subsection{Metrics}
\paragraph{Regime $F_1$ score} We evaluate the performance of our method to extract regimes by computing the $F_1$ score in terms of estimated switches $\hat{\s}_{1:T}$. We compute each $\hat{s}_t$ by selecting the regime label with highest probability weight according to the marginal posterior conditioned on $\z_{1:T}$: i.e., $\hat{s}_t = \arg\max_{k\in\{1,\dots,K\}} p(s_{t}=k\mid\z_{1:T})$.

\paragraph{MCC} Mean correlation coefficient (MCC) is a standard metric to evaluate identifiable generative models \citep{khemakhem2020variational}. It evaluates whether the learned representation agrees with the ground truth up to permutation and component-wise nonlinearity. Therefore, it fits our linear scaling and permutation equivalence class for the continuous latent variables.

\paragraph{LPIPS} LPIPS \citep{zhang2018unreasonable} measures perceptual similarity between two images using deep visual features extracted from a pretrained network. Compared with pixel-wise errors such as MSE, LPIPS better reflects human-perceived visual differences, making it suitable for evaluating the perceptual quality of video rollouts.

\subsection{Baseline specifications}\label{app:baselines}

Throughout our empirical evaluation, we compare our method against baselines for regime estimation, disentanglement, and forecasting. Below, we briefly describe each baseline and provide implementation details for reproducibility. 
\paragraph{Synthetic data.}
For the synthetic experiments, we consider the following baselines:
\begin{itemize}[leftmargin=2em]
    \item \textbf{iSDS} \citep{balsells-rodas2024on}, based on non-recurrent SDSs with piecewise linear emissions.
    \item \textbf{CtrlNS} \citep{song2024causal}, based on observation-dependent switching with a known number of regimes $K$.
    \item \textbf{NCTRL} \citep{song2023temporally}, which assumes non-Gaussian temporal latent variables and Gaussian observations.
    \item \textbf{TDRL} \citep{yao2022temporally}, which assumes a stationary non-Gaussian temporal representation.
    \item \textbf{iVAE} \citep{khemakhem2020variational}, a non-temporal identifiable generative model. We provide ground-truth regime information as auxiliary information to obtain a stronger baseline.
\end{itemize}
For all synthetic-data baselines, public code is available and used in our experiments. Our synthetic dataset can be directly imported into the model training pipeline, and we use the default training configuration. Since SDS training can fail due to poor local optima, we train each model with $3$ random seeds for each dataset and select the checkpoint with the best validation loss, corresponding to the ELBO for all cases.

\paragraph{Bouncing Balls}
For bouncing balls, we compare against VAE-based baselines, which use the same encoder-decoder representation for fairness. The encoder consists of a $5$-layer CNN with $64$ channels, padding $1$, and kernel size $3$, with LeakyReLU activations, and strides alternating between $2$ and $1$. A $2$-layer MLP with LeakyReLU activations and $64$ hidden units is used as the posterior distribution head. The decoder uses a similar architecture in reverse, starting with a $2$-layer MLP with $64$ hidden dimensions and LeakyReLU activations, followed by transposed convolutions under the same configuration. We use the following baselines:
\begin{itemize}[leftmargin=2em]
\item \textbf{iSDS} \citep{balsells-rodas2024on}, described above. Notably, the encoder-decoder architecture matches the original implementation.
\item \textbf{REDSDS} \citep{ansari2021deep}, which proposes recurrent explicit-duration models and uses the same collapsed-switching as in \citet{dong2020collapsed}. To improve segmentation results, we explore several annealing schedules implemented in their publicly available code to prevent state collapse and select the best result on validation ELBO.
\item \textbf{SNLDS} \citep{dong2020collapsed}, which combines collapsed-switching with recurrent switching dynamics. This is the closest VAE-based design to $\Omega$SDS (R). We use the publicly available implementation from REDSDS. For both SNLDS and REDSDS, we adapt the implementations to image data, which was not originally supported.
\item \textbf{KVAE} \citep{fraccaro2017disentangled}, where a soft-switching Kalman filter is used as the prior dynamics. We use our own implementation. To encourage regime learning, we use a warm-up phase where the prior dynamics weights are detached for $1$k iterations, as suggested in the original paper.
\item \textbf{VRNN} \citep{chung2015recurrent}, a recurrent latent variable model for sequential data. We use our own implementation.
\end{itemize}

\paragraph{Dancing Videos.}
For dancing videos, we compare with iSDS \citep{balsells-rodas2024on}. Since they use the same dataset in their original work, we use the same architecture as described in their implementation details. 

\subsection{Synthetic Data}\label{app:synthetic_data}

For synthetic evaluation, we generate data using a similar procedure described in \citet{balsells-rodas2024on}, with several modifications. Since we aim to compare our likelihood-based formulation against VAE-based baselines, we consider a non-recurrent dataset to align with previous approaches. For each dataset, we sample $10{,}000$ train sequences and $1{,}000$ test sequences, both with length $T=100$. 

The latent regime dynamics evolve in terms of a first-order Markov chain, where $Q$ is set to maintain the same regime with probability $90\%$ and transition to the next one with probability $10\%$, and the initial distribution is uniform over the regimes. The initial latent distributions are sampled from random Gaussian components, and the mean transitions $\bm m_k$ are parametrised using locally connected networks \citep{zheng2018dags} where the sparsity is set to allow $3$ interactions per element on average. The networks consist of two-layer MLPs with 16 hidden units and cosine activations. We modify data generation in terms of \citet{balsells-rodas2024on}, where we use regime-dependent variances and sample scales $\sigma_k\sim\mathcal U(0.001,0.5)$. To avoid ambiguity in disentanglement during learning, we verify the variance ratios for pairs of regimes as described in Section~\ref{sec:stage_3}, and resample if the ratios are too similar across dimensions and regime pairs. Specifically, for $m=3,5,10,20$, we resample if no pair of regimes has dimension-wise ratio values that differ by more than $0.35,\ 0.35,\ 0.10,\ 0.05$, respectively. For noise-augmented settings, we concatenate the latent variables with noise sampled from $\N(0,0.01)$ at each time step. Observations are then generated using random two-layer LeakyReLU networks with a hidden dimension equal to the observation dimension.

In our experiments, we increase the number of regimes $K$ with the latent dimension, setting $K=3,3,4,5$ for $m=3,5,10,20$. Sparsity of the transitions is also adjusted with dimension as described above. For each setting, we generate $5$ data seeds and train our models on NVIDIA GeForce RTX 2080 Ti GPUs. We use the Adam optimiser \citep{kingma2014adam} with a learning rate of $5\cdot10^{-3}$ for the first $100$ epochs, after which we reduce it to $5\cdot10^{-4}$. In all our settings, we detach the transition matrix $Q$ for the first $10$ epochs. In the noise-augmented settings, we additionally use PCA-based initialisations to encourage better initial MSM parametrisations and prevent state collapse. We train $\Omega$SDS for $1000$ epochs with batch size $64$, which takes roughly 3 hours per seed for $m=3$ and $16$ hours for $m=20$. The increased training time is due to the causal connectivity in TarFlow, as it considers a single-vector sequence. Future work could explore using a similar patch-based architecture for vector-valued inputs to improve scalability.

\subsubsection[Ablation on A8]
{Ablation on \eqref{ass:disentanglement}}

In Section~\ref{sec:stage_3}, we introduce \eqref{ass:disentanglement} as a generalisation of \citet[Assumption P2]{kivva2022identifiability} for non-temporal mixture generative models. As discussed, \citet[Assumption P2]{kivva2022identifiability} establishes disentanglement by considering differences in variance ratios from a single pair of regimes. Instead, \eqref{ass:disentanglement} relaxes this requirement by exploiting all possible $\binom{K}{2}$ regime pairs. This generalisation is useful when the latent dimensionality increases, since dimension-wise ratio differences within a single pair of regimes may become smaller. 

\begin{figure}
    \centering
    \vspace{-.65em}
    \includegraphics[width=0.5\linewidth]{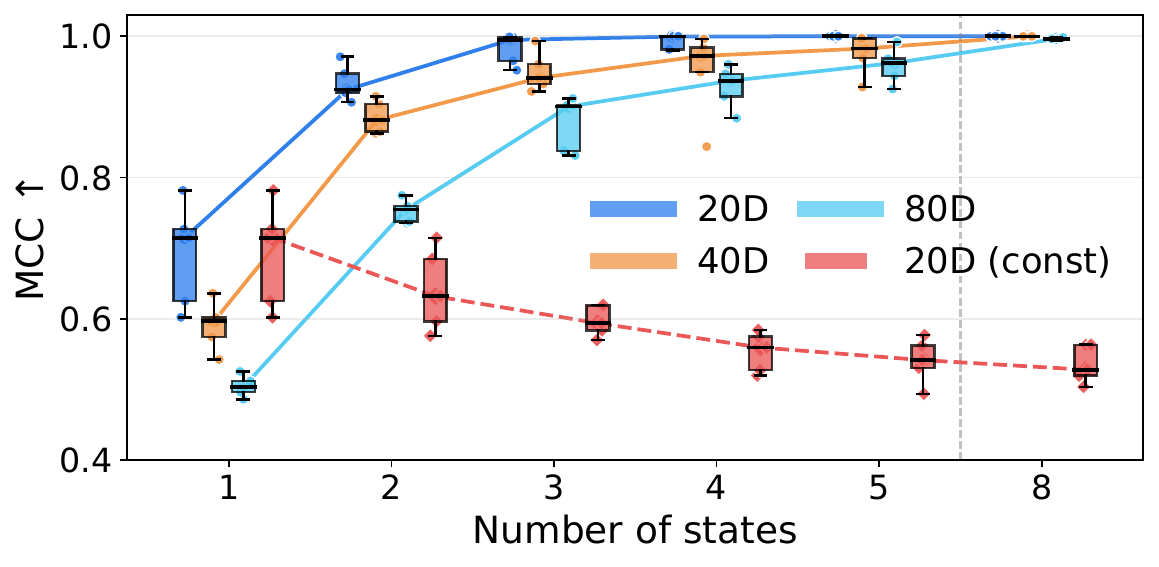}
    \caption{Ablation on \eqref{ass:disentanglement}. Strong MCC on linear data for increasing regimes and dimensions with regime-dependent variances, with ``const'' using constant variances across regimes. Results are reported using $5$ data seeds.}
    \label{fig:mcc_ablation_states}
    \vspace{-1em}
\end{figure}
We assess the scalability of \eqref{ass:disentanglement} and experiment with $\Omega$SDS on noiseless linear emissions and the same MSM settings described above. We consider increasing latent dimensions $m\in\{20, 40, 80\}$, increasing regimes $K\in\{1,2,3,4,5,8\}$, and randomly sampled regime-dependent variances for each regime. We also include a control setting with constant variances across regimes and $m=20$. Figure~\ref{fig:mcc_ablation_states} reports the MCC of the learned latent features by $\Omega$SDS in each setting, across $5$ data seeds. We observe that increasing dimensions makes disentanglement more challenging, but the results improve when introducing additional regimes with regime-dependent variances. However, maintaining constant variances results in worse disentanglement. This shows that the regime-dependent heterogeneity condition introduced by \eqref{ass:disentanglement} is a scalable solution for learning disentangled representations.

\subsubsection{Selecting the number of regimes}

\begin{figure}
    \centering
    \includegraphics[width=.475\linewidth]{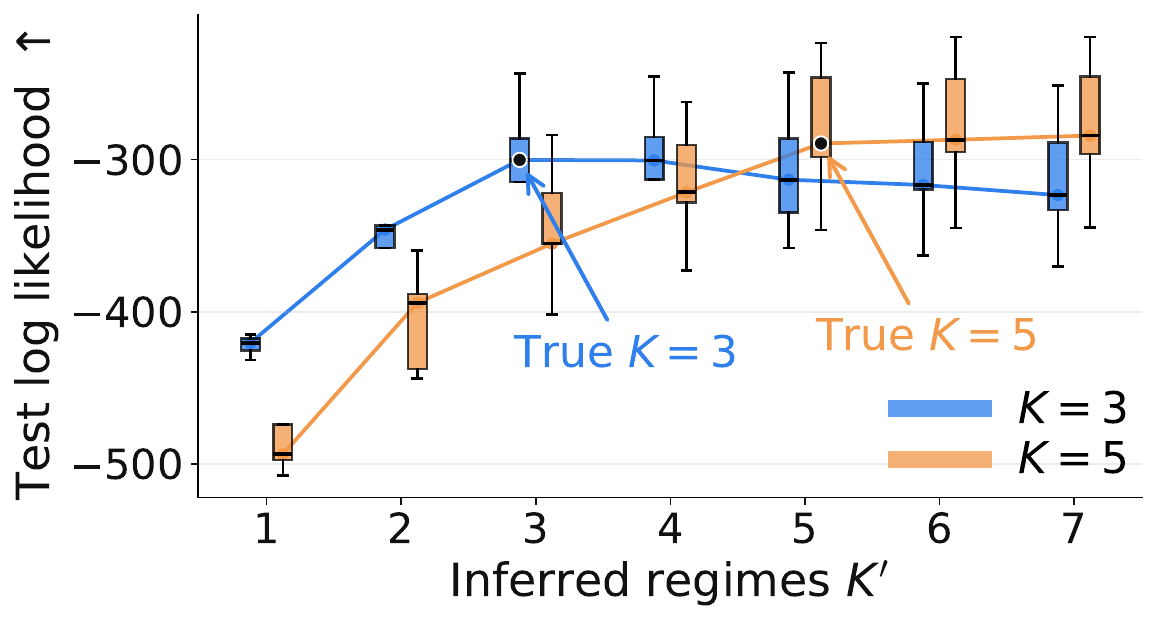}
    \caption{Model selection results with ground-truth numbers of regimes $K=3$ and $K=5$, with test log-likelihood across $5$ data seeds.}
    \label{fig:model_selection_results}
\end{figure}
For scenarios where the number of regimes $K$ is unknown, it needs to be selected a priori for $\Omega$SDS learning. Theorem~\ref{thm:affine_identifiability}~(ii) establishes identifiability of $K$ when $Q$ is analytic. This implies that given a generated dataset with pre-defined $K$, we cannot have another model with $K'$ such that $K'\neq K$ with the same data likelihood. This connects to model selection in clustering methods such as the elbow method \citep{thorndike1953belongs}, where the selection rule considers the test log-likelihood plateau for increasing number of regimes $K$. This approach is similar to \citet{balsells-rodas25causal}, where identifiability of $K$ allows for such model selection heuristics. In our synthetic experiments, the non-recurrent $Q$ already satisfies analyticity since it is constant with respect to $\z$. Figure \ref{fig:model_selection_results} shows model selection results for $K'=1,\dots,7$ for noiseless linear emissions with $m=n=5$, where we consider scenarios with $K=3$ and $K=5$, and $5$ data seeds. The results show that the test log likelihood of $\Omega$SDS plateaus when $K'=K$, and stays lower for $K'<K$. Furthermore, inspection of the regime posteriors on test data shows that the total usage of redundant states stays below 1\% across all seeds for $K'=6,7$ when $K=5$, and $K'=4$ when $K=3$. When likelihood decreases substantially, the average total usage of redundant states increases from 1.75\% at $K'=5$ to 4.75\% at $K'=7$ for $K=3$. Therefore the elbow selection rule coincides with the observed likelihood plateau at the ground-truth number of regimes. 

\subsubsection{Recurrent switching processes}

We explore recurrent switching in terms of our synthetic dataset, by maintaining the same regime-dependent dynamics on the continuous latent variables, and modifying the non-recurrent transition matrix $Q$ used in Section~\ref{sec:synthetic_exp}. We parametrise $Q(\z)$ as follows:
\begin{equation}
    Q(\z)_j = \mathrm{softmax}(\bm u_j(\z)), \quad
    u_{jk}(\z)=\begin{cases}
        \alpha, & j=k, \\
        \bm \beta_{jk}^\top\tanh(\z), & j\neq k,
    \end{cases}
\end{equation}
where $Q(\z)_j$ denotes the $j$-th row of $Q(\z)$ and $u_{jk}$ denotes the corresponding unnormalised logit. Each vector $\bm \beta_{jk}\in\R^m$ is sampled from $\mathcal{N}(\bm 0, I_m)$ and $L_1$-normalised to ensure $|\bm\beta_{jk}^\top\tanh(\z)| \leq 1$ for stability. $\alpha$ is chosen to ensure stickiness. Similarly to the non-recurrent case, we enforce $Q_{jj}(\z)=0.9$ at $\z=\bm{0}$:
\begin{equation}
    \alpha= \log \left(\frac{(K-1)0.9}{1-0.9}\right),
\end{equation}
which gives $\alpha\approx2.89$ for $K=3$. Since each off-diagonal logit lies in $[-1,1]$, for $K=3$ the diagonal entries are bounded by: $0.768 \leq Q_{jj}(\z) \leq 0.961$.

\begin{table}
\centering
\caption{Results with recurrent $Q(\z)$. Values are mean $\pm$ 95\% CI across $5$ data seeds.}
\label{tab:recurrent-q-results}
\begin{tabular}{lcccc}
\toprule
Method & MCC $\uparrow$ & Regime $F_1$ $\uparrow$ & Q off-diag MAE $\downarrow$ & Q diag MAE $\downarrow$ \\
\midrule
iSDS & $0.722 \pm 0.114$ & $0.790 \pm 0.211$ & $0.101 \pm 0.097$ & $0.198 \pm 0.197$ \\
$\Omega$SDS & $\mathbf{0.946 \pm 0.052}$ & $\mathbf{0.986 \pm 0.012}$ & $\mathbf{0.022 \pm 0.010}$ & $\mathbf{0.039 \pm 0.019}$ \\
\bottomrule
\end{tabular}
\end{table}

The above choices ensure that the switching process introduces feedback from the continuous latents, while maintaining stability and stickiness across different time steps and data seeds. The remaining configurations follow the non-recurrent case, including the nonlinear emissions and regime-dependent dynamics. We consider $m=5$, $K=3$, and noiseless emissions across $5$ data seeds. Table~\ref{tab:recurrent-q-results} shows results for $\Omega$SDS and iSDS, where we report MCC, regime $F_1$ score, and mean absolute error (MAE) for both diagonal and off-diagonal transition entries, after permutation alignment. We observe that $\Omega$SDS achieves high MCC and regime $F_1$, and accurately captures the recurrent transition parametrisation $Q(\z)$. In contrast, the non-recurrent iSDS performs worse on both regime and continuous latent recovery because its fixed transition matrix is misspecified.

\subsection{Bouncing Ball} \label{appendix:bouncing_ball}


\begin{table}[t]
\centering
\caption{Regime estimation and long-rollout quality reported in Table~\ref{tab:bb32_regime_rollout} with $95\%$ CI for completeness.}
\label{tab:bb32_regime_rollout_ci}
\resizebox{\textwidth}{!}{
\begin{tabular}{lcccc}
\toprule
Method
& Input $F_1 \uparrow$
& Pred. $F_1 \uparrow$
& LPIPS@20 $\downarrow$
& LPIPS@500 $\downarrow$ \\
\midrule
VRNN
& --
& --
& $0.014 \pm 0.016$
& $0.028 \pm 0.021$ \\
KVAE
& $0.105 \pm 0.020$
& $0.105 \pm 0.029$
& $0.053 \pm 0.029$
& $0.324 \pm 0.014$ \\
iSDS
& $0.819 \pm 0.027$
& $0.345 \pm 0.032$
& $0.111 \pm 0.034$
& $0.106 \pm 0.022$ \\
REDSDS
& $0.655 \pm 0.025$
& $0.580 \pm 0.051$
& $0.073 \pm 0.067$
& $0.135 \pm 0.051$ \\
SNLDS
& $0.998 \pm 0.001$
& $0.913 \pm 0.007$
& $0.007 \pm 0.003$
& $0.059 \pm 0.017$ \\
\midrule
$\Omega$SDS (auto)
& $0.867 \pm 0.023$
& $0.362 \pm 0.030$
& $0.110 \pm 0.034$
& $0.092 \pm 0.015$ \\
$\Omega$SDS (recurrent)
& $\mathbf{1.000 \pm 0.000}$
& $\mathbf{0.932 \pm 0.004}$
& $\mathbf{0.003 \pm 0.002}$
& $\mathbf{0.027 \pm 0.026}$ \\
\bottomrule
\end{tabular}
}
\end{table}

As described in Section~\ref{sec:videos}, we generate videos of a bouncing ball with a resolution of $32\times32$. All included methods are trained for $30$k steps with a batch size of $64$ and a video length of $T=64$, and we use $100$ held-out samples for evaluation. In this setting, we use the Adam optimiser \citep{kingma2014adam} with a learning rate of $10^{-4}$ for the flow architecture and $10^{-3}$ for the rMSM architecture. We do not use PCA-based initialisation as described in Section~\ref{sec:experiments}. Instead, we freeze the weights of the switching network $Q$ for $1$k iterations. Training $\Omega$SDS takes roughly 2 hours on 4 NVIDIA RTX 5080 GPUs. We use uniform dequantisation before centering the pixels, following \citet{kingma2018glow}. Specifically, each normalised pixel value is mapped from a discrete bin to a continuous value by adding uniform noise within its quantisation bin. Unless otherwise stated, we use $32$ dequantisation bins and fixed-seed noise for evaluation to ensure deterministic comparisons across methods. The added dequantisation noise relaxes the discrete 8-bit pixel values into a continuous input distribution while leaving the perceptual appearance of the frames essentially unchanged. Because this small pixel-level noise can disproportionately increase image-space MSE, we report rollout frame quality using the perceptual LPIPS metric instead.

We show representative regime segmentations across baselines in Figure~\ref{fig:appendix_bb_regime}, and visualisations of rollout predictions for both short and long horizons in Figures~\ref{fig:appendix_bb_rollout_short} and~\ref{fig:appendix_bb_rollout_long}. These examples illustrate that $\Omega$SDS correctly identifies the distinct velocity patterns of the ball and produces accurate forecasts. Importantly, we distinguish the non-recurrent setting, represented by iSDS and $\Omega$SDS (A), from the recurrent setting, represented by $\Omega$SDS (R) and SNLDS. Notably, we observe that KVAE makes the ball disappear as time passes, while no-switching baselines move the ball or split it across boundaries. The improved results in the recurrent case motivate the development of our identifiability theory for recurrent dynamics.

\begin{figure}[htbp]
    \centering
    \includegraphics[width=\linewidth]{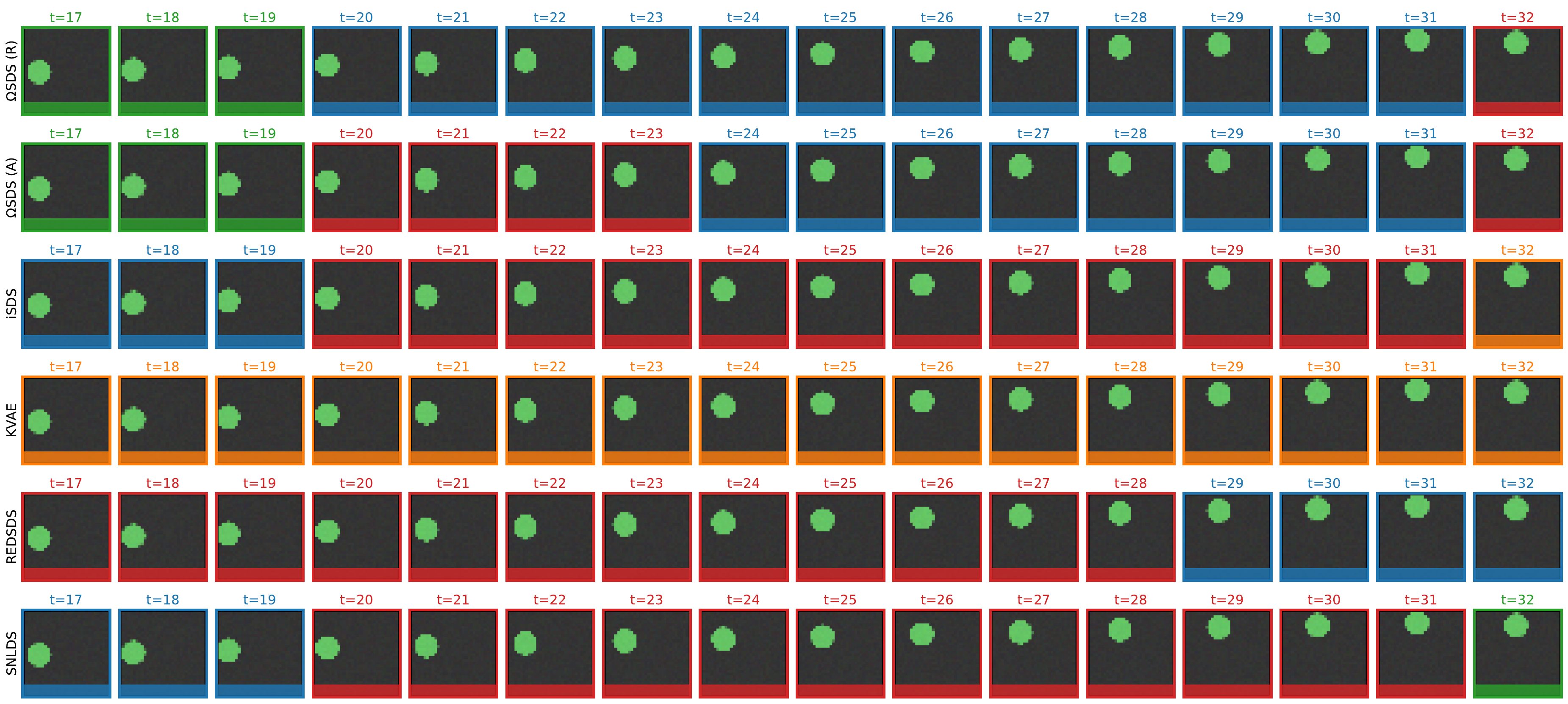}
    \caption{Posterior regimes detected by different models. Regimes are recovered up to permutations.}
    \label{fig:appendix_bb_regime}
\end{figure}

\begin{figure}[htbp]
    \centering
    \includegraphics[width=\linewidth]{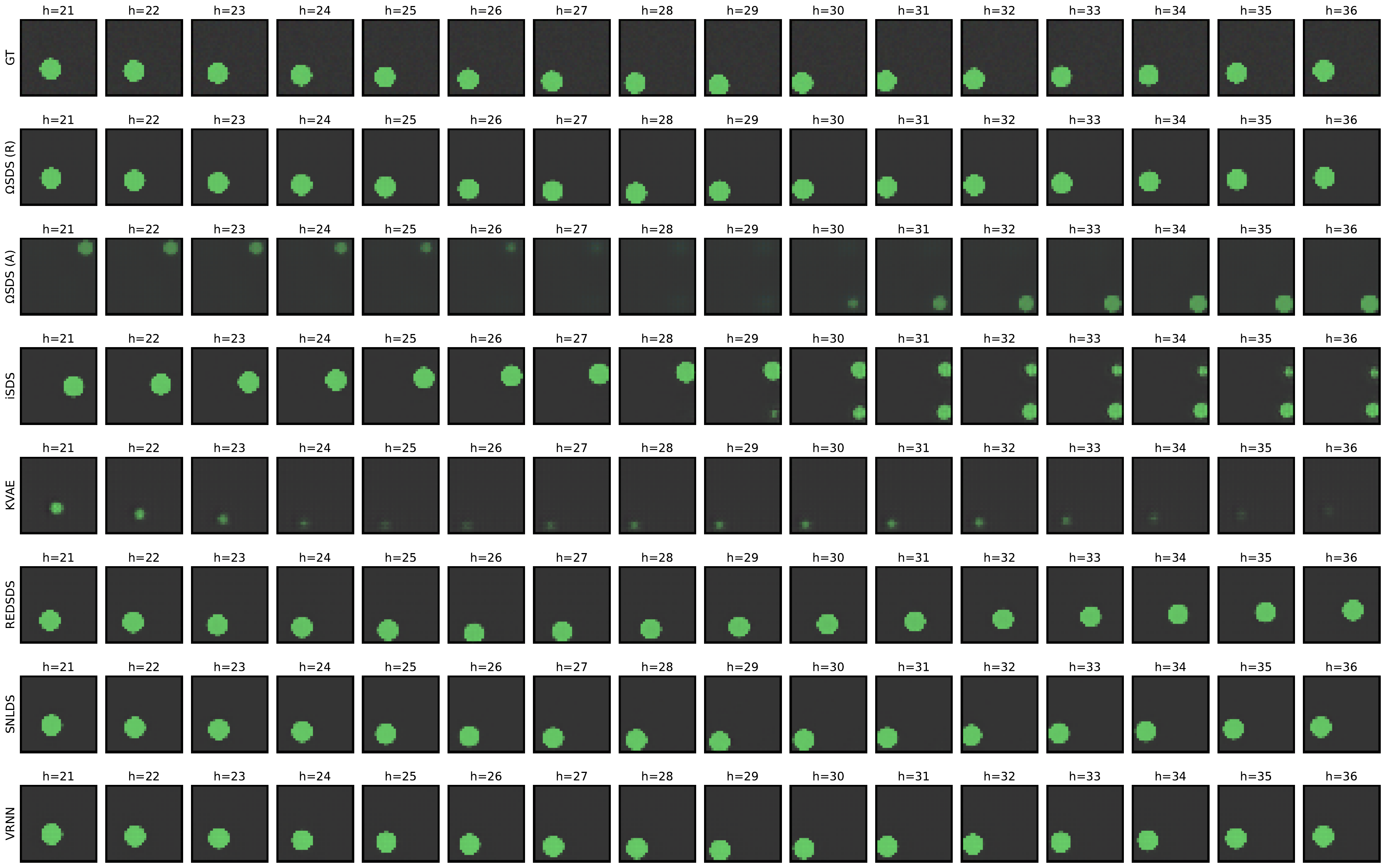}
    \caption{Comparison of short predictive rollouts with different methods. The given context length is 5. }
    \label{fig:appendix_bb_rollout_short}
\end{figure}

\begin{figure}[htbp]
    \centering
    \includegraphics[width=\linewidth]{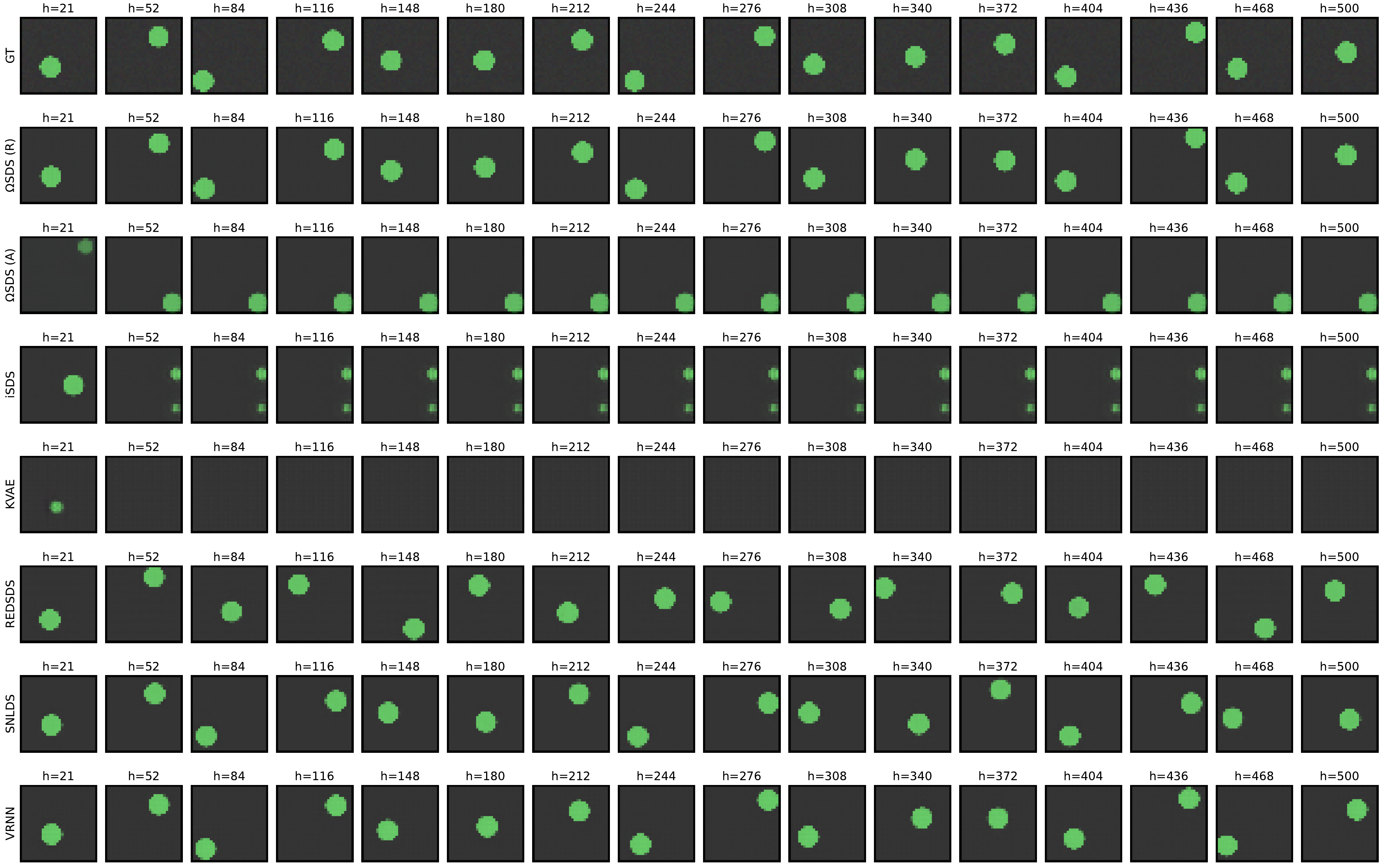}
    \caption{Comparison of long predictive rollouts with different methods. The given context length is 5.}
    \label{fig:appendix_bb_rollout_long}
\end{figure}

For quantitative comparison, Figure~\ref{fig:appendix_bb_rollout_lineplot} shows LPIPS as a function of rollout time on held-out data across baselines. We observe that both recurrent-switching methods SNLDS and $\Omega$SDS (R) maintain low perceptual error, as well as the non-switching VRNN baseline, while other switching-based methods with weaker assumptions fail to capture the long-term underlying dynamics.  

\begin{figure}
    \centering
    \includegraphics[width=.6\linewidth]{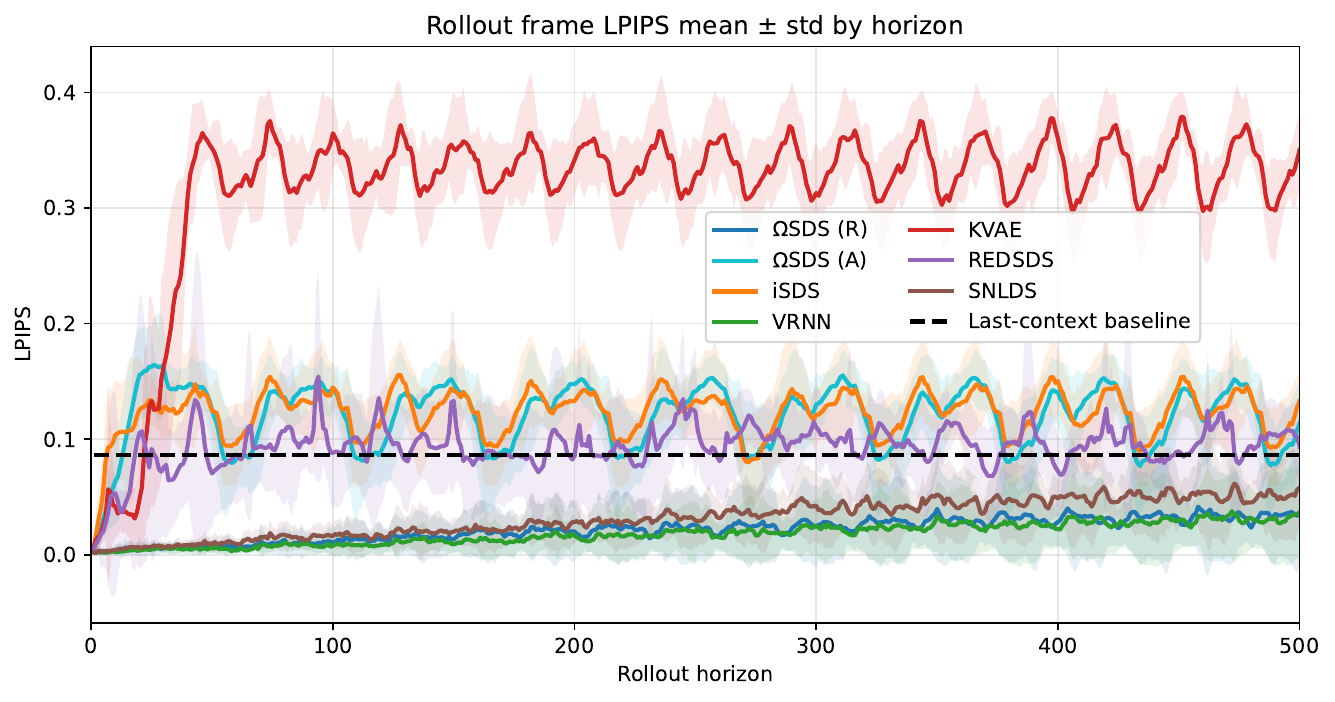}
    \caption{Plot of rollout frame LPIPS, context length = 5.}
    \label{fig:appendix_bb_rollout_lineplot}
\end{figure}

\subsection{AIST Dance Videos}\label{app:dancing_videos}

We conduct experiments on the real-world AIST dance video dataset \citet{aist-dance-db}. The dataset contains $12{,}670$ sequences of varying lengths, spanning $10$ dance genres, with $1{,}267$ sequences per genre, and includes variations in actors and camera viewpoints. We focus on segmenting sequences from the \textit{Middle Hip Hop} genre, reserving $100$ sequences for testing. Each sequence is preprocessed by: (i) temporally subsampling the video by a factor of $4$, (ii) cropping each frame around the dancer, (iii) resizing frames to $64 \times 64$. 

Since video lengths vary, during training we randomly choose $64$ consecutive frames from each sequence. We use the same dequantisation setting as in the bouncing ball experiment (see Section~\ref{appendix:bouncing_ball}). We train the model under the same conditions as described in the bouncing ball experiments, and set $K=16$. We train both $\Omega$SDS and iSDS for $150$k steps with a batch size $64$, which takes roughly one day on $4$ NVIDIA RTX 5080 GPUs. 

Training $\Omega$SDS on real-world videos is challenging since the number of regimes $K$ is unknown and the underlying dynamics may not be Markovian. Despite these difficulties, $\Omega$SDS learns meaningful regime-dependent dynamics. We show qualitative regime prediction results in Figure \ref{fig:appendix_aist_regime}. Overall, $\Omega$SDS produces more stable regime inference, although it suffers from state collapse in some examples (bottom). In contrast, iSDS estimates regimes with very short durations, suggesting that it fails to uncover informative regimes.

\begin{figure}[ht]
    \centering
    \includegraphics[width=\linewidth]{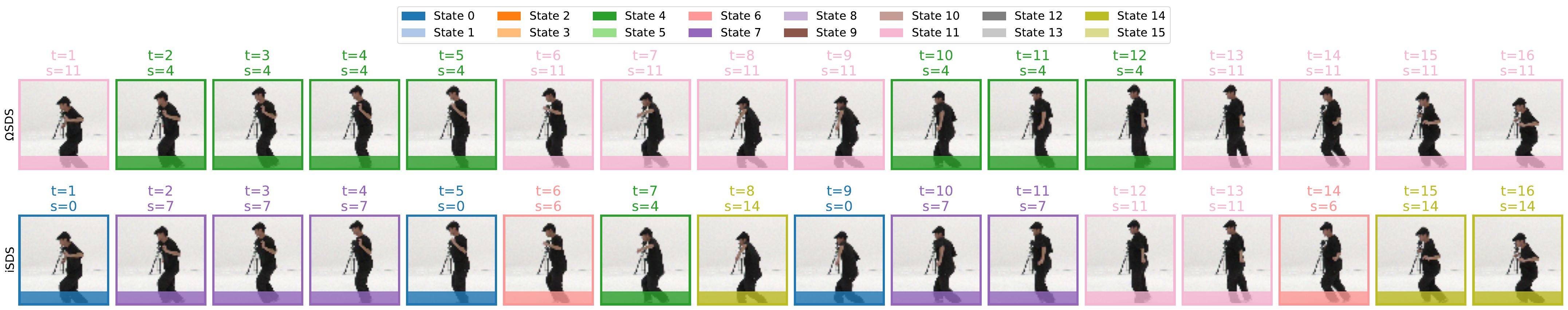}
    \includegraphics[width=\linewidth]{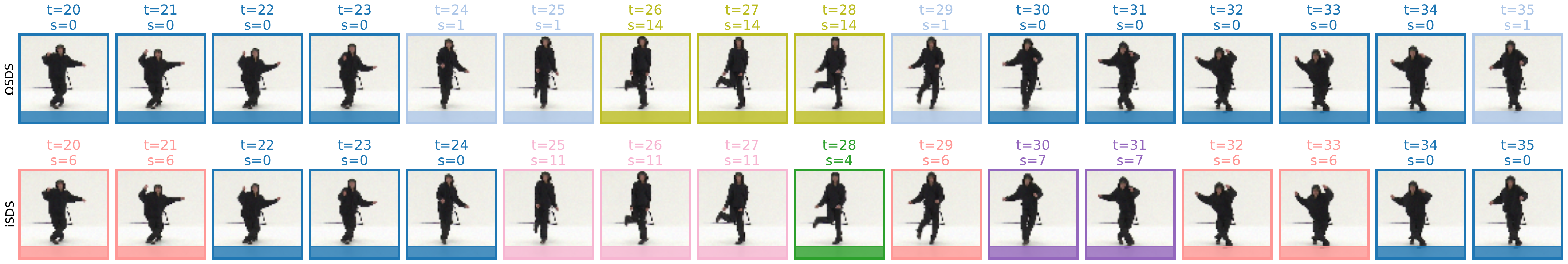}
    \includegraphics[width=\linewidth]{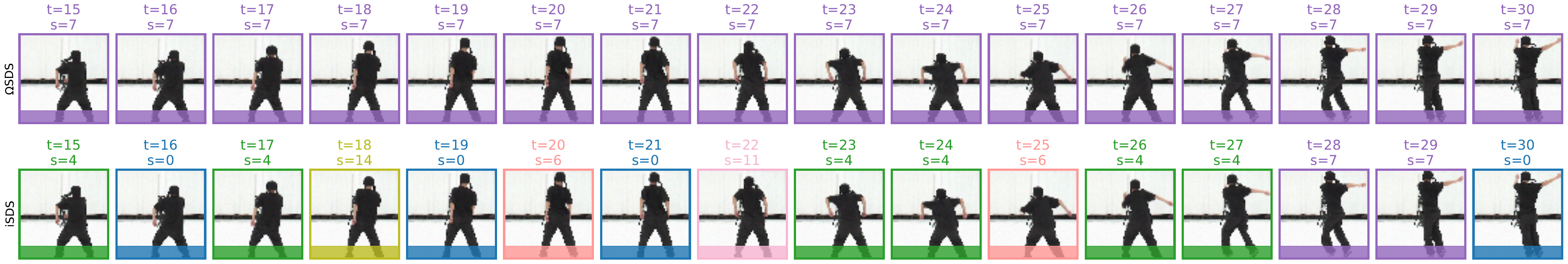}
    \caption{Posterior regimes of AIST dance videos.}
    \label{fig:appendix_aist_regime}
\end{figure}
We also show rollout predictions for both $\Omega$SDS and iSDS in Figure \ref{fig:appendix_aist_rollout}. The examples illustrate the benefit of adding recurrent feedback to the switching process, which helps the model capture future dance patterns. For example, in the top rollout sample, $\Omega$SDS predicts expressive body-motion patterns, while iSDS preserves a stationary pose. 

\begin{figure}[ht]
    \centering
    \includegraphics[width=\linewidth]{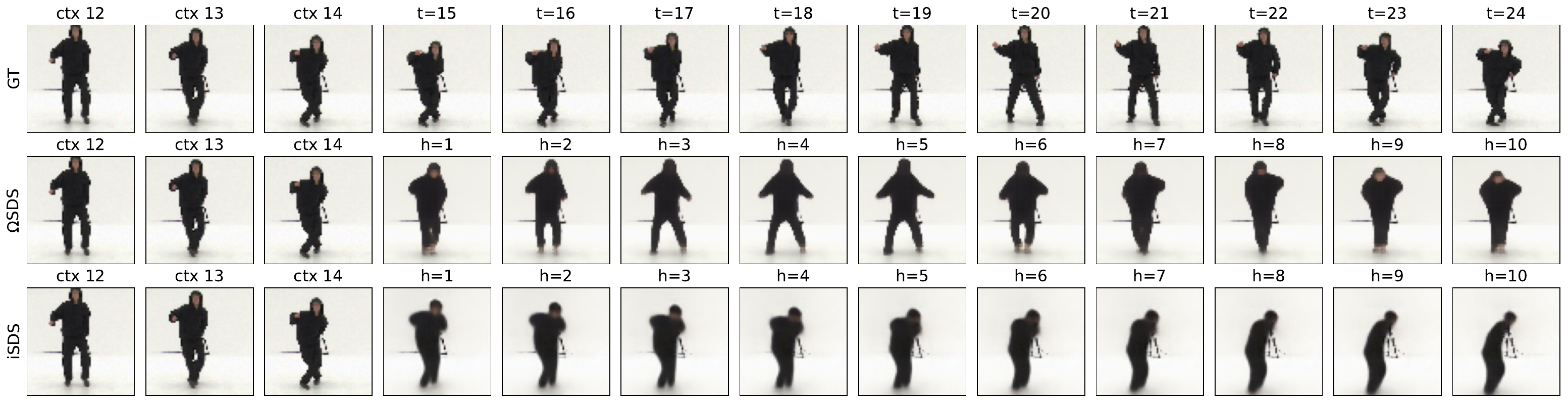}
    \includegraphics[width=\linewidth]{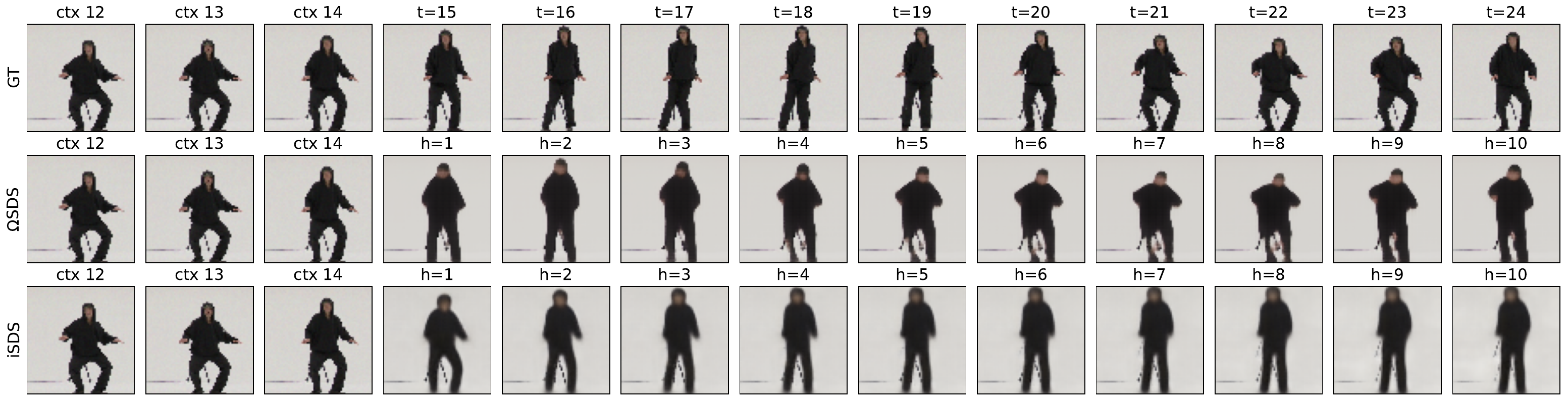}
    \includegraphics[width=\linewidth]{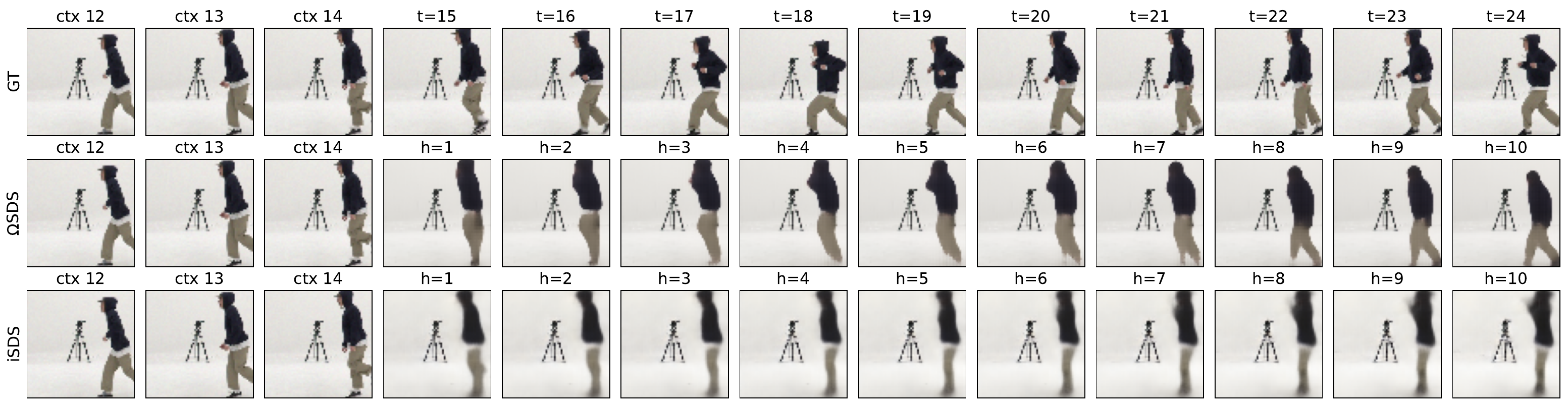}
    \caption{MAP predictive rollout results of $\Omega$SDS and iSDS given 15 context frames.}
    \label{fig:appendix_aist_rollout}
\end{figure}


\section{Broader Impacts}\label{app:broaderImpact}
This work contributes to identifiable and interpretable representation learning for sequential data with regime-switching dynamics. By improving the recovery of latent states and switching processes, the proposed framework may support more transparent modelling of complex time series, with potential applications such as climate data, motion analysis, behavioural modelling, and neuroscience. Such interpretability can be beneficial when downstream decisions require understanding of the underlying dynamics rather than relying only on opaque predictions. However, applying these models to real-world human behavioural data requires caution, and inferred regimes without expert validation should not be interpreted as definitive labels.

\clearpage
\newpage
\section*{NeurIPS Paper Checklist}

\begin{enumerate}

\item {\bf Claims}
    \item[] Question: Do the main claims made in the abstract and introduction accurately reflect the paper's contributions and scope?
    \item[] Answer: \answerYes{} 
    \item[] Justification: The claim (identifiability) made in the abstract and introduction is included in Section \ref{sec:method} and \ref{sec:estimation}, supported by empirical results in the experiment section \ref{sec:experiments}.

\item {\bf Limitations}
    \item[] Question: Does the paper discuss the limitations of the work performed by the authors?
    \item[] Answer: \answerYes{} 
    \item[] Justification: We provide a discussion on limitations in Section~\ref{sec:conclusions}

\item {\bf Theory assumptions and proofs}
    \item[] Question: For each theoretical result, does the paper provide the full set of assumptions and a complete (and correct) proof?
    \item[] Answer: \answerYes{} 
    \item[] Justification: All assumptions and proof sketch are introduced accordingly in Section \ref{sec:back_sds} and \ref{sec:method}. The corresponding proof is included in Appendix \ref{sec:appendix_proof}.

    \item {\bf Experimental result reproducibility}
    \item[] Question: Does the paper fully disclose all the information needed to reproduce the main experimental results of the paper to the extent that it affects the main claims and/or conclusions of the paper (regardless of whether the code and data are provided or not)?
    \item[] Answer: \answerYes{} 
    \item[] Justification: We provide detailed experimental settings in Section~\ref{app:experimental_details}. Code is available in the following repository: \url{https://github.com/charlio23/omega-SDS}.

\item {\bf Open access to data and code}
    \item[] Question: Does the paper provide open access to the data and code, with sufficient instructions to faithfully reproduce the main experimental results, as described in supplemental material?
    \item[] Answer: \answerYes{} 
    \item[] Justification: Part of the dataset we used is publicly available (AIST Dance DB \citep{aist-dance-db}), and code for data preprocessing, generation and training can be found in \url{https://github.com/charlio23/omega-SDS}, where model and dataset seeds keep results reproducible from code. For specific checkpoints and pre-processed data, please contact the corresponding authors.

\item {\bf Experimental setting/details}
    \item[] Question: Does the paper specify all the training and test details (e.g., data splits, hyperparameters, how they were chosen, type of optimizer) necessary to understand the results?
    \item[] Answer: \answerYes{} 
    \item[] Justification: In Section~\ref{sec:experiments}, Appendix~\ref{app:architecture} and Appendix~\ref{app:experimental_details} we have provided all necessary experimental settings to understand the results.

\item {\bf Experiment statistical significance}
    \item[] Question: Does the paper report error bars suitably and correctly defined or other appropriate information about the statistical significance of the experiments?
    \item[] Answer: \answerYes{} 
    \item[] Justification: For synthetic experiments, we run our experiments on multiple dataset seeds to verify the robustness of our empirical evidence. For video datasets, we report the error bar within the validation set for the experiments included in Section~\ref{sec:experiments} and Section~\ref{app:experimental_details}.

\item {\bf Experiments compute resources}
    \item[] Question: For each experiment, does the paper provide sufficient information on the computer resources (type of compute workers, memory, time of execution) needed to reproduce the experiments?
    \item[] Answer: \answerYes{} 
    \item[] Justification: We provide the approximate time to complete the training in Section~\ref{app:experimental_details}.
    
\item {\bf Code of ethics}
    \item[] Question: Does the research conducted in the paper conform, in every respect, with the NeurIPS Code of Ethics \url{https://neurips.cc/public/EthicsGuidelines}?
    \item[] Answer: \answerYes{} 
    \item[] Justification: -

\item {\bf Broader impacts}
    \item[] Question: Does the paper discuss both potential positive societal impacts and negative societal impacts of the work performed?
    \item[] Answer: \answerYes{} 
    \item[] Justification: We provide discussions about broader impacts in Appendix~\ref{app:broaderImpact}.
    
\item {\bf Safeguards}
    \item[] Question: Does the paper describe safeguards that have been put in place for responsible release of data or models that have a high risk for misuse (e.g., pre-trained language models, image generators, or scraped datasets)?
    \item[] Answer: \answerNA{}{} 
    \item[] Justification: -

\item {\bf Licenses for existing assets}
    \item[] Question: Are the creators or original owners of assets (e.g., code, data, models), used in the paper, properly credited and are the license and terms of use explicitly mentioned and properly respected?
    \item[] Answer: \answerYes{} 
    \item[] Justification: We use publicly available codebases for experiments with baselines and provide citations accordingly.

\item {\bf New assets}
    \item[] Question: Are new assets introduced in the paper well documented and is the documentation provided alongside the assets?
    \item[] Answer: \answerNA{} 
    \item[] Justification: 

\item {\bf Crowdsourcing and research with human subjects}
    \item[] Question: For crowdsourcing experiments and research with human subjects, does the paper include the full text of instructions given to participants and screenshots, if applicable, as well as details about compensation (if any)? 
    \item[] Answer: \answerNA{} 
    \item[] Justification: -

\item {\bf Institutional review board (IRB) approvals or equivalent for research with human subjects}
    \item[] Question: Does the paper describe potential risks incurred by study participants, whether such risks were disclosed to the subjects, and whether Institutional Review Board (IRB) approvals (or an equivalent approval/review based on the requirements of your country or institution) were obtained?
    \item[] Answer: \answerNA{} 
    \item[] Justification: -

\item {\bf Declaration of LLM usage}
    \item[] Question: Does the paper describe the usage of LLMs if it is an important, original, or non-standard component of the core methods in this research? Note that if the LLM is used only for writing, editing, or formatting purposes and does \emph{not} impact the core methodology, scientific rigor, or originality of the research, declaration is not required.
    \item[] Answer: \answerNA{} 
    \item[] Justification: -

\end{enumerate}

\end{document}